%% file: model_based_offline_game.tex
\documentclass[english]{article}
\usepackage[T1]{fontenc}
\usepackage[latin9]{inputenc}
\usepackage{geometry}
\usepackage{babel}
\usepackage{verbatim}
\usepackage{float}
\usepackage{mathtools}
\usepackage{amsmath}
\usepackage{amssymb}
\usepackage[unicode=true,
 bookmarks=false,
 breaklinks=false,pdfborder={0 0 1},colorlinks=false]
 {hyperref}
\hypersetup{
 colorlinks,linkcolor=red,anchorcolor=blue,citecolor=blue}

\makeatletter

\floatstyle{ruled}
\newfloat{algorithm}{tbp}{loa}
\providecommand{\algorithmname}{Algorithm}
\floatname{algorithm}{\protect\algorithmname}

\usepackage{babel}
\usepackage{babel}
\usepackage{babel}

\usepackage{mathtools}

\usepackage{cite}\usepackage{amsthm}\usepackage{dsfont}\usepackage{array}\usepackage{mathrsfs}\usepackage{comment}\onecolumn
\usepackage{natbib}
\usepackage{color}\usepackage{babel}

\allowdisplaybreaks

\usepackage{enumitem}
\setlist[itemize]{leftmargin=1.5em}
\setlist[enumerate]{leftmargin=1.5em}

\usepackage{babel}
\usepackage{algorithm}% http://ctan.org/pkg/algorithms
\usepackage{algorithmic}% http://ctan.org/pkg/algorithmicx
\usepackage{arydshln}
\usepackage{bm}

\DeclareMathOperator{\ind}{\mathds{1}}  % Indicator

\newcommand{\mymid}{\,|\,}

\numberwithin{equation}{section}

\definecolor{yxc}{RGB}{255,0,0}
\definecolor{yjc}{RGB}{125,0,0}
\definecolor{cm}{RGB}{0,0,200}
\definecolor{yly}{RGB}{0,150,0}

\makeatother

\begin{document}
\theoremstyle{plain} \newtheorem{lemma}{\textbf{Lemma}} \newtheorem{prop}{\textbf{Proposition}}\newtheorem{theorem}{\textbf{Theorem}}\setcounter{theorem}{0}
\newtheorem{corollary}{\textbf{Corollary}} \newtheorem{assumption}{\textbf{Assumption}}
\newtheorem{example}{\textbf{Example}} \newtheorem{definition}{\textbf{Definition}}
\newtheorem{fact}{\textbf{Fact}} \newtheorem{condition}{\textbf{Condition}}\theoremstyle{definition}

\theoremstyle{remark}\newtheorem{remark}{\textbf{Remark}}\newtheorem{claim}{\textbf{Claim}}\newtheorem{conjecture}{\textbf{Conjecture}}
\title{Model-Based Reinforcement Learning for \\ Offline
Zero-Sum Markov Games}
\author{Yuling Yan\thanks{Institute for Data, Systems, and Society, MIT, Cambridge, MA 02142, USA; Email: \texttt{yulingy@mit.edu}.} \and Gen Li\thanks{Department of Statistics, The Chinese University of Hong Kong, Hong Kong SAR, China; Email: \texttt{genli@cuhk.edu.hk}.} \and Yuxin Chen\thanks{Department of Statistics and Data Science, Wharton School, University
of Pennsylvania, Philadelphia, PA, 19104, USA; Email: \texttt{yuxinc@wharton.upenn.edu}.} \and Jianqing Fan\thanks{Department of Operations Research and Financial Engineering, Princeton
University, Princeton, NJ 08544, USA; Email: \texttt{jqfan@princeton.edu}.}}

\date{June 2022;~~Revised: February 2024}

\maketitle
\input{abstract.tex}

\noindent \textbf{Keywords:} zero-sum Markov games, Nash equilibrium,
offline RL,  unilateral coverage, curse of multiple
agents, minimax optimality

\tableofcontents{}

\input{intro.tex}

\input{problem_formulation.tex}

\input{algorithm.tex}

\input{related_work.tex}

\input{discussion.tex}

\section*{Acknowledgements}

Y.~Yan is supported in part by the Charlotte Elizabeth Procter Honorific
Fellowship from Princeton University and the Norbert Wiener Postdoctoral Fellowship from MIT. Y.~Chen is supported in part by the Alfred P.~Sloan
Research Fellowship, the Google Research Scholar Award, the AFOSR
grant FA9550-22-1-0198, the ONR grant N00014-22-1-2354, and the NSF
grants CCF-2221009, CCF-1907661, IIS-2218713, DMS-2014279, and IIS-2218773.
J.~Fan is supported in part by the  NSF grants DMS-2052926, DMS-2053832, and DMS-2210833 and ONR grant N00014-22-1-2340. Y.~Chen thanks Shicong Cen for helpful
discussions about Markov games.

\appendix
\input{additional_notation.tex}\input{proof_outline.tex}

\input{appendix_auxiliary_lemmas.tex}

\input{appendix_lower_bound.tex}

\bibliographystyle{apalike}
\bibliography{bibfileRL}

\end{document}

%% file: abstract.tex
\begin{abstract}
This paper makes progress towards learning Nash equilibria in two-player
zero-sum Markov games from offline data. Specifically, consider a
$\gamma$-discounted infinite-horizon Markov game with $S$ states,
where the max-player has $A$ actions and the min-player has $B$
actions. We propose a pessimistic model-based algorithm with Bernstein-style
lower confidence bounds --- called VI-LCB-Game --- that provably
finds an $\varepsilon$-approximate Nash equilibrium with a sample
complexity no larger than $\frac{C_{\mathsf{clipped}}^{\star}S(A+B)}{(1-\gamma)^{3}\varepsilon^{2}}$
(up to some log factor). Here, $C_{\mathsf{clipped}}^{\star}$ is
some unilateral clipped concentrability coefficient that reflects
the coverage and distribution shift of the available data (vis-à-vis
the target data), and the target accuracy $\varepsilon$ can be any
value within $\big(0,\frac{1}{1-\gamma}\big]$. Our sample complexity
bound strengthens prior art by a factor of $\min\{A,B\}$, achieving
minimax optimality for the entire $\varepsilon$-range. An appealing
feature of our result lies in algorithmic simplicity, which reveals
the unnecessity of variance reduction and sample splitting in achieving
sample optimality. 
\end{abstract}

%% file: intro.tex
\section{Introduction}

Multi-agent reinforcement learning (MARL), a subfield of reinforcement
learning (RL) that involves multiple individuals interacting/competing
with each other in a shared environment, has garnered widespread recent
interest, partly sparked by its capability of achieving superhuman
performance in game playing and autonomous driving \citep{vinyals2019grandmaster,berner2019dota,jaderberg2019human,baker2019emergent,brown2019superhuman,shalev2016safe}.
The coexistence of multiple players --- whose own welfare might come
at the expense of other parties involved --- makes MARL inherently
more intricate than the single-agent counterpart. 

A standard framework to describe the environment and dynamics in competitive
MARL is Markov games (MGs), which is generally attributed to \citet{shapley1953stochastic}
(originally referred to as stochastic games). Given the conflicting
needs of the players, a standard goal in Markov games is to seek some
sort of steady-state solutions, with the Nash equilibrium (NE) being
arguably the most prominent one. While computational intractability
has been observed when calculating NEs in general-sum MGs and/or MGs
with more than two players \citep{daskalakis2009complexity,daskalakis2013complexity},
an assortment of tractable algorithms have been put forward to solve
two-player zero-sum Markov games. On this front, a large strand of
recent works revolves around developing sample- and computation-efficient
paradigms \citep{bai2020near,xie2020learning,zhang2020model-based,tian2021online,liu2021sharp}.
What is particularly noteworthy here is the recent progress in overcoming
the so-called ``curse of multiple agents'' \citep{bai2020near,jin2021v,li2022minimax};
that is, although the total number of joint actions exhibits exponential
scaling in the number of agents, learnability of Nash equilibria becomes
plausible even when the sample size scales only linearly with the
maximum cardinality of the individual action spaces. See also \citet{jin2021v,song2021can,mao2022provably,daskalakis2022complexity}
for similar accomplishments in learning coarse correlated equilibria
in multi-player general-sum MGs. 

The aforementioned works permit online data collection either via
active exploration of the environment, or through sampling access
to a simulator. Nevertheless, the fact that real-time data acquisition
might be unaffordable or unavailable --- e.g., it could be time-consuming,
costly, and/or unsafe in healthcare and autonomous driving --- constitutes
a major hurdle for widespread adoption of these online algorithms.
This practical consideration inspires a recent flurry of studies collectively
referred to as \emph{offline RL} or \emph{batch RL} \citep{levine2020offline,kumar2020conservative},
with the aim of learning based on a historical dataset of logged interactions.

\paragraph{Data coverage for offline Markov games. }

The feasibility and efficiency of offline RL are largely governed
by the coverage of the offline data in hand. On  one hand, if the
available dataset covers all state-action pairs adequately, then there
is sufficient information to guarantee learnability; on the other
hand, full data coverage imposes an overly stringent requirement that
is rarely fulfilled in practice, and is oftentimes wasteful in terms
of data efficiency. Consequently, a recurring theme in offline RL
gravitates around the quest for algorithms that work under minimal
data coverage. Encouragingly, the recent advancement on this frontier
(e.g., \citet{rashidinejad2021bridging,xie2021policy}) uncovers the
sufficiency of ``single-policy'' data coverage in single-agent RL;
namely, offline RL becomes information theoretically feasible as soon
as the historical data covers the part of the state-action space reachable
by a single target policy. 

Unfortunately, single-policy coverage is provably insufficient when
it comes to Markov games, with negative evidence observed in \citet{cui2022offline}.
Instead, a sort of unilateral coverage --- i.e., a condition that
requires the data to cover not only the target policy pair but also
any unilateral deviation from it --- seems necessary to ensure learnability
of Nash equilibria in two-player zero-sum MGs. Employing the so-called
``unilateral concentrability coefficient'' $C^{\star}$ to quantify
such unilateral coverage as well as the degree of distribution shift
(which we shall define shortly in Assumption~\ref{assumption:uniliteral-original}),
\citet{cui2022offline} demonstrated how to find $\varepsilon$-Nash
solutions in a finite-horizon two-player zero-sum MG once the number
of sample rollouts exceeds 
\begin{equation}
\widetilde{O}\left(\frac{C^{\star}H^{3}SAB}{\varepsilon^{2}}\right).\label{eq:summary-cui-du}
\end{equation}
Here, $S$ is the number of shared states, $A$ and $B$ represent
respectively the number of actions of the max-player and the min-player,
$H$ stands for the horizon length, and the notation $\widetilde{O}(\cdot)$
denotes the orderwise scaling with all logarithmic dependency hidden. 

Despite being an intriguing polynomial sample complexity bound, a
shortfall of \eqref{eq:summary-cui-du} lies in its unfavorable scaling
with $AB$ (i.e., the total number of joint actions), which is substantially
larger than the total number of individual actions $A+B$. Whether
it is possible to alleviate this curse of multiple agents in two-player zero-sum Markov game --- and
if so, how to accomplish it --- is the key question to be investigated
in the current paper. 

\paragraph{An overview of main results. }

The objective of this paper is to design a sample-efficient offline
RL algorithm for learning Nash equilibria in two-player zero-sum Markov
games, ideally breaking the curse of multiple agents. Focusing on
$\gamma$-discounted infinite-horizon MGs, we propose a model-based
paradigm --- called VI-LCB-Game --- that is capable of learning
an $\varepsilon$-approximate Nash equilibrium with sample complexity
\[
\widetilde{O}\left(\frac{C_{\mathsf{clipped}}^{\star}S\left(A+B\right)}{\left(1-\gamma\right)^{3}\varepsilon^{2}}\right),
\]
where $C_{\mathsf{clipped}}^{\star}$ is the so-called ``clipped
unilateral concentrability coefficient'' (to be formalized in Assumption~\ref{assumption:uniliteral})
and always satisfies $C_{\mathsf{clipped}}^{\star}\leq C^{\star}$.
Our result strengthens prior theory in \citet{cui2022offline} by a
factor of $\min\{A,B\}$ (if we view the horizon length $H$ in finite-horizon
MGs and the effective horizon $\frac{1}{1-\gamma}$ in the infinite-horizon
counterpart as equivalence). To demonstrate that this bound is essentially
un-improvable, we develop a matching minimax lower bound (up to some
logarithmic factor), thus settling this problem. Our algorithm is
a pessimistic variant of value iteration with carefully designed Bernstein-style
penalties, which requires neither sample splitting nor sophisticated
schemes like reference-advantage decomposition. The fact that our
sample complexity result holds for the full $\varepsilon$-range (i.e.,
any $\varepsilon\in\big(0,\frac{1}{1-\gamma}\big]$) unveils that
sample efficiency is achieved without incurring any burn-in cost. 

Finally, while finalizing the current paper, we became aware of an
independent study \citet{cui2022provably} (posted to arXiv on June 1, 2022) that also manages to overcome the curse
of multiple agents in two-player zero-sum Markov game, which we shall elaborate on towards the end of
Section~\ref{sec:Algorithm-and-main}.

\paragraph{Notation. }

Before proceeding, let us introduce several notation that will be
used throughout. With slight abuse of notation, we shall use $P$
to denote a probability transition kernel and the associated probability
transition matrix exchangeably. We also use the notation $\mu$ exchangeably
for a probability distribution and its associated probability vector
(and we often do not specify whether $\mu$ is a row vector or column
vector as long as it is clear from the context). For any two vectors
$x=[x_{i}]_{i=1}^{n}$ and $y=[y_{i}]_{i=1}^{n}$, we use $x\circ y=[x_{i}y_{i}]_{i=1}^{n}$
to denote their Hadamard product, and we also define $x^{2}=[x_{i}^{2}]_{i=1}^{n}$
in an entrywise fashion. 
For a finite set $\mathcal{S}=\{1,\cdots,S\}$, 
we let $\Delta(\mathcal{S})\coloneqq\big\{ x\in\mathbb{R}^{S}\mid1^{\top}x=1,x\geq0\big\}$ represent the probability simplex over the set $\mathcal{S}$.

%% file: problem_formulation.tex
\section{Problem formulation\label{sec:Problem-formulation}}

In this section, we introduce the background of zero-sum Markov games,
followed by a description of the offline dataset. 

\subsection{Preliminaries}

\paragraph{Zero-sum two-player Markov games.}

Consider a discounted infinite-horizon zero-sum Markov Game (MG) \citep{shapley1953stochastic,littman1994markov},
as represented by the tuple $\mathcal{MG}=(\mathcal{S},\mathcal{A},\mathcal{B},P,r,\gamma)$.
Here, $\mathcal{S}=\{1,\ldots,S\}$ is the shared state space; $\mathcal{A}=\{1,\ldots,A\}$
(resp.~$\mathcal{B}=\{1,\ldots,B\}$) is the action space of the
max-player (resp.~min-player); $P:\mathcal{S}\times\mathcal{A}\times\mathcal{B}\to\Delta(\mathcal{S})$
is the (\emph{a priori} unknown) probability transition kernel, where
$P(s'\mymid s,a,b)$ denotes the probability of transiting from state
$s$ to state $s'$ if the max-player executes action $a$ and the
min-player chooses action $b$; $r:\mathcal{S}\times\mathcal{A}\times\mathcal{B}\to[0,1]$
is the reward function, such that $r(s,a,b)$ indicates the immediate
reward observed by both players in state $s$ when the max-player
takes action $a$ and the min-player takes action $b$; and $\gamma\in(0,1)$
is the discount factor, with $\frac{1}{1-\gamma}$ commonly referred
to as the {\em effective horizon}. Throughout this paper, we primarily focus
on the scenario where $S,A,B$ and $\frac{1}{1-\gamma}$ could all
be large. Additionally, for notational simplicity, we shall define
the vector $P_{s,a,b}\in\mathbb{R}^{1\times S}$ as $P_{s,a,b}\coloneqq P(\cdot\mymid s,a,b)$
for any $(s,a,b)\in\mathcal{S}\times\mathcal{A}\times\mathcal{B}$. 

\paragraph{Policy, value function, Q-function, and occupancy distribution.}

Let $\mu:\mathcal{S}\to\Delta(\mathcal{A})$ and $\nu:\mathcal{S}\to\Delta(\mathcal{B})$
be (possibly random) stationary policies of the max-player and the
min-player, respectively. In particular, $\mu(\cdot\mymid s)\in\Delta(\mathcal{A})$
(resp.~$\nu(\cdot\mymid s)\in\Delta(\mathcal{B})$) specifies the
action selection probability of the max-player (resp.~min-player)
in state $s$. The value function $V^{\mu,\nu}:\mathcal{S}\to\mathbb{R}$
for a given product policy $\mu\times\nu$ is defined as 
\[
V^{\mu,\nu}\left(s\right)\coloneqq\mathbb{E}\left[\sum_{t=0}^{\infty}\gamma^{t}r\left(s_{t},a_{t},b_{t}\right)\mymid s_{0}=s;\mu,\nu\right],\qquad\forall s\in\mathcal{S},
\]
where the expectation is taken with respect to the randomness of the
trajectory $\{(s_{t},a_{t},b_{t})\}_{t\geq0}$ induced by the product
policy $\mu\times\nu$ (i.e., for any $t\geq0$, the players take
$a_{t}\sim\mu(\cdot\mymid s_{t})$ and $b_{t}\sim\nu(\cdot\mymid s_{t})$
{\em independently} conditional on the past) and the probability transition
kernel $P$ (i.e., $s_{t+1}\sim P(\cdot\mymid s_{t},a_{t},b_{t})$
for $t\geq0$). Similarly, we can define the Q-function $Q^{\mu,\nu}:\mathcal{S}\times\mathcal{A}\times\mathcal{B}\to\mathbb{R}$
for a given product policy $\mu\times\nu$ as follows
\[
Q^{\mu,\nu}\left(s,a,b\right)\coloneqq\mathbb{E}\left[\sum_{t=0}^{\infty}\gamma^{t}r\left(s_{t},a_{t},b_{t}\right)\mymid s_{0}=s,a_{0}=a,b_{0}=b;\mu,\nu\right],\qquad\forall(s,a,b)\in\mathcal{S}\times\mathcal{A}\times\mathcal{B},
\]
where the actions are drawn from $\mu\times\nu$ except for the initial
time step (namely, for any $t\geq1$, we execute $a_{t}\sim\mu(\cdot\mymid s_{t})$
and $b_{t}\sim\nu(\cdot\mymid s_{t})$ independently conditional on
the past). Additionally, for any state distribution $\rho\in\Delta(\mathcal{S})$,
we introduce the following notation tailored to the weighted value
function of policy pair $(\mu,\nu)$:
\[
V^{\mu,\nu}\left(\rho\right)\coloneqq\mathbb{E}_{s\sim\rho}\left[V^{\mu,\nu}\left(s\right)\right].
\]
Moreover, we define the discounted occupancy measures associated with
an initial state distribution $\rho\in\Delta(\mathcal{S})$ and the
product policy $\mu\times\nu$ as follows: 
\begin{align}
d^{\mu,\nu}\left(s;\rho\right) & \coloneqq\left(1-\gamma\right)\sum_{t=0}^{\infty}\gamma^{t}\mathbb{P}\left(s_{t}=s\mymid s_{0}\sim\rho;\mu,\nu\right),\qquad\qquad\qquad\qquad\,\forall\,s\in\mathcal{S},\label{eq:defn-d-mu-nu-s-rho}\\
d^{\mu,\nu}\left(s,a,b;\rho\right) & \coloneqq\left(1-\gamma\right)\sum_{t=0}^{\infty}\gamma^{t}\mathbb{P}\left(s_{t}=s,a_{t}=a,b_{t}=b\mymid s_{0}\sim\rho;\mu,\nu\right),\qquad\forall\,\left(s,a,b\right)\in\mathcal{S}\times\mathcal{A}\times\mathcal{B},\label{eq:defn-d-mu-nu-s-rho-sa}
\end{align}
where the sample trajectory $\{(s_{t},a_{t},b_{t})\}_{t\geq0}$ is
initialized with $s_{0}\sim\rho$ and then induced by the product
policy $\mu\times\nu$ and the transition kernel $P$ as before. It
is clearly seen from the above definition that 
\begin{equation}
d^{\mu,\nu}\left(s,a,b;\rho\right)=d^{\mu,\nu}\left(s;\rho\right)\mu(a\mymid s)\nu(b\mymid s).\label{eq:relation-dmunu-ab-s}
\end{equation}

\paragraph{Nash equilibrium. }

In general, the two players have conflicting goals, with the max-player
aimed at maximizing the value function and the min-player minimizing
the value function. As a result, a standard compromise in Markov games
becomes finding a Nash equilibrium (NE). To be precise, a policy pair
$(\mu^{\star},\nu^{\star})$ is said to be a Nash equilibrium if no
player can benefit from unilaterally changing her own policy given
the opponent's policy \citep{nash1951non}; that is, for any policies
$\mu:\mathcal{S}\to\Delta(\mathcal{A})$ and $\nu:\mathcal{S}\to\Delta(\mathcal{B})$,
one has
\[
V^{\mu,\nu^{\star}}\leq V^{\mu^{\star},\nu^{\star}}\leq V^{\mu^{\star},\nu}.
\]
As is well known \citep{shapley1953stochastic}, there exists at least
one Nash equilibrium $(\mu^{\star},\nu^{\star})$ in the discounted two-player
zero-sum Markov game, and every NE results in the same value function
\[
V^{\star}\left(s\right)\coloneqq V^{\mu^{\star},\nu^{\star}}\left(s\right)=\max_{\mu}\min_{\nu}V^{\mu,\nu}\left(s\right)=\min_{\nu}\max_{\mu}V^{\mu,\nu}\left(s\right).
\]
In addition, when the max-player's policy $\mu$ is fixed, it is clearly
seen that the MG reduces to a single-agent Markov Decision Process
(MDP). In light of this, we define, for each $s\in\mathcal{S}$,
\[
V^{\mu,\star}(s)\coloneqq\min_{\nu}V^{\mu,\nu}(s)\qquad\text{and}\qquad V^{\star,\nu}\coloneqq\max_{\mu}V^{\mu,\nu}(s),
\]
each of which corresponds to the optimal value function of one player
with the opponent's policy frozen. Moreover, for any policy pair $(\mu,\nu)$,
the following weak duality property always holds: 
\[
V^{\mu,\star}\leq V^{\mu^{\star},\nu^{\star}}=V^{\star}\leq V^{\star,\nu}.
\]
In this paper, our goal can be posed as calculating a policy pair
$(\widehat{\mu},\widehat{\nu})$ such that
\[
V^{\widehat{\mu},\star}\left(\rho\right)-\varepsilon\leq V^{\star}\left(\rho\right)\leq V^{\star,\widehat{\nu}}\left(\rho\right)+\varepsilon,
\]
where $\rho\in\Delta(\mathcal{S})$ is some prescribed initial state
distribution, and $\varepsilon\in\big(0,\frac{1}{1-\gamma}\big]$
denotes the target accuracy level. The gap $V^{\star,\widehat{\nu}}\left(\rho\right)-V^{\widehat{\mu},\star}\left(\rho\right)$
shall often be referred to as the duality gap of $(\widehat{\mu},\widehat{\nu})$
in the rest of the present paper. 

\begin{comment}
; we define the optimal policy of the min player in this MDP, denoted
by $\mathsf{br}(\mu)$, as the min player's best response policy to
$\mu$; similarly when the min player's policy $\nu$ is fixed, we
define $\mathsf{br}(\nu)$ as the max player's best response policy
to $\nu$.
\end{comment}

\subsection{Offline dataset (batch dataset)}

Suppose that we have access to a historical dataset containing a batch
of $N$ sample transitions $\mathcal{D}=\{(s_{i},a_{i},b_{i},s_{i}')\}_{1\leq i\leq N}$,
which are generated \emph{independently} from a distribution $d_{\mathsf{b}}\in\Delta(\mathcal{S}\times\mathcal{A}\times\mathcal{B})$
and the  probability transition kernel $P$, namely,
\begin{equation}
\left(s_{i},a_{i},b_{i}\right)\overset{\text{i.i.d.}}{\sim}d_{\mathsf{b}}\qquad\text{and}\qquad s_{i}'\overset{\mathrm{ind.}}{\sim}P\left(\cdot\mymid s_{i},a_{i},b_{i}\right).\label{eq:offline-data-generation}
\end{equation}
The goal is to learn an approximate Nash equilibrium on the basis
of this historical dataset. 

In general, the data distribution $d_{\mathsf{b}}$ might deviate
from the one generated by a Nash equilibrium $(\mu^{\star},\nu^{\star})$.
As a result, whether reliable learning is feasible depends heavily
upon the quality of the historical data. To quantify the quality of
the data distribution, \citet{cui2022offline} introduced the following
unilateral concentrability condition.

\begin{assumption}[Unilateral concentrability]\label{assumption:uniliteral-original}Suppose
that the following quantity 

\begin{equation}
C^{\star}\coloneqq\max\left\{ \sup_{\mu,s,a,b}\frac{d^{\mu,\nu^{\star}}\left(s,a,b;\rho\right)}{d_{\mathsf{b}}\left(s,a,b\right)},\sup_{\nu,s,a,b}\frac{d^{\mu^{\star},\nu}\left(s,a,b;\rho\right)}{d_{\mathsf{b}}\left(s,a,b\right)}\right\} \label{eq:defn-Cstar-unilateral-original}
\end{equation}
is  finite, where we define $0/0=0$ by convention. This quantity
$C^{\star}$ is termed the ``unilateral concentrability coefficient.''
\end{assumption}

In words, this quantity $C^{\star}$ employs certain density ratios
to measure the distribution mismatch between the target distribution
and the data distribution in hand. On the one hand, Assumption~\ref{assumption:uniliteral-original}
is substantially weaker than the type of uniform coverage requirement
(which imposes a uniform bound on the density ratio $\frac{d^{\mu,\nu}\left(s,a,b;\rho\right)}{d_{\mathsf{b}}\left(s,a,b\right)}$
over all $(\mu,\nu)$ simultaneously), as (\ref{eq:defn-Cstar-unilateral-original})
freezes the policy of one side while exhausting over all policies
of the other side. On the other hand, Assumption~\ref{assumption:uniliteral-original}
remains more stringent than a single-policy coverage requirement (which
only requires the dataset to cover the part of the state-action space
reachable by a given policy pair $(\mu^{\star},\nu^{\star})$), as
(\ref{eq:defn-Cstar-unilateral-original}) requires the data to cover
those state-action pairs reachable by \emph{any unilateral deviation}
from the target policy pair $(\mu^{\star},\nu^{\star})$. As posited
by \citet{cui2022offline}, unilateral coverage (i.e., a finite $C^{\star}<\infty$)
is necessary for learning Nash Equilibria in Markov games, which stands
in sharp contrast to the single-agent case where single-policy concentrability
suffices for finding the optimal policy \citep{rashidinejad2021bridging,xie2021policy,li2022settling}. 

In this paper, we introduce a modified assumption that might give
rise to slightly improved sample complexity bounds. 

\begin{assumption}[Clipped unilateral concentrability]\label{assumption:uniliteral}Suppose
that the following quantity 
\begin{equation}
C_{\mathsf{clipped}}^{\star}\coloneqq\max\left\{ \sup_{\mu,s,a,b}\frac{\min\left\{ d^{\mu,\nu^{\star}}\left(s,a,b;\rho\right),\frac{1}{S\left(A+B\right)}\right\} }{d_{\mathsf{b}}\left(s,a,b\right)},\sup_{\nu,s,a,b}\frac{\min\left\{ d^{\mu^{\star},\nu}\left(s,a,b;\rho\right),\frac{1}{S\left(A+B\right)}\right\} }{d_{\mathsf{b}}\left(s,a,b\right)}\right\} \label{eq:defn-Cstar-unilateral}
\end{equation}
is finite, where we define $0/0=0$ by convention. This quantity $C_{\mathsf{clipped}}^{\star}$
is termed the ``clipped unilateral concentrability coefficient.''
\end{assumption}

In a nutshell, when $d^{\mu,\nu^{\star}}\left(s,a,b;\rho\right)$
or $d^{\mu^{\star},\nu}\left(s,a,b;\rho\right)$ is reasonably large
(i.e., larger than $\frac{1}{S\left(A+B\right)}$), Assumption~\ref{assumption:uniliteral}
no longer requires the data distribution $d_{\mathsf{b}}$ to scale
proportionally with $d^{\mu,\nu^{\star}}$ or $d^{\mu^{\star},\nu}$,
thus resulting in (slight) relaxation of Assumption~\ref{assumption:uniliteral-original}.
Comparing (\ref{eq:defn-Cstar-unilateral}) with (\ref{eq:defn-Cstar-unilateral-original})
immediately reveals that
\[
C^{\star}\geq C_{\mathsf{clipped}}^{\star}
\]
holds all the time. Further, it is straightforward to verify that
$C^{\star}\geq\max\{A,B\}$; in comparison, $C_{\mathsf{clipped}}^{\star}$
can be as small as $\frac{2AB}{S(A+B)}$, as shown in our lower bound
construction in Appendix \ref{sec:Proof-of-Theorem-lower-bound}.

%% file: algorithm.tex
\section{Algorithm and main theory\label{sec:Algorithm-and-main}}

In this section, we propose a pessimistic model-based offline algorithm
--- called VI-LCB-Game --- to solve the two-player zero-sum Markov
games. The proposed algorithm is then shown to achieve minimax-optimal
sample complexity in finding an approximate Nash equilibrium of the
Markov game given offline data.

\subsection{Algorithm design\label{subsec:Algorithm-design}}

\paragraph{The empirical Markov game.}

With the offline dataset $\{(s_{i},a_{i},b_{i},s_{i}')\}_{1\leq i\leq N}$
in hand, we can readily construct an empirical Markov game. To do
so, we first compute the sample size
\[
N\left(s,a,b\right)=\sum_{i=1}^{N}\ind\big\{\left(s_{i},a_{i},b_{i}\right)=\left(s,a,b\right)\big\}
\]
for each $(s,a,b)\in\mathcal{S}\times\mathcal{A}\times\mathcal{B}$.
The empirical transition kernel $\widehat{P}:\mathcal{S}\times\mathcal{A}\times\mathcal{B}\to\Delta(\mathcal{S})$
is then constructed as follows: 
\begin{equation}
\widehat{P}\left(s'\mymid s,a,b\right)=\begin{cases}
\frac{1}{N\left(s,a,b\right)}\sum_{i=1}^{N}\ind\left\{ \left(s_{i},a_{i},b_{i},s_{i}'\right)=\left(s,a,b,s'\right)\right\} , & \text{if }N\left(s,a,b\right)>0\\
\frac{1}{S}, & \text{if }N\left(s,a,b\right)=0
\end{cases}\label{eq:empirical-transition}
\end{equation}
for any $s'\in\mathcal{S}$ and any $(s,a,b)\in\mathcal{S}\times\mathcal{A}\times\mathcal{B}$.
Throughout this paper, we shall often let $\widehat{P}_{s,a,b}\in\mathbb{R}^{1\times S}$
abbreviate $\widehat{P}(\cdot\mymid s,a,b)$. In addition, the empirical
reward function $\widehat{r}:\mathcal{S}\times\mathcal{A}\times\mathcal{B}\to\mathbb{R}$
is taken to be
\begin{equation}
\widehat{r}\left(s,a,b\right)=\begin{cases}
r\left(s,a,b\right), & \text{if }N\left(s,a,b\right)>0\\
0, & \text{if }N\left(s,a,b\right)=0
\end{cases}\label{eq:empirical-reward}
\end{equation}
for any $(s,a,b)\in\mathcal{S}\times\mathcal{A}\times\mathcal{B}$.
Armed with these components, we arrive at an empirical zero-sum Markov
game, denoted by $\widehat{\mathcal{MG}}=(\mathcal{S},\mathcal{A},\mathcal{B},\widehat{P},\widehat{r},\gamma)$. 

\paragraph{Pessimistic Bellman operators.}

Recall that the classical Bellman operator $\mathcal{T}:\mathbb{R}^{SAB}\to\mathbb{R}^{SAB}$
is defined such that \citep{shapley1953stochastic,lagoudakis2002value}:
for any $Q:\mathcal{S}\times\mathcal{A}\times\mathcal{B}\to\mathbb{R}$,
\[
\mathcal{T}\left(Q\right)\left(s,a,b\right)=r\left(s,a,b\right)+\gamma P_{s,a,b}V,
\]
where $V:\mathcal{S}\to\mathbb{R}$ is the value function associated
with the input $Q$, i.e.,
\begin{equation}
V\left(s\right)\coloneqq\max_{\mu_{s}\in\Delta(\mathcal{A})}\min_{\nu_{s}\in\Delta(\mathcal{B})}\mathop{\mathbb{E}}_{a\sim\mu_{s},b\sim\nu_{s}}\big[Q(s,a,b)\big],\qquad\forall s\in\mathcal{S}.\label{eq:V-defn-Bellman}
\end{equation}
Note, however, that we are in need of modified versions of the Bellman
operator in order to accommodate the offline setting. In this paper,
we introduce the pessimistic Bellman operator $\mathcal{\widehat{\mathcal{T}}}_{\mathsf{pe}}^{-}$
(resp.~$\mathcal{\widehat{\mathcal{T}}}_{\mathsf{pe}}^{+}$) for
the max-player (resp.~min-player) as follows: for every $(s,a,b)\in\mathcal{S}\times\mathcal{A}\times\mathcal{B}$,\begin{subequations}\label{eq:bellman-op-pe-defn}
\begin{align}
\mathcal{\widehat{\mathcal{T}}}_{\mathsf{pe}}^{-}\left(Q\right)\left(s,a,b\right) & \coloneqq\max\left\{ \widehat{r}\left(s,a,b\right)+\gamma\widehat{P}_{s,a,b}V-\beta\left(s,a,b;V\right),0\right\} ,\label{eq:bellman-pe-defn}\\
\mathcal{\widehat{\mathcal{T}}}_{\mathsf{pe}}^{+}\left(Q\right)\left(s,a,b\right) & \coloneqq\min\left\{ \widehat{r}\left(s,a,b\right)+\gamma\widehat{P}_{s,a,b}V+\beta\left(s,a,b;V\right),\frac{1}{1-\gamma}\right\} ,\label{eq:bellman-op-defn}
\end{align}
\end{subequations}where $V$ is again defined in (\ref{eq:V-defn-Bellman}).
The additional term $\beta\left(s,a,b;V\right)$ is incorporated into
the operators in order to implement pessimism; informally, we anticipate
this penalty term to help $\mathcal{\widehat{\mathcal{T}}}_{\mathsf{pe}}^{-}$
(resp.~$\mathcal{\widehat{\mathcal{T}}}_{\mathsf{pe}}^{+}$) produce
a conservative estimate of the Q-function from the max-player's (resp.~min-player's)
viewpoint. Here and throughout, we choose this term based on \emph{Bernstein-style}
concentration bounds; specifically, we take
\begin{equation}
\beta\left(s,a,b;V\right)=\min\left\{ \max\left\{ \sqrt{\frac{C_{\mathsf{b}}\log\frac{N}{\delta}}{N\left(s,a,b\right)}\mathsf{Var}_{\widehat{P}_{s,a,b}}\left(V\right)},\frac{2C_{\mathsf{b}}\log\frac{N}{\delta}}{\left(1-\gamma\right)N\left(s,a,b\right)}\right\} ,\frac{1}{1-\gamma}\right\} +\frac{4}{N}\label{eq:bonus-term-defn}
\end{equation}
for some sufficiently large constant $C_{\mathsf{b}}>0$, where $1-\delta$
denotes the target success probability, and the empirical variance
term is defined as
\begin{equation}
\mathsf{Var}_{\widehat{P}_{s,a,b}}\left(V\right)\coloneqq\widehat{P}_{s,a,b}V^{2}-(\widehat{P}_{s,a,b}V)^{2}.\label{eq:empirical-variance-defn}
\end{equation}

It is well known that the classical Bellman operator $\mathcal{T}$
satisfies the $\gamma$-contraction property, which guarantees fast
global convergence of classical value iteration. As it turns out,
the pessimistic Bellman operators introduced above also enjoy the
$\gamma$-contraction property in the sense that 
\begin{equation}
\big\|\mathcal{\widehat{\mathcal{T}}}_{\mathsf{pe}}^{-}\left(Q_{1}\right)-\mathcal{\widehat{\mathcal{T}}}_{\mathsf{pe}}^{-}\left(Q_{2}\right)\big\|_{\infty}\leq\gamma\left\Vert Q_{1}-Q_{2}\right\Vert _{\infty}\quad\text{and}\quad\big\|\mathcal{\widehat{\mathcal{T}}}_{\mathsf{pe}}^{+}\left(Q_{1}\right)-\mathcal{\widehat{\mathcal{T}}}_{\mathsf{pe}}^{+}\left(Q_{2}\right)\big\|_{\infty}\leq\gamma\left\Vert Q_{1}-Q_{2}\right\Vert _{\infty};\label{eq:gamma-contraction-That}
\end{equation}
see Lemma~\ref{lemma:gamma-contraction} for precise statements. 

\paragraph{Pessimistic value iteration with Bernstein-style penalty.}

With the pessimistic Bellman operators in place, we are positioned
to present the proposed paradigm. Our algorithm maintains the Q-function
iterates $\big\{ Q_{\mathsf{pe},t}^{-}\}$, the policy iterates $\big\{\mu_{t}^{-}\big\}$
and $\big\{\nu_{t}^{-}\big\}$, and the value function iterates $\big\{ V_{\mathsf{pe},t}^{-}\big\}$
from the max-player's perspective; at the same time, it also maintains
an analogous group of iterates $\big\{ Q_{\mathsf{pe},t}^{+}\}$,
$\big\{\mu_{t}^{+}\big\}$ and $\big\{\nu_{t}^{+}\big\}$ and $\big\{ V_{\mathsf{pe},t}^{+}\big\}$
from the min-player's perspective. The updates of the two groups of
iterates are carried out in a \emph{completely decoupled} manner,
except when determining the final output. 

In what follows, let us describe the update rules from the max-player's
perspective. For notational simplicity, we shall write $\mu(s)\coloneqq\mu(\cdot\mymid s)\in\Delta(\mathcal{A})$
and $\nu(s)\coloneqq\nu(\cdot\mymid s)\in\Delta(\mathcal{B})$ whenever
it is clear from the context. In each round $t=1,2,\cdots,$ we carry
out the following update rules:
\begin{enumerate}
\item \emph{Updating Q-function estimates}. Run a pessimistic variant of
value iteration to yield
\begin{equation}
Q_{\mathsf{pe},t}^{-}=\widehat{\mathcal{T}}_{\mathsf{pe}}^{-}\left(Q_{\mathsf{pe},t-1}^{-}\right).\label{eq:VI-description}
\end{equation}
The $\gamma$-contraction property (\ref{eq:gamma-contraction-That})
helps ensure sufficient progress made in each iteration of this update
rule. 
\item \emph{Updating policy estimates}. We then adjust the policies based
on the updated Q-function estimates (\ref{eq:VI-description}). Specifically,
for each $s\in\mathcal{S}$, we compute the Nash equilibrium $\big(\mu_{t}^{-}(s),\nu_{t}^{-}(s)\big)\in\Delta(\mathcal{A})\times\Delta(\mathcal{B})$
of the zero-sum matrix game with payoff matrix $Q_{\mathsf{pe},t}^{-}(s,\cdot,\cdot)$.
It is worth noting that there is a host of methods for efficiently
calculating the NE of a zero-sum matrix game, prominent examples including
linear programming and no-regret learning \citep{raghavan1994zero,freund1999adaptive,roughgarden2016twenty,rakhlin2013optimization}. 
\item \emph{Policy evaluation}: for each $s\in\mathcal{S}$, update the
value function estimates based on the updated policies $\big(\mu_{t}^{-}(s),\nu_{t}^{-}(s)\big)$
as follows
\begin{align*}
V_{\mathsf{pe},t}^{-}\left(s\right) & =\mathop{\mathbb{E}}\limits _{a\sim\mu_{t}^{-}(s),b\sim\nu_{t}^{-}(s)}\left[Q_{\mathsf{pe},t}^{-}\left(s,a,b\right)\right].
\end{align*}
\end{enumerate}
The updates for $\big\{ Q_{\mathsf{pe},t}^{+}\}$, $\big\{\mu_{t}^{+}\big\}$
and $\big\{\nu_{t}^{+}\big\}$ from the min-player's perspective are
carried out in an analogous and completely independent manner; see
Algorithm~\ref{alg:VI} for details. 

\paragraph{Final output. }

By running the above update rules for $T=\lceil\frac{\log(N/(1-\gamma))}{\log(1/\gamma)}\rceil$
iterations, we arrive at the Q-function estimates 
\begin{equation}
Q_{\mathsf{pe}}^{-}\coloneqq Q_{\mathsf{pe},T}^{-}\qquad\text{and}\qquad Q_{\mathsf{pe}}^{+}\coloneqq Q_{\mathsf{pe},T}^{+},\label{eq:final-Q-estimate}
\end{equation}
in addition to two sets of policy estimates
\begin{equation}
\big(\mu^{-},\nu^{-}\big)\coloneqq\big(\mu_{T}^{-},\nu_{T}^{-}\big)\qquad\text{and}\qquad\big(\mu^{+},\nu^{+}\big)\coloneqq\big(\mu_{T}^{+},\nu_{T}^{+}\big).\label{eq:final-policy-estimate-2groups}
\end{equation}
The final policy estimate of the algorithm is then chosen to be 
\[
\big(\widehat{\mu},\widehat{\nu}\big)=\big(\mu^{-},\nu^{+}\big).
\]
The full algorithm is summarized in Algorithm~\ref{alg:VI}. 

\begin{comment}
As we shall demonstrate shortly in Section~\ref{subsec:properties-operators},
the iterates $\{\widehat{Q}_{\mathsf{pe},t}\}_{t\geq0}$ (resp.~$\{\widehat{Q}_{\mathsf{op},t}\}_{t\geq0}$)
converge linearly to the unique fixed point $\widehat{Q}_{\mathsf{pe}}^{\star}$
(resp.~$\widehat{Q}_{\mathsf{op}}^{\star}$) of $\widehat{\mathcal{T}}_{\mathsf{pe}}$
(resp.~$\widehat{\mathcal{T}}_{\mathsf{op}}$), owing to the $\gamma$-contraction
property (\ref{eq:gamma-contraction-That}). 

For each $s\in\mathcal{S}$, we compute the Nash equilibrium $\big(\widehat{\mu}^{-}(s),\widehat{\nu}^{-}(s)\big)\in\Delta(\mathcal{A})\times\Delta(\mathcal{B})$
of the matrix game with payoff matrix $\widehat{Q}_{\mathsf{pe}}^{-}(s,\cdot,\cdot)$,
as well as the Nash equilibrium $(\widehat{\mu}^{+}(s),\widehat{\nu}^{+}(s))\in\Delta(\mathcal{A})\times\Delta(\mathcal{B})$
of the matrix game with payoff matrix $\widehat{Q}_{\mathsf{pe}}^{+}(s,\cdot,\cdot)$. 

outputs the policy pair $(\widehat{\mu},\widehat{\nu})$ as an estimate
of the Nash equilibrium of the Markov game, where $\widehat{\mu}=\{\widehat{\mu}^{-}(s)\}$
and $\widehat{\nu}=\{\widehat{\nu}^{+}(s)\}$. 
\end{comment}

\begin{algorithm}[t]
\caption{Value iteration with lower confidence bounds for zero-sum Markov games
(VI-LCB-Game).}

\label{alg:VI}\begin{algorithmic}

\STATE \textbf{{Initialization}}: set $Q_{\mathsf{pe},0}^{-}\left(s,a,b\right)=0$
and $Q_{\mathsf{pe},0}^{+}(s,a,b)=\frac{1}{1-\gamma}$ for all $(s,a,b)\in\mathcal{S}\times\mathcal{A}\times\mathcal{B}$;
set $T=\lceil\frac{\log(N/(1-\gamma))}{\log(1/\gamma)}\rceil$.

\STATE \textbf{{Compute}} the empirical transition kernel $\widehat{P}$
as (\ref{eq:empirical-transition}) and the empirical reward function
$\widehat{r}$ as (\ref{eq:empirical-reward}).

\STATE \textbf{{For}}:\textbf{ }$t=1,\ldots,T$ \textbf{do }

\STATE
\begin{itemize}
\item Update
\begin{align*}
Q_{\mathsf{pe},t}^{-}\left(s,a,b\right) & =\widehat{\mathcal{T}}_{\mathsf{pe}}^{-}\left(Q_{\mathsf{pe},t-1}^{-}\right)=\max\left\{ \widehat{r}\left(s,a,b\right)+\gamma\widehat{P}_{s,a,b}V_{\mathsf{pe},t-1}^{-}-\beta\left(s,a,b;V_{\mathsf{pe},t-1}^{-}\right),\,0\right\} ,\\
Q_{\mathsf{pe},t}^{+}\left(s,a,b\right) & =\widehat{\mathcal{T}}_{\mathsf{pe}}^{+}\left(Q_{\mathsf{pe},t-1}^{+}\right)=\min\left\{ \widehat{r}\left(s,a,b\right)+\gamma\widehat{P}_{s,a,b}V_{\mathsf{pe},t-1}^{+}+\beta\left(s,a,b;V_{\mathsf{pe},t-1}^{+}\right),\,\frac{1}{1-\gamma}\right\} ,
\end{align*}
where
\[
\beta\left(s,a,b;V\right)=\min\left\{ \max\left\{ \sqrt{\frac{C_{\mathsf{b}}\log\frac{N}{\delta}}{N\left(s,a,b\right)}\mathsf{Var}_{\widehat{P}_{s,a,b}}\left(V\right)},\frac{2C_{\mathsf{b}}\log\frac{N}{\delta}}{\left(1-\gamma\right)N\left(s,a,b\right)}\right\} ,\frac{1}{1-\gamma}\right\} +\frac{4}{N}
\]
for some sufficiently large constant $C_{\mathsf{b}}>0$, with $\mathsf{Var}_{\widehat{P}_{s,a,b}}\left(V\right)$
defined in (\ref{eq:empirical-variance-defn}).
\item For each $s\in\mathcal{S}$, compute
\begin{align*}
\left(\mu_{t}^{-}\left(s\right),\nu_{t}^{-}\left(s\right)\right) & =\mathsf{MatrixNash}\left(Q_{\mathsf{pe},t}^{-}\left(s,\cdot,\cdot\right)\right),\\
\left(\mu_{t}^{+}\left(s\right),\nu_{t}^{+}\left(s\right)\right) & =\mathsf{MatrixNash}\left(Q_{\mathsf{pe},t}^{+}\left(s,\cdot,\cdot\right)\right),
\end{align*}
where for any matrix $M\in\mathbb{R}^{A\times B}$, the function $\mathsf{MatrixNash}(M)$
returns a solution $(\widehat{w},\widehat{z})$ to the minimax program
$\max_{w\in\Delta(\mathcal{A})}\min_{z\in\Delta(\mathcal{B})}\,w^{\top}Mz$. 
\item For each $s\in\mathcal{S}$, update
\begin{align*}
V_{\mathsf{pe},t}^{-}\left(s\right) & =\mathbb{E}_{a\sim\mu_{t}^{-}\left(s\right),b\sim\nu_{t}^{-}(s)}\left[Q_{\mathsf{pe},t}^{-}\left(s,a,b\right)\right],\\
V_{\mathsf{pe},t}^{+}\left(s\right) & =\mathbb{E}_{a\sim\mu_{t}^{+}\left(s\right),b\sim\nu_{t}^{+}(s)}\left[Q_{\mathsf{pe},t}^{+}\left(s,a,b\right)\right].
\end{align*}
\end{itemize}
\STATE \textbf{{Output}}: the policy pair $(\widehat{\mu},\widehat{\nu})$,
where $\widehat{\mu}=\{\mu_{T}^{-}(s)\}_{s\in\mathcal{S}}$ and $\widehat{\nu}=\{\nu_{T}^{+}(s)\}_{s\in\mathcal{S}}$.

\end{algorithmic}
\end{algorithm}

\subsection{Theoretical guarantees\label{subsec:Theoretical-guarantees}}

Our main result is to uncover the intriguing sample efficiency of
the proposed model-based algorithm. This is formally stated below,
with the proof postponed to Appendix~\ref{sec:Proof-of-Theorem-main-UB}. 

\begin{theorem}\label{thm:main}Consider any initial state distribution
$\rho\in\Delta(\mathcal{S})$, and suppose that Assumption \ref{assumption:uniliteral}
holds. Assume that $1/2\leq\gamma<1$, and consider any $\delta\in(0,1)$
and $\varepsilon\in\big(0,\frac{1}{1-\gamma}\big]$. Then with probability
exceeding $1-\delta$, the policy pair $(\widehat{\mu},\widehat{\nu})$
returned by Algorithm~\ref{alg:VI} satisfies
\[
V^{\widehat{\mu},\star}\left(\rho\right)-\varepsilon\leq V^{\star}\left(\rho\right)\leq V^{\star,\widehat{\nu}}\left(\rho\right)+\varepsilon,
\]
 as long as the sample size exceeds
\[
N\geq c_{1}\frac{C_{\mathsf{clipped}}^{\star}S\left(A+B\right)}{\left(1-\gamma\right)^{3}\varepsilon^{2}}\log\frac{N}{\delta}
\]
for some sufficiently large constant $c_{1}>0$. \end{theorem}\begin{remark}Our
result and analysis have been inspired by prior works that showed
that model-based RL achieves, in multiple settings, sample efficiency
without the need of variance reduction \citep{agarwal2019optimality,li2020breaking,li2022settling}.
The proof of this sample complexity bound entails several key analysis
ingredients: (i) a leave-one-out analysis argument
that proves effective in decoupling complicated statistical dependency;
and (ii) a careful self-bounding trick (i.e., upper bounding a certain
quantity by a contraction of itself in addition to some other error
terms) to derive a sharp control of the target duality gap. See Appendix~\ref{sec:Proof-of-Theorem-main-UB} for details. 
Although techniques like leave-one-out analysis have been used in some prior RL literature \citep{agarwal2019optimality,li2020breaking,li2022settling}, as far as we know, our work applies this technique for the first time to multi-agent reinforcement learning. It has been observed that extending the algorithmic or analysis ideas in single-agent RL to the multi-agent counterpart often leads to sub-optimal sample complexity bounds that scale linearly in the total number of joint actions $AB$ \citep{cui2022offline,zhang2020model-based}. In contrast, our analysis framework leads to optimal sample complexity bound that scales linearly in the total number of individual actions $A+B$.
\end{remark}

\begin{remark}  \label{remark:instance}A line of recent works focus on instance-optimality of RL algorithms \citep{khamaru2021instance,khamaru2021temporal,mou2022optimal}.
	%, and it is natural to ask whether our current analysis framework enables instance-optimal results. 
	However, it remains challenging to establish instance-dependent bounds for multi-agent RL, even in two-player zero-sum Markov games, due to the difficulties arising from offline data and multi-agent settings. Unlike RL with a generative model (simulator) that can generate independent samples for all state-action pairs, offline RL suffers from substantially more challenges like distribution shift and limited data coverage, making it more difficult to derive instance-dependent error bounds. In addition, the prior literature \citet{khamaru2021instance} that establishes instance optimality of variance-reduced Q-learning algorithms for the optimal value estimation problem requires one of the following two conditions: the optimal policy is unique, or a meaningful sample complexity bound that depends on an optimality gap can be obtained. However, neither condition has a direct analog in zero-sum Markov games; this is because the Nash equilibrium in a zero-sum Markov game is not unique in general, and there is no well-defined analog of optimality gap for zero-sum Markov games. Detailed discussion on the challenges and difficulty of extending our analysis to develop instance-dependent error bounds can be found in Appendix \ref{appendix:discuss-instance}.
\end{remark}

The sample complexity needed for Algorithm \ref{alg:VI} to compute
a policy pair with $\varepsilon$-duality gap is at most
\begin{equation}
\widetilde{O}\left(\frac{C_{\mathsf{clipped}}^{\star}S\left(A+B\right)}{\left(1-\gamma\right)^{3}\varepsilon^{2}}\right),\label{eq:sample-complexity}
\end{equation}
which accommodates any target accuracy within the range $\big(0,\frac{1}{1-\gamma}\big]$.
In addition to linear dependency on $C_{\mathsf{clipped}}^{\star}$,
the sample complexity bound (\ref{eq:sample-complexity}) scales linearly
(as opposed to quadratically) with the aggregate size $A+B$ of the
individual action spaces. It is noteworthy that our algorithm is a
fairly straightforward implementation of the model-based approach
(except that the pessimism principle is incorporated), and does not
require either sample splitting or sophisticated schemes like variance
reduction \citep{zhang2020almost,li2021breaking,yan2022efficacy,xie2021policy,zhang2020model-based}. 

As it turns out, the above sample complexity theory for Algorithm
\ref{alg:VI} matches the minimax lower limit modulo some logarithmic
term, as asserted by the following theorem. This minimax lower bound
--- whose proof is postponed to Appendix~\ref{sec:Proof-of-Theorem-lower-bound}
--- is inspired by prior lower bound theory for single-agent MDPs
(e.g., \citet{azar2013minimax,li2022settling}) and might shed light
on how to establish lower bounds for other game-theoretic settings.

\begin{theorem}\label{thm:lower-bound} Consider any $S\geq2$, $A\geq2$,
$B\geq2$, $\gamma\in[\frac{2}{3},1)$ and $C_{\mathsf{clipped}}^{\star}\geq\frac{2AB}{S(A+B)}$,
and define the set
\begin{align*}
\mathsf{MG}\left(C_{\mathsf{clipped}}^{\star}\right) & \coloneqq\Bigg\{\big\{\mathcal{MG},\rho,d_{\mathsf{b}}\big\}\,\,\Bigg|\,\,\left|\mathcal{S}\right|=S,\ \left|\mathcal{A}\right|=A,\ \left|\mathcal{B}\right|=B,\\
 & \qquad\rho\in\Delta\left(\mathcal{S}\right),\ d_{\mathsf{b}}\in\Delta\left(\mathcal{S}\times\mathcal{A}\times\mathcal{B}\right),\ \exists\text{ an NE }\left(\mu^{\star},\nu^{\star}\right)\text{ of }\mathcal{MG}\text{ such that }\\
 & \qquad\max\Bigg\{\sup_{\mu,s,a,b}\frac{\min\left\{ d^{\mu,\nu^{\star}}\left(s,a,b;\rho\right),\frac{1}{S\left(A+B\right)}\right\} }{d_{\mathsf{b}}\left(s,a,b\right)},\sup_{\nu,s,a,b}\frac{\min\left\{ d^{\mu^{\star},\nu}\left(s,a,b;\rho\right),\frac{1}{S\left(A+B\right)}\right\} }{d_{\mathsf{b}}\left(s,a,b\right)}\Bigg\}=C_{\mathsf{clipped}}^{\star}\Bigg\}.
\end{align*}
Then there exist some universal constants $c_{2},c_{\varepsilon}>0$
such that: for any $\varepsilon\in(0,\frac{1}{c_{\varepsilon}(1-\gamma)\log(A+B)}]$,
if the sample size obeys
\[
N<\frac{c_{2}S\left(A+B\right)C_{\mathsf{clipped}}^{\star}}{\left(1-\gamma\right)^{3}\varepsilon^{2}\log\left(A+B\right)},
\]
then one necessarily has
\[
\underset{(\widehat{\mu},\widehat{\nu})}{\inf}\underset{\{\mathcal{MG},\rho,d_{\mathsf{b}}\}\in\mathsf{MG}(C_{\mathsf{clipped}}^{\star})}{\sup}\mathbb{E}\left[V^{\star,\widehat{\nu}}\left(\rho\right)-V^{\widehat{\mu},\star}\left(\rho\right)\right]\geq\varepsilon.
\]
Here, the infimum is taken over all estimators $(\widehat{\mu},\widehat{\nu})$
for the Nash equilibrium based on the batch dataset $\mathcal{D}=\{(s_{i},a_{i},b_{i},s_{i}')\}_{i=1}^{n}$
generated according to (\ref{eq:offline-data-generation}). \end{theorem}

\begin{remark} 
The target we are estimating is the NE of a zero-sum MG, which is more challenging than standard statistical estimation problems in the sense that (i) NE is not unique in general and (ii) the error metric is a duality-gap. It is challenging to use standard proof frameworks like Fano's and Le Cam's methods to derive a meaningful lower bound for this problem. To overcome this challenge, we construct a family of hard Markov game instances indexed by a binary parameter $\theta\in\{0,1\}^{\max\{A,B\}}$, and then put a prior distribution over this set and compute
the posterior probability of failure to differentiate
each entry of $\theta$. These steps taken together carefully allow us to compute the desired minimax risk. 
\end{remark}

\begin{comment}
\begin{remark}This minimax lower bound is established by putting
a prior distribution over a set of hard Markov game instances indexed
by an $\max\{A,B\}$-dimensional binary parameter $\theta$ and computing
the posterior probability of being unsuccessful in differentiating
each entry of $\theta$.\end{remark}
\end{comment}

As a direct implication of Theorem~\ref{thm:lower-bound}, if the
total number of samples in the offline dataset obeys
\[
N<\frac{c_{2}S\left(A+B\right)C_{\mathsf{clipped}}^{\star}}{\left(1-\gamma\right)^{3}\varepsilon^{2}\log\left(A+B\right)},
\]
then one can construct a hard Markov game instance such that no algorithm
whatsoever can reach a duality gap below $\varepsilon$. This taken
collectively (\ref{eq:sample-complexity}) unveils, up to some logarithmic
factor, the minimax statistical limit for finding NEs based on offline
data. 

Our theory makes remarkable improvement upon prior art, which can
be seen through comparisons with the most relevant prior work \citep{cui2022offline}
(even though the focus therein is finite-horizon zero-sum MGs). On
a high level, \citet{cui2022offline} proposed an algorithm that combines
pessimistic value iteration with variance reduction (also called reference-advantage
decomposition \citep{zhang2020almost}), which provably finds an $\varepsilon$-Nash
policy pair using 
\begin{equation}
\widetilde{O}\left(\frac{C^{\star}SABH^{3}}{\varepsilon^{2}}\right)\label{eq:sample-complexity-cui}
\end{equation}
sample trajectories, provided that $\varepsilon\leq1/H$. Here, $H$
stands for the horizon length of the finite-horizon Markov game, and
$C^{\star}$ is the unilateral concentrability coefficient tailored
to the finite-horizon setting. Despite the difference between discounted
infinite-horizon and finite-horizon settings, our algorithm design
and theory achieve several improvements upon \citet{cui2022offline}. 
\begin{itemize}
\item Perhaps most importantly, our result scales linearly in the total
number of individual actions $A+B$ (as opposed to the number of joint
actions $AB$ as in \citet{cui2022offline}), which manages to alleviate
the curse of multiple agents in two-player zero-sum Markov game. 
\item Our theory accommodates the full $\varepsilon$-range $\big(0,\frac{1}{1-\gamma}\big]$,
which is much wider than the range $(0,1/H]$ covered by \citet{cui2022offline}
(if we view the effective horizon $\frac{1}{1-\gamma}$ in the infinite-horizon
case and the horizon length $H$ in the finite-horizon counterpart
as equivalence). 
\item The algorithm design herein is substantially simpler than \citet{cui2022offline}:
neither does it require sample splitting to decouple statistical dependency,
nor does it rely on reference-advantage decomposition techniques to
sharpen the horizon dependency. 
\end{itemize}
While we were finalizing the present manuscript, we became aware of
the independent work \citet{cui2022provably} proposing a different
offline algorithm --- based on incorporation of strategy-wise lower
confidence bounds --- that improved the prior art as well. When it
comes to two-player zero-sum Markov games with finite horizon and
non-stationary transition kernels, \citet[Algorithm 1]{cui2022provably}
provably yields an $\varepsilon$-Nash policy pair using 
\begin{equation}
\widetilde{O}\left(\frac{C^{\star}S\left(A+B\right)H^{4}}{\varepsilon^{2}}\right)\label{eq:new-bound-Du}
\end{equation}
sample trajectories each containing $H$ samples. This bound (\ref{eq:new-bound-Du})
is \emph{at least} a factor of $H$ above the minimax limit. It is
worth noting that \citet{cui2022provably} is able to accommodate offline
multi-agent general-sum MGs, although the algorithm proposed therein
becomes computationally intractable when going beyond two-player zero-sum
MGs.

%% file: related_work.tex
\section{Related works } \label{sec:related-works}

\paragraph{Offline RL and pessimism principle. } 

The principle of pessimism in the face of uncertainty, namely, being
conservative in value estimation of those state-action pairs that
have been under-covered, has been adopted extensively in recent development
of offline RL. A highly incomplete list includes \citet{kumar2020conservative,kidambi2020morel,yu2020mopo,yu2021conservative,yu2021combo,yin2021near_double,rashidinejad2021bridging,jin2021pessimism,xie2021policy,liu2020provably,zhang2021corruption,chang2021mitigating,yin2021towards,uehara2021pessimistic,munos2003error,munos2007performance,yin2021near_b,zanette2021provable,yan2022efficacy,li2022settling,shi2022pessimistic,cui2022offline,zhong2022pessimistic,lu2022pessimism,li2022pessimism,wang2022gap,xu2022provably},
which unveiled the efficacy of the pessimism principle in both model-based
and model-free approaches. Among this body of prior works, the ones
that are most related to the current paper are \citet{cui2022offline,zhong2022pessimistic,cui2022provably},
both of which focused on episodic finite-horizon zero-sum Markov games
with two players. More specifically, \citet{cui2022offline} demonstrated
that a unilateral concentrability condition is necessary for learning
NEs in offline settings, and proposed a pessimistic value iteration
with reference-advantage decomposition to enable sample efficiency;
\citet{zhong2022pessimistic} proposed a \emph{Pessimistic Minimax
Value Iteration} algorithm which achieves appealing sample complexity
in the presence of linear function representation, which was recently
improved by \citet{xiong2022nearly}. The concurrent work \citet{cui2022provably}
proposed a different pessimistic algorithm that designed LCBs for
policy pairs instead of state-action pairs; for two-player zero-sum
MGs, their algorithm is capable of achieving a sample complexity proportional
to $A+B$. In the single-agent offline RL setting, \citet{rashidinejad2021bridging,yan2022efficacy,li2022settling}
studied offline RL for infinite-horizon MDPs, and \citet{jin2021pessimism,xie2021policy,shi2022pessimistic,li2022settling}
looked at the finite-horizon episodic counterpart, all of which operate
upon some single-policy concentrability assumptions. Among these works,
\citet{li2022settling} and \citet{yan2022efficacy} achieved minimax-optimal
sample complexity $\widetilde{O}(\frac{SC^{\star}}{(1-\gamma)^{3}\varepsilon^{2}})$
for discounted infinite-horizon MDPs by means of model-based and model-free
algorithms, respectively; similar results have been established for
finite-horizon MDPs as well \citep{xie2021policy,li2022settling,shi2022pessimistic,yin2021near_a,yin2021near_b}.

\paragraph{Multi-agent RL and Markov games. }

The concept of Markov games --- also under the name of stochastic
games --- dated back to \citet{shapley1953stochastic}, which has
become a central framework to model competitive multi-agent decision
making. A large strand of prior works studied how to efficiently solve
Markov games when perfect model description is available \citep{littman1994markov,littman2001friend,hu2003nash,hansen2013strategy,cen2021fast,mao2022provably,daskalakis2020independent,perolat2015approximate,wei2021last,zhao2021provably,daskalakis2022complexity,chen2021almost}.
Recent years have witnessed much activity in studying the sample efficiency
of learning Nash equilibria in zero-sum Markov games, covering multiple
different types of sampling schemes; for instance, \citet{wei2017online,xie2020learning,bai2020near,bai2020provable,liu2021sharp,jin2021v,song2021can,mao2022provably,daskalakis2022complexity,tian2021online,chen2021almost}
focused on the online explorative environments, whereas \citet{zhang2020model-based}
paid attention to the scenario that assumes sampling access to a generative
model. While the majority of these works exhibited a sample complexity
that scales at least as $\widetilde{O}(SAB)$ in order to learn an
approximate NE, the recent work \citet{bai2020near} proposed a V-learning
algorithm attaining a sample complexity that scales linearly with
$S(A+B)$, thus matching the minimax-optimal lower bound up to a factor
of $H^{2}$. When a generative model is available, \citet{li2022minimax}
further developed an algorithm that learns $\varepsilon$-Nash using
$\widetilde{O}\big(\frac{H^{4}S(A+B)}{\varepsilon^{2}}\big)$ samples,
which attains the minimax lower bound for non-stationary finite-horizon
MGs. The setting of general-sum multi-player Markov games is much
more challenging, given that learning Nash equilibria is known to
be PPAD-complete \citep{daskalakis2009complexity,daskalakis2013complexity}.
Shifting attention to more tractable solution concepts, \citet{jin2021v,daskalakis2022complexity,mao2022provably,song2021can}
proposed algorithms that provably learn (coarse) correlated equilibria
with sample complexities that scale linearly with $\max_{i}A_{i}$
(where $A_{i}$ is the number of actions of the $i$-th player), thereby
breaking the curse of multi-agents. Additionally, there have also
been several works investigating the turn-based setting where the
two players take actions in turn; see \citet{sidford2020solving,cui2021minimax,jia2019feature,jin2022complexity}.
Moreover, another two works \citet{zhang2021finite,abe2020off} studied
offline sampling oracles under uniform coverage requirements (which
are clearly more stringent than the unilateral concentrability assumption).
The interested readers are also referred to \citet{zhang2021multi,yang2020overview}
for an overview of recent development.

\paragraph{Model-based RL. }

The method proposed in the current paper falls under the category
of model-based algorithms, which decouple model estimation and policy
learning (planning). The model-based approach has been extensively
studied in the single-agent setting including the online exploration
setting \citep{azar2017minimax}, the case with a generative model
\citep{agarwal2020model,azar2013minimax,li2020breaking,wang2021sample,jin2021towards},
the offline RL setting \citep{li2022settling,xie2021policy}, and turn-based
Markov games \citep{cui2021minimax}. Encouragingly, the model-based
approach is capable of attaining minimax-optimal sample complexities
in a variety of settings (e.g., \citet{azar2017minimax,li2022settling,agarwal2020model}),
sometimes even without incurring any burn-in cost \citep{li2020breaking,li2022settling,cui2021minimax}.
The method proposed in \citet{cui2022offline} also exhibited the flavor
of a model-based algorithm, although an additional variance reduction
scheme is incorporated in order to optimize the horizon dependency.

%% file: discussion.tex
\section{Discussion \label{sec:Discussion}}

In the present paper, we have proposed a model-based offline algorithm,
which leverages the principle of pessimism in solving two-player zero-sum
Markov games on the basis of past data. In order to find an $\varepsilon$-approximate
Nash equilibrium of the Markov game, our algorithm requires no more
than $\widetilde{O}\big(\frac{S(A+B)C^{\star}}{(1-\gamma)^{3}\varepsilon^{2}}\big)$
samples, and this sample complexity bound is provably minimax optimal
for the entire range of target accuracy level $\varepsilon\in\big(0,\frac{1}{1-\gamma}\big]$.
Our theory has improved upon prior sample complexity bounds in \citet{cui2021minimax}
in terms of the dependency on the size of the action space. Another
appealing feature is the simplicity of our algorithm, which does not
require complicated variance reduction schemes and is hence easier
to implement and interpret. Moving forward, there are a couple of
interesting directions that are worthy of future investigation. For
instance, one natural extension is to explore whether the current
algorithmic idea and analysis extend to multi-agent general-sum Markov
games, with the goal of learning other solutions concepts of equilibria
like coarse correlated equilibria (given that finding Nash equilibria
in general-sum games is PPAD-complete). Another topic of interest
is to design model-free algorithms for offline NE learning in zero-sum
or general-sum Markov games. Furthermore, the current paper focuses
attention on tabular Markov games, and it would be of great interest
to design sample-efficient offline multi-agent algorithms in the presence
of function approximation.

%% file: additional_notation.tex
\section{Additional notation}

Let us collect a set of additional notation that will be used in the
analysis. First of all, for any $(s,a,b)\in\mathcal{S}\times\mathcal{A}\times\mathcal{B}$,
any vector $V\in\mathbb{R}^{S}$ and any probability transition kernel
$P:\mathcal{S}\times\mathcal{A}\times\mathcal{B}\rightarrow\Delta(\mathcal{S})$,
we define
\begin{equation}
\mathsf{Var}_{P_{s,a,b}}(V)=P_{s,a,b}\big(V\circ V\big)-\big(P_{s,a,b}V\big)^{2},\label{eq:var-Psab-V-defn}
\end{equation}
where $P_{s,a,b}$ abbreviates $P(\cdot\mymid s,a,b)$ as usual. 
When the max-player's policy $\mu$ is fixed, the Markov game reduces
to a (single-agent) MDP for the min-player. For any MDP, it is known
that there exists at least one policy that simultaneously maximizes
the value function (resp.~Q-function) for all states (resp.~state-action
pairs) \citep{bertsekas2017dynamic}. In light of this, when the policy
$\mu$ of the max-player is frozen, we denote by $\nu_{\mathsf{br}}(\mu)$
the optimal policy of the min-player, which shall often be referred
to as the best response of the min-player when the max-player adopts
policy $\mu$. Similarly, we can define the best response of the max-player
when the min-player adopts policy $\nu$, which we denoted by $\mu_{\mathsf{br}}(\nu)$.
These allow one to define
\[
V^{\mu,\star}\left(s\right)\coloneqq V^{\mu,\nu_{\mathsf{br}}(\mu)}\left(s\right)=\min_{\nu}V^{\mu,\nu}\left(s\right),\qquad V^{\star,\nu}\left(s\right)\coloneqq V^{\mu_{\mathsf{br}}(\nu),\nu}\left(s\right)=\max_{\mu}V^{\mu,\nu}\left(s\right)
\]
for all $s\in\mathcal{S}$, and 
\[
Q^{\mu,\star}\left(s,a,b\right)\coloneqq Q^{\mu,\nu_{\mathsf{br}}(\mu)}\left(s,a,b\right)=\min_{\nu}Q^{\mu,\nu}\left(s,a,b\right),\quad Q^{\star,\nu}\left(s,a,b\right)\coloneqq Q^{\mu_{\mathsf{br}}(\nu),\nu}\left(s,a,b\right)=\max_{\mu}Q^{\mu,\nu}\left(s,a,b\right)
\]
for all $(s,a,b)\in\mathcal{S}\times\mathcal{A}\times\mathcal{B}$.
Note that the definitions of $V^{\mu,\star}$ and $V^{\star,\nu}$
here are consistent with the ones in Section \ref{sec:Problem-formulation}.

%We
%use the standard notation $f(n)\lesssim g(n)$ or $f(n)=O(g(n))$
%to denote $\vert f(n)\vert\leq Cg(n)$ for some universal constant
%$C>0$ when $n$ is sufficiently large; we let $f(n)\gtrsim g(n)$
%denote $f(n)\geq C\vert g(n)\vert$ for some constant $C>0$ when
%$n$ is large enough; and we employ $f(n)\asymp g(n)$ to indicate
%that $f(n)\gtrsim g(n)$ and $f(n)\lesssim g(n)$ hold simultaneously.

%% file: proof_outline.tex
\section{Proof of Theorem \ref{thm:main}\label{sec:Proof-of-Theorem-main-UB}}

Towards proving Theorem \ref{thm:main}, we first state a slightly
stronger result as follows.

\begin{theorem}\label{thm:main-complete}Consider any initial state
distribution $\rho\in\Delta(\mathcal{S})$, and suppose that Assumption
\ref{assumption:uniliteral} holds. Assume that $1/2\leq\gamma<1$.
Then with probability exceeding $1-\delta$, the policy pair $(\widehat{\mu},\widehat{\nu})$
returned by Algorithm~\ref{alg:VI} satisfies\begin{subequations}\label{eq:V-munustar-gap-bound-thm}
\begin{align}
V^{\star}\left(\rho\right)-V^{\widehat{\mu},\star}\left(\rho\right) & \leq c_{0}\sqrt{\frac{C_{\mathsf{clipped}}^{\star}S\left(A+B\right)}{\left(1-\gamma\right)^{3}N}\log\frac{N}{\delta}}+c_{0}\frac{C_{\mathsf{clipped}}^{\star}S\left(A+B\right)}{\left(1-\gamma\right)^{2}N}\log\frac{N}{\delta},\label{eq:V-mustar-gap-bound-thm}\\
V^{\star,\widehat{\nu}}\left(\rho\right)-V^{\star}\left(\rho\right) & \leq c_{0}\sqrt{\frac{C_{\mathsf{clipped}}^{\star}S\left(A+B\right)}{\left(1-\gamma\right)^{3}N}\log\frac{N}{\delta}}+c_{0}\frac{C_{\mathsf{clipped}}^{\star}S\left(A+B\right)}{\left(1-\gamma\right)^{2}N}\log\frac{N}{\delta}\label{eq:V-nustar-gap-bound-thm}
\end{align}
\end{subequations}for some sufficiently large constant $c_{0}>0$.
As an immediate consequence, the duality gap of $(\widehat{\mu},\widehat{\nu})$
obeys, with probability at least $1-\delta$, that
\begin{equation}
V^{\star,\widehat{\nu}}\left(\rho\right)-V^{\widehat{\mu},\star}\left(\rho\right)\leq2c_{0}\sqrt{\frac{C_{\mathsf{clipped}}^{\star}S\left(A+B\right)}{\left(1-\gamma\right)^{3}N}\log\frac{N}{\delta}}+2c_{0}\frac{C_{\mathsf{clipped}}^{\star}S\left(A+B\right)}{\left(1-\gamma\right)^{2}N}\log\frac{N}{\delta}.\label{eq:V-duality-gap-theorem}
\end{equation}
\end{theorem}

As can be straightforwardly verified, Theorem \ref{thm:main} is a
direct consequence of Theorem \ref{thm:main-complete} (by taking
the right-hand side of \eqref{eq:V-duality-gap-theorem} to be no
larger than $\varepsilon$).

The remainder of this section is thus
dedicated to establishing Theorem \ref{thm:main-complete}. 
Before proceeding, let us now take a moment to  provide a brief roadmap of the proof.
\begin{enumerate}
	\item We will first show in Appendix \ref{subsec:properties-operators} that the pessimistic Bellman operators $\widehat{\mathcal{T}}_{\mathsf{pe}}^{-}$
	and $\widehat{\mathcal{T}}_{\mathsf{pe}}^{+}$ introduced in (\ref{eq:bellman-op-pe-defn}) are both monotone, $\gamma$-contractive, and admit unique fixed points $Q_{\mathsf{pe},t}^{-\star}$  and $Q_{\mathsf{pe},t}^{-\star}$ respectively. These properties reveal that the pessimistic value iterations $\{Q_{\mathsf{pe},t}^{-}\}_{1\leq t\leq T}$ (resp.~$\{Q_{\mathsf{pe},t}^{+}\}_{1\leq t\leq T}$) in Algorithm \ref{alg:VI} converge to $Q_{\mathsf{pe},t}^{-\star}$  (resp.~$Q_{\mathsf{pe},t}^{-\star}$) at a geometric rate, and therefore it suffices to analyze the fixed points $Q_{\mathsf{pe},t}^{-\star}$  and $Q_{\mathsf{pe},t}^{+\star}$.
	\item Next we show Bernstein-style concentration bounds for random quantities like $(\widehat{P}_{s,a,b}-P_{s,a,b})V_{\mathsf{pe},t}^{-\star}$ and $(\widehat{P}_{s,a,b}-P_{s,a,b})V_{\mathsf{pe},t}^{+\star}$ in Appendix \ref{subsec:bernstein-concentration}, where $V_{\mathsf{pe},t}^{-\star}$  and $V_{\mathsf{pe},t}^{+\star}$ are the value functions associated with $Q_{\mathsf{pe},t}^{-\star}$  and $Q_{\mathsf{pe},t}^{+\star}$. Due to the complicated statistical dependency between $\widehat{P}_{s,a,b}$ and $V_{\mathsf{pe},t}^{-\star}$, we use a leave-one-out argument to establish this concentration result in Lemma \ref{lemma:loo}.
	\item Finally, based on the aforementioned results, we derive error bounds for $V^{\star}(\rho)-V^{\widehat{\mu},\star}(\rho)$
	and $V^{\star,\widehat{\nu}}(\rho)-V^{\star}(\rho)$ in Appendix \ref{subsec:thm-1-upper-bound}. Our analysis makes use of a ``self-bounding'' trick, which allows one to derive sharp estimation error bounds which turns out to be minimax-optimal.
\end{enumerate}

\subsection{Preliminary facts \label{sec:proof-outline-prelim} }

Before continuing, we collect several preliminary facts that will
be useful throughout. 
\begin{enumerate}
\item For any $Q_{1},Q_{2}:\mathcal{S}\times\mathcal{A}\times\mathcal{B}\to\mathbb{R}$,
we have
\begin{equation}
\left\Vert V_{1}-V_{2}\right\Vert _{\infty}\leq\left\Vert Q_{1}-Q_{2}\right\Vert _{\infty},\label{eq:Q-V-bound}
\end{equation}
where $V_{1}$ (resp.~$V_{2}$) denotes the value function associated
with $Q_{1}$ (resp.~$Q_{2}$); see (\ref{eq:V-defn-Bellman}) for
the precise definition. 
\item For any $V_{1},V_{2}:\mathcal{S}\to\big[0,\frac{1}{1-\gamma}\big]$,
any probability transition kernel $P:\mathcal{S}\times\mathcal{A}\times\mathcal{B}\to\Delta(\mathcal{S})$
and any $(s,a,b)\in\mathcal{S}\times\mathcal{A}\times\mathcal{B}$,
we have
\begin{equation}
\left|\mathsf{Var}_{P_{s,a,b}}\left(V_{1}\right)-\mathsf{Var}_{P_{s,a,b}}\left(V_{2}\right)\right|\leq\frac{4}{1-\gamma}\left\Vert V_{1}-V_{2}\right\Vert _{\infty},\label{eq:variance-V-bound}
\end{equation}
where $\mathsf{Var}_{P,s,a,b}(V)$ is defined in \eqref{eq:var-Psab-V-defn}. 
\item As a consequence, we also know that for any$(s,a,b)\in\mathcal{S}\times\mathcal{A}\times\mathcal{B}$
and any $V_{1},V_{2}:\mathcal{S}\to\big[0,\frac{1}{1-\gamma}\big]$,
the corresponding penalty terms (cf.~\eqref{eq:bonus-term-defn})
obey
\begin{equation}
\left|\beta\left(s,a,b;V_{1}\right)-\beta\left(s,a,b;V_{2}\right)\right|\leq2\left\Vert V_{1}-V_{2}\right\Vert _{\infty}.\label{eq:b-V-bound}
\end{equation}
\end{enumerate}
The proof of the preceding results can be found in Appendix \ref{sec:proof-proof-outline}.

\subsection{Step 1: key properties of pessimistic Bellman operators \label{subsec:properties-operators}}

Recall the definition of the pessimistic Bellman operators $\widehat{\mathcal{T}}_{\mathsf{pe}}^{-}$
and $\widehat{\mathcal{T}}_{\mathsf{pe}}^{+}$ introduced in (\ref{eq:bellman-op-pe-defn}).
The following lemma gathers a couple of key properties of these two
operators.

\begin{lemma}\label{lemma:gamma-contraction}The following properties
hold true:
\begin{itemize}
\item (Monotonicity) For any $Q_{1}\geq Q_{2}$ we have $\widehat{\mathcal{T}}_{\mathsf{pe}}^{-}(Q_{1})\geq\widehat{\mathcal{T}}_{\mathsf{pe}}^{-}(Q_{2})$
and $\widehat{\mathcal{T}}_{\mathsf{pe}}^{+}(Q_{1})\geq\widehat{\mathcal{T}}_{\mathsf{pe}}^{+}(Q_{2})$;
\item (Contraction) Both operators are $\gamma$-contractive in the $\ell_{\infty}$
sense, i.e., 
\[
\big\|\widehat{\mathcal{T}}_{\mathsf{pe}}^{-}\left(Q_{1}\right)-\widehat{\mathcal{T}}_{\mathsf{pe}}^{-}\left(Q_{2}\right)\big\|_{\infty}\leq\gamma\big\| Q_{1}-Q_{2}\big\|_{\infty},\qquad\big\|\widehat{\mathcal{T}}_{\mathsf{pe}}^{+}\left(Q_{1}\right)-\widehat{\mathcal{T}}_{\mathsf{pe}}^{+}\left(Q_{2}\right)\big\|_{\infty}\leq\gamma\left\Vert Q_{1}-Q_{2}\right\Vert _{\infty}
\]
for any $Q_{1}$ and $Q_{2}$;
\item (Uniqueness of fixed points) $\widehat{\mathcal{T}}_{\mathsf{pe}}^{-}$
(resp.~$\widehat{\mathcal{T}}_{\mathsf{pe}}^{+}$) has a unique fixed
point $Q_{\mathsf{pe}}^{-\star}$ (resp.~$Q_{\mathsf{pe}}^{+\star}$),
which also satisfies $0\leq Q_{\mathsf{pe}}^{\star-}(s,a,b)\leq\frac{1}{1-\gamma}$
(resp.~$0\leq Q_{\mathsf{pe}}^{+\star}(s,a,b)\leq\frac{1}{1-\gamma}$)
for any $(s,a,b)\in\mathcal{S}\times\mathcal{A}\times\mathcal{B}$.
\end{itemize}
\end{lemma}\begin{proof}See Appendix \ref{sec:proof-lemma-gamma-contraction}.
\end{proof}

Next, we make note of several immediate consequences of Lemma~\ref{lemma:gamma-contraction}.
Here and throughout, $V_{\mathsf{pe}}^{-\star}$ and $V_{\mathsf{pe}}^{+\star}$
are defined to be the value functions (see \eqref{eq:V-defn-Bellman})
associated with $Q_{\mathsf{pe}}^{-\star}$ and $Q_{\mathsf{pe}}^{+\star}$,
respectively. 
\begin{itemize}
\item First of all, the above lemma implies that
\begin{equation}
Q_{\mathsf{pe},t}^{-}\leq Q_{\mathsf{pe}}^{-\star}\quad(\forall t\geq0)\qquad\text{and hence}\qquad Q_{\mathsf{pe}}^{-}\leq Q_{\mathsf{pe}}^{-\star}.\label{eq:Q-pe-t-minus-star-all-t}
\end{equation}
To see this, we first note that $Q_{\mathsf{pe},0}^{-}=0\leq Q_{\mathsf{pe}}^{-\star}$.
Next, suppose that $Q_{\mathsf{pe},t}^{-}\leq Q_{\mathsf{pe}}^{-\star}$
for some iteration $t\geq0$, then the monotonicity of $\widehat{\mathcal{T}}_{\mathsf{pe}}^{-}$
(cf.~Lemma \ref{lemma:gamma-contraction}) tells us that
\[
Q_{\mathsf{pe},t+1}^{-}=\widehat{\mathcal{T}}_{\mathsf{pe}}^{-}\left(Q_{\mathsf{pe},t}^{-}\right)\leq\widehat{\mathcal{T}}_{\mathsf{pe}}^{-}(Q_{\mathsf{pe}}^{-\star})=Q_{\mathsf{pe}}^{-\star},
\]
from which \eqref{eq:Q-pe-t-minus-star-all-t} follows. 
\item In addition, the $\gamma$-contraction property in Lemma~\ref{lemma:gamma-contraction}
leads to 
\begin{equation}
\left\Vert V_{\mathsf{pe}}^{-}-V_{\mathsf{pe}}^{-\star}\right\Vert _{\infty}\leq\left\Vert Q_{\mathsf{pe}}^{-}-Q_{\mathsf{pe}}^{-\star}\right\Vert _{\infty}\leq\frac{1}{N},\label{eq:Q_pe_approx_err}
\end{equation}
To justify this, observe that
\begin{align*}
\left\Vert Q_{\mathsf{pe},t}^{-}-Q_{\mathsf{pe}}^{-\star}\right\Vert _{\infty} & =\left\Vert \widehat{\mathcal{T}}_{\mathsf{pe}}^{-}\big(Q_{\mathsf{pe},t-1}^{-}\big)-\widehat{\mathcal{T}}_{\mathsf{pe}}^{-}\big(Q_{\mathsf{pe}}^{-\star}\big)\right\Vert _{\infty}\leq\gamma\left\Vert Q_{\mathsf{pe},t-1}^{-}-Q_{\mathsf{pe}}^{-\star}\right\Vert _{\infty}\\
 & \leq\cdots\leq\gamma^{t}\left\Vert Q_{\mathsf{pe},0}^{-}-Q_{\mathsf{pe}}^{-\star}\right\Vert _{\infty}\leq\frac{\gamma^{t}}{1-\gamma},
\end{align*}
which together with $T=\lceil\frac{\log(N/(1-\gamma))}{\log(1/\gamma)}\rceil$
and \eqref{eq:Q-V-bound} gives
\[
\left\Vert V_{\mathsf{pe}}^{-}-V_{\mathsf{pe}}^{-\star}\right\Vert _{\infty}\leq\big\| Q_{\mathsf{pe}}^{-}-Q_{\mathsf{pe}}^{-\star}\big\|_{\infty}=\big\| Q_{\mathsf{pe},T}^{-}-Q_{\mathsf{pe}}^{-\star}\big\|_{\infty}\leq\frac{\gamma^{T}}{1-\gamma}\leq\frac{1}{N}.
\]
\item A similar argument also yields
\begin{equation}
Q_{\mathsf{pe}}^{+}\geq Q_{\mathsf{pe}}^{+\star},\qquad\big\| Q_{\mathsf{pe}}^{+}-Q_{\mathsf{pe}}^{+\star}\big\|_{\infty}\leq1/N,\qquad\big\| V_{\mathsf{pe}}^{+}-V_{\mathsf{pe}}^{+\star}\big\|_{\infty}\leq1/N.\label{eq:Qpe-Vpe-Qpeplus-properties}
\end{equation}
\end{itemize}

\subsection{Step 2: decoupling statistical dependency and establishing pessimism} \label{subsec:bernstein-concentration}

To proceed, we rely on the following theorem to quantify the difference
between $\widehat{P}$ and $P$ when projected onto a value function
direction. 

\begin{lemma}\label{lemma:loo}For any $(s,a,b)\in\mathcal{S}\times\mathcal{A}\times\mathcal{B}$
satisfying $N(s,a,b)\geq1$, with probability exceeding $1-\delta$,
\begin{equation}
\left|\big(\widehat{P}_{s,a,b}-P_{s,a,b}\big)\widetilde{V}\right|\leq\widetilde{c}\sqrt{\frac{1}{N\left(s,a,b\right)}\mathsf{Var}_{\widehat{P}_{s,a,b}}\big(\widetilde{V}\big)\log\frac{N}{\delta}}+\widetilde{c}\frac{\log\frac{N}{\delta}}{\left(1-\gamma\right)N\left(s,a,b\right)}\label{eq:Bernstein-type-lemma-loo}
\end{equation}
for some sufficiently large constant $\widetilde{c}>0$, and
\begin{equation}
\mathsf{Var}_{\widehat{P}_{s,a,b}}\big(\widetilde{V}\big)\leq2\mathsf{Var}_{P_{s,a,b}}\big(\widetilde{V}\big)+O\left(\frac{1}{\left(1-\gamma\right)^{2}N\left(s,a,b\right)}\log\frac{N}{\delta}\right)\label{eq:Var-Phat-P-diff}
\end{equation}
hold simultaneously for all $\widetilde{V}\in\mathbb{R}^{S}$ satisfying
$0\leq\widetilde{V}\leq\frac{1}{1-\gamma}1$ and $\min\big\{\Vert\widetilde{V}-V_{\mathsf{pe}}^{-\star}\Vert_{\infty},\Vert\widetilde{V}-V_{\mathsf{pe}}^{+\star}\Vert_{\infty}\big\}\leq1/N$.
\end{lemma}\begin{proof}See Appendix \ref{sec:proof-lemma-loo}.\end{proof}

In words, the first result \eqref{eq:Bernstein-type-lemma-loo} delivers
some Bernstein-type concentration bound, whereas the second result
\eqref{eq:Var-Phat-P-diff} guarantees that the empirical variance
estimate (i.e., the plug-in estimate) is close to the true variance.
It is worth noting that Lemma~\ref{lemma:loo} does \emph{not} require
$\widetilde{V}$ to be statistically independent from $\widehat{P}_{s,a,b}$,
which is particularly crucial when coping with complicated statistical
dependency of our iterative algorithm. The proof of Lemma~\ref{lemma:loo}
is established upon a leave-one-out analysis argument (see, e.g.,
\citet{agarwal2020model,li2020breaking,li2022settling,chen2021spectral,ma2020implicit})
that helps decouple statistical dependency; see details in Appendix
\ref{sec:proof-lemma-loo}. Armed with Lemma \ref{lemma:loo}, we
can readily see that 
\begin{equation}
\left|\big(\widehat{P}_{s,a,b}-P_{s,a,b}\big)\widetilde{V}\right|+\frac{4}{N}\leq\beta\big(s,a,b;\widetilde{V}\big)\label{eq:b-dominance-1}
\end{equation}
holds for any $(s,a,b)\in\mathcal{S}\times\mathcal{A}\times\mathcal{B}$
satisfying $N(s,a,b)\geq1$ and any $\widetilde{V}$ satisfying the
conditions in Lemma \ref{lemma:loo}. In turn, this important fact
allows one to justify that $Q_{\mathsf{pe}}^{-}$ (resp.~$Q_{\mathsf{pe}}^{+}$)
is indeed an upper (resp.~lower) bound on $Q^{\widehat{\mu},\star}$
(resp.~$Q^{\star,\widehat{\nu}}$), as formalized below. 

\begin{lemma}\label{lemma:Q-monononicity}With probability exceeding
$1-\delta$, it holds that
\[
Q_{\mathsf{pe}}^{-}\leq Q^{\widehat{\mu},\star},\qquad Q_{\mathsf{pe}}^{+}\geq Q^{\star,\widehat{\nu}},\qquad V_{\mathsf{pe}}^{-}\leq V^{\widehat{\mu},\star}\qquad\text{and}\qquad V_{\mathsf{pe}}^{+}\geq V^{\star,\widehat{\nu}}.
\]
\end{lemma}\begin{proof}See Appendix \ref{sec:proof-lemma-Q-monotonicity}.\end{proof}

This lemma makes clear a key rationale for the principle of pessimism:
we would like the Q-function estimates to be always conservative uniformly
over all entries. 

\subsection{Step 3: bounding $V^{\star}(\rho)-V^{\widehat{\mu},\star}(\rho)$
and $V^{\star,\widehat{\nu}}(\rho)-V^{\star}(\rho)$} 
\label{subsec:thm-1-upper-bound}

Before proceeding to bound $V^{\star}-V^{\widehat{\mu},\star}$, we
first develop a lower bound on $V_{\mathsf{pe}}^{-}$, given that
$V^{\widehat{\mu},\star}$ is lower bounded by $V_{\mathsf{pe}}^{-}$
(according to Lemma~\ref{lemma:Q-monononicity}). Towards this end,
we invoke the definition of $V_{\mathsf{pe}}^{-}$ to reach
\begin{align}
V_{\mathsf{pe}}^{-}\left(s\right) & =\max_{\mu(s)\in\Delta(\mathcal{A})}\min_{\nu(s)\in\Delta(\mathcal{B})}\mathop{\mathbb{E}}\limits _{a\sim\mu(s),b\sim\nu(s)}\left[Q_{\mathsf{pe}}^{-}\left(s,\cdot,\cdot\right)\right]\geq\min_{\nu(s)\in\Delta(\mathcal{B})}\mathop{\mathbb{E}}\limits _{a\sim\mu^{\star}(s),b\sim\nu(s)}\left[Q_{\mathsf{pe}}^{-}\left(s,a,b\right)\right],\label{eq:Vpe-minus-lower-bound-min}
\end{align}
where we set the policy of the max-player to be $\mu^{\star}$ on
the right-hand side of the above equation. Clearly, there exists a
deterministic policy $\nu_{0}:\mathcal{S}\to\Delta(\mathcal{B})$
such that 
\begin{equation}
\nu_{0}\left(s\right)=\arg\min_{\nu(s)\in\Delta(\mathcal{B})}\mathop{\mathbb{E}}\limits _{a\sim\mu^{\star}(s),b\sim\nu(s)}\big[Q_{\mathsf{pe}}^{-}\left(s,a,b\right)\big]\label{eq:defn-nu0-min}
\end{equation}
for any $s\in\mathcal{S}$; for instance, one can simply set, for
any $s\in\mathcal{S}$,
\begin{equation}
\nu_{0}(s)=\ind_{b_{s}}\qquad\text{with }b_{s}\coloneqq\arg\max_{b\in\mathcal{B}}\ \big\langle\mu^{\star}\left(s\right),Q_{\mathsf{pe}}^{-}\left(s,\cdot,b\right)\big\rangle,\label{eq:defn-nu0-s}
\end{equation}
with $\ind_{b_{s}}$ denoting a probability vector that is nonzero
only in $b_{s}$. This deterministic policy $\nu_{0}$ helps us lower
bound $V_{\mathsf{pe}}^{-}$, as accomplished in the following lemma.
Here and below, we define two vectors $r^{\mu^{\star},\nu_{0}},\beta^{\mu^{\star},\nu_{0}}\in\mathbb{R}^{S}$
and a probability transition kernel $P^{\mu^{\star},\nu_{0}}:\mathcal{S}\to\Delta(\mathcal{S})$
restricted to $\mu^{\star}$ and $\nu_{0}$ such that: for any $s,s'\in\mathcal{S}$,
\begin{subequations}\label{eq:defn-beta-r-P-mu-nu0}
\begin{align}
r^{\mu^{\star},\nu_{0}}\left(s\right) & \coloneqq\mathbb{E}_{a\sim\mu^{\star}(s),b\sim\nu_{0}(s)}\left[r\left(s,a,b\right)\right],\label{eq:defn-r-mu-nu0}\\
\beta^{\mu^{\star},\nu_{0}}\left(s\right) & \coloneqq\mathbb{E}_{a\sim\mu^{\star}(s),b\sim\nu_{0}(s)}\left[\beta\left(s,a,b;V_{\mathsf{pe}}^{-}\right)\right],\label{eq:defn-beta-mu-nu0}\\
P^{\mu^{\star},\nu_{0}}\left(s'\mymid s\right) & \coloneqq\mathbb{E}_{a\sim\mu^{\star}(s),b\sim\nu_{0}(s)}\left[P\left(s'\mymid s,a,b\right)\right].\label{eq:defn-P-mu-nu0}
\end{align}
\end{subequations}

\begin{lemma}\label{lemma:V_pe_lower_bound}With probability exceeding
$1-\delta$, we have
\begin{equation}
V_{\mathsf{pe}}^{-}\geq r^{\mu^{\star},\nu_{0}}+\gamma P^{\mu^{\star},\nu_{0}}V_{\mathsf{pe}}^{-}-2\beta^{\mu^{\star},\nu_{0}}.\label{eq:main-proof-2}
\end{equation}
 \end{lemma}\begin{proof}See Appendix \ref{subsec:proof-lemma_V_pe_lower_bound}.\end{proof}

In addition, we can invoke Lemma \ref{lemma:Q-monononicity} and the
fact that $V^{\star}=V^{\mu^{\star},\nu^{\star}}=V^{\mu^{\star},\star}$
to reach
\begin{equation}
V^{\star}-V^{\widehat{\mu},\star}=V^{\mu^{\star},\star}-V^{\widehat{\mu},\star}\leq V^{\mu^{\star},\nu_{0}}-V_{\mathsf{pe}}^{-},\label{eq:main-proof-6}
\end{equation}
which motivates us to look at $V^{\mu^{\star},\nu_{0}}-V_{\mathsf{pe}}^{-}$.
Towards this, we note that the Bellman equation tells us that
\begin{equation}
V^{\mu^{\star},\nu_{0}}=r^{\mu^{\star},\nu_{0}}+\gamma P^{\mu^{\star},\nu_{0}}V^{\mu^{\star},\nu_{0}}.\label{eq:main-proof-3}
\end{equation}
Taking (\ref{eq:main-proof-2}) and (\ref{eq:main-proof-3}) collectively
yields 
\begin{align}
V^{\mu^{\star},\nu_{0}}-V_{\mathsf{pe}}^{-} & \leq\gamma P^{\mu^{\star},\nu_{0}}\big(V^{\mu^{\star},\nu_{0}}-V_{\mathsf{pe}}^{-}\big)+2\beta^{\mu^{\star},\nu_{0}},\label{eq:main-proof-4}
\end{align}
thus resulting in a ``self-bounding'' type of relations. Applying
(\ref{eq:main-proof-4}) recursively, we arrive at 
\begin{align*}
\rho^{\top}\big(V^{\mu^{\star},\nu_{0}}-V_{\mathsf{pe}}^{-}\big) & \leq\gamma\rho^{\top}P^{\mu^{\star},\nu_{0}}\big(V^{\mu^{\star},\nu_{0}}-V_{\mathsf{pe}}^{-}\big)+2\rho^{\top}\beta^{\mu^{\star},\nu_{0}}\\
 & \leq\gamma^{2}\rho^{\top}\big(P^{\mu^{\star},\nu_{0}}\big)^{2}\big(V^{\mu^{\star},\nu_{0}}-V_{\mathsf{pe}}^{-}\big)+2\rho^{\top}\beta^{\mu^{\star},\nu_{0}}+2\gamma\rho^{\top}P^{\mu^{\star},\nu_{0}}\beta^{\mu^{\star},\nu_{0}}\\
 & \leq\cdots\leq\gamma^{n}\rho^{\top}\big(P^{\mu^{\star},\nu_{0}}\big)^{n}\big(V^{\mu^{\star},\nu_{0}}-V_{\mathsf{pe}}^{-}\big)+2\rho^{\top}\left[\sum_{i=0}^{n-1}\gamma^{i}\big(P^{\mu^{\star},\nu_{0}}\big)^{i}\right]\beta^{\mu^{\star},\nu_{0}}
\end{align*}
holds for all positive integers $n$. Letting $n\to\infty$ and recalling
that the vector $d^{\mu^{\star},\nu_{0}}\coloneqq[d^{\mu^{\star},\nu_{0}}(s;\rho)]_{s\in\mathcal{S}}$
obeys (see \eqref{eq:defn-d-mu-nu-s-rho}))
\begin{equation}
d^{\mu^{\star},\nu_{0}}=(1-\gamma)\rho^{\top}\sum_{i=0}^{\infty}\gamma^{i}(P^{\mu^{\star},\nu_{0}})^{i}=(1-\gamma)\rho^{\top}\big(I-\gamma P^{\mu^{\star},\nu_{0}}\big)^{-1},\label{eq:defn-d-mu-nu-vector}
\end{equation}
we arrive at
\begin{align}
\rho^{\top}\big(V^{\mu^{\star},\nu_{0}}-V_{\mathsf{pe}}^{-}\big) & \leq\left\{ \lim_{n\rightarrow\infty}\gamma^{n}\rho^{\top}\big(P^{\mu^{\star},\nu_{0}}\big)^{n}\big(V^{\mu^{\star},\nu_{0}}-V_{\mathsf{pe}}^{-}\big)\right\} +\frac{2}{1-\gamma}\big(d^{\mu^{\star},\nu_{0}}\big)^{\top}\beta^{\mu^{\star},\nu_{0}}\nonumber \\
 & =\frac{2}{1-\gamma}\big(d^{\mu^{\star},\nu_{0}}\big)^{\top}\beta^{\mu^{\star},\nu_{0}},\label{eq:main-proof-5}
\end{align}
where the last line makes use of the fact that $\|\rho^{\top}\big(P^{\mu^{\star},\nu_{0}}\big)^{n}\|_{1}=1$
for any $n\geq1$ and hence $\gamma^{n}\rho^{\top}\big(P^{\mu^{\star},\nu_{0}}\big)^{n}\rightarrow0$
as $n\rightarrow\infty$ when $\gamma<1$. 

In order to further control \eqref{eq:main-proof-5}, we resort to
the following lemma for bounding $\big(d^{\mu^{\star},\nu_{0}}\big)^{\top}\beta^{\mu^{\star},\nu_{0}}$,
whose proof can be found in Appendix \ref{sec:proof_lemma_d_b_upper_bound}.

\begin{lemma}\label{lemma:d_b_upper_bound}There exists some large
enough universal constant $c_{6}>0$ such that
\[
\big(d^{\mu^{\star},\nu_{0}}\big)^{\top}\beta^{\mu^{\star},\nu_{0}}\leq c_{6}\frac{C_{\mathsf{clipped}}^{\star}S\left(A+B\right)}{\left(1-\gamma\right)N}\log\frac{N}{\delta}+c_{6}\sqrt{\frac{C_{\mathsf{clipped}}^{\star}S\left(A+B\right)}{N\left(1-\gamma\right)}\log\frac{N}{\delta}}.
\]
with probability exceeding $1-\delta$. \end{lemma}

To finish up, taking (\ref{eq:main-proof-6}), (\ref{eq:main-proof-5})
and Lemma \ref{lemma:d_b_upper_bound} together gives
\begin{align*}
V^{\star}\left(\rho\right)-V^{\widehat{\mu},\star}\left(\rho\right) & =\rho^{\top}\big(V^{\star}-V^{\widehat{\mu},\star}\big)\leq\rho^{\top}\big(V^{\mu^{\star},\nu_{0}}-V_{\mathsf{pe}}^{-}\big)\leq\frac{2}{1-\gamma}\big(d^{\mu^{\star},\nu_{0}}\big)^{\top}\beta^{\mu^{\star},\nu_{0}}\\
 & \leq2c_{6}\sqrt{\frac{C_{\mathsf{clipped}}^{\star}S\left(A+B\right)}{N\left(1-\gamma\right)^{3}}\log\frac{N}{\delta}}+\frac{2c_{6}C_{\mathsf{clipped}}^{\star}S\left(A+B\right)}{\left(1-\gamma\right)^{2}N}\log\frac{N}{\delta}.
\end{align*}
This has completed the proof for the claim \eqref{eq:V-mustar-gap-bound-thm}.
The proof for the other claim \eqref{eq:V-nustar-gap-bound-thm} follows
from an almost identical argument, and is hence omitted.

\subsection{Discussion: instance-dependent statistical bounds?}  \label{appendix:discuss-instance}
Thus far, we have presented the proof of Theorem~\ref{thm:main} that concerns the minimax optimality of the model-based algorithm. 
Note that a recent line of works have attempted to move beyond minimax-optimal statistical guarantees and pursue more refined instance-optimal (or locally minimax) performance guarantees \citep{khamaru2021instance,khamaru2021temporal,mou2022optimal}.  
Here, we take a moment to discuss the challenges that need to be overcome in order to extend our analysis in an instance-optimal fashion. 
\begin{itemize}
\item A crucial step that allows us to obtain error bounds that scale linearly with $A+B$ instead of the ones that scale linearly with $AB$ in \citet{cui2022offline} is the introduction of an auxiliary policy $\nu_0$ in \eqref{eq:defn-nu0-min}. This allows us to upper bound the error with
\begin{align*}
	\rho^{\top}\big(V^{\mu^{\star},\nu_{0}}-V_{\mathsf{pe}}^{-}\big)  \leq\frac{2}{1-\gamma}\big(d^{\mu^{\star},\nu_{0}}\big)^{\top}\beta^{\mu^{\star},\nu_{0}};
\end{align*} 
see \eqref{eq:main-proof-5}. While this facilitates our analysis, the terms $\beta^{\mu^\star,\nu_0}$ (defined in \eqref{eq:defn-beta-r-P-mu-nu0}) and $d^{\mu^\star,\nu_0}$ (defined in \eqref{eq:defn-d-mu-nu-vector}) both depend on the auxiliary policy $\nu_0$. Due to the complicated dependency between $\nu_0$ (which is determined by a random function $V_{\mathsf{pe}}^{-}$) and the model parameters, it remains quite challenging to connect these error terms with instance-dependent quantities (i.e.~model parameters) without losing optimality.  

\item Note that we might be able to resolve the above issue in a coarse way, e.g., by taking the supremum over all possible policy $\nu$:
\begin{equation}
	\rho^{\top}\big(V^{\mu^{\star},\nu_{0}}-V_{\mathsf{pe}}^{-}\big)  \leq\frac{2}{1-\gamma}\sup_{\nu\in\Delta(\mathcal{B})}\big(d^{\mu^{\star},\nu}\big)^{\top}\beta^{\mu^{\star},\nu}. \label{eq:discuss-instance-dependent}
\end{equation}
Let us assume for the moment that this could work (despite the potential suboptimality of this error bound) and see what this will lead to. By checking the proof of Lemma \ref{lemma:d_b_upper_bound}, we can see that: in order to upper bound \eqref{eq:discuss-instance-dependent} in an instance-optimal manner, it is important to relate $\mathsf{Var}_{P_{s,a,b}}(V_{\mathsf{pe}}^{-})$ to model parameters for all $(s,a,b)\in\mathcal{S}\times\mathcal{A}\times\mathcal{B}$. Ideally, we can replace $V_{\mathsf{pe}}^{-}$ with the value function associated with the Nash equilibrium $V^\star$. However, based on our current analysis framework, we can only show that $V^\star(\rho)$ and $V_{\mathsf{pe}}^{-}(\rho)$ are close, which is insufficient to guarantee the closeness of $\mathsf{Var}_{P_{s,a,b}}(V_{\mathsf{pe}}^{-})$ and $\mathsf{Var}_{P_{s,a,b}}(V^\star)$ for all $(s,a,b)\in\mathcal{S}\times\mathcal{A}\times\mathcal{B}$.
\item As we have briefly mentioned in Remark \ref{remark:instance}, prior literature \citet{khamaru2021instance} that establishes instance optimality for the optimal value estimation problem in single-agent RL under a generative model requires one of the following two conditions: the optimal policy is unique, or a sample complexity bound that depends on an optimality gap
\begin{equation}
\Delta\coloneqq \min_{\pi\in\bm{\Pi}\setminus\bm{\Pi}^\star} \left\Vert Q^\star-r-\gamma P^\pi Q^\star \right\Vert_{\infty}, \label{eq:gap}
\end{equation}
where $Q^\star$ is the optimal Q-function, $\bm{\Pi}$ is the set of deterministic policies, and $\bm{\Pi}^\star$ is the set of optimal (deterministic) policies.
However, neither condition has a direct analog in two-player zero-sum Markov games: for the first one, this is because the Nash equilibrium of a zero-sum Markov game is not unique in general; for the second one, this is because the Nash equilibrium policy pair is usually random, and it is not clear how to define a non-zero optimality gap like \eqref{eq:gap}. 
\end{itemize}
In view of the above challenges, our current analysis framework remains incapable of deriving instance-optimal performance guarantees.  
Accomplishing instance-optimal results for zero-sum Markov games might require substantially more reinfed analysis techniques, 
and we leave this important direction to future investigation.

%% file: appendix_auxiliary_lemmas.tex
\section{Auxiliary lemmas for Theorem \ref{thm:main}}

\subsection{Proof for the preliminary facts in Appendix~\ref{sec:proof-outline-prelim}\label{sec:proof-proof-outline}}

We start by proving (\ref{eq:Q-V-bound}). For any $s\in\mathcal{S}$,
suppose that a NE of the matrix game $Q_{1}(s,\cdot,\cdot)$ is given
by $\big(\mu_{1}(s),\nu_{1}(s)\big)$, and that a NE of the matrix
game $Q_{2}(s,\cdot,\cdot)$ is $\big(\mu_{2}(s),\nu_{2}(s)\big)$.
Then we can derive
\begin{align*}
V_{1}\left(s\right)-V_{2}\left(s\right) & =\mathop{\mathbb{E}}\limits _{a\sim\mu_{1}(s),b\sim\nu_{1}(s)}\big[Q_{1}\left(s,a,b\right)\big]-\mathop{\mathbb{E}}\limits _{a\sim\mu_{2}(s),b\sim\nu_{2}(s)}\big[Q_{2}\left(s,a,b\right)\big]\\
 & \leq\mathop{\mathbb{E}}\limits _{a\sim\mu_{1}(s),b\sim\nu_{2}(s)}\big[Q_{1}\left(s,a,b\right)\big]-\mathop{\mathbb{E}}\limits _{a\sim\mu_{1}(s),b\sim\nu_{2}(s)}\big[Q_{2}\left(s,a,b\right)\big]\\
 & =\mathop{\mathbb{E}}\limits _{a\sim\mu_{1}(s),b\sim\nu_{2}(s)}\big[Q_{1}\left(s,a,b\right)-Q_{2}\left(s,a,b\right)\big]\leq\left\Vert Q_{1}-Q_{2}\right\Vert _{\infty},
\end{align*}
and similarly, 
\begin{align*}
V_{1}\left(s\right)-V_{2}\left(s\right) & \geq\mathop{\mathbb{E}}\limits _{a\sim\mu_{2}(s),b\sim\nu_{1}(s)}\big[Q_{1}\left(s,a,b\right)\big]-\mathop{\mathbb{E}}\limits _{a\sim\mu_{2}(s),b\sim\nu_{1}(s)}\big[Q_{2}\left(s,a,b\right)\big]\\
 & =\mathop{\mathbb{E}}\limits _{a\sim\mu_{2}(s),b\sim\nu_{1}(s)}\big[Q_{1}\left(s,a,b\right)-Q_{2}\left(s,a,b\right)\big]\geq-\left\Vert Q_{1}-Q_{2}\right\Vert _{\infty}.
\end{align*}
Taking the above two inequalities collectively yields
\[
\left|V_{1}\left(s\right)-V_{2}\left(s\right)\right|\leq\left\Vert Q_{1}-Q_{2}\right\Vert _{\infty}
\]
for any $s\in\mathcal{S}$, allowing us to conclude that
\[
\left\Vert V_{1}-V_{2}\right\Vert _{\infty}\leq\left\Vert Q_{1}-Q_{2}\right\Vert _{\infty}.
\]

Next, we turn attention to proving (\ref{eq:variance-V-bound}). Towards
this, we observe from \eqref{eq:var-Psab-V-defn} that
\begin{align}
\left|\mathsf{Var}_{P_{s,a,b}}\left(V_{1}\right)-\mathsf{Var}_{P_{s,a,b}}\left(V_{2}\right)\right| & =\left|P_{s,a,b}\left(V_{1}\circ V_{1}\right)-\left(P_{s,a,b}V_{1}\right)^{2}-P_{s,a,b}\left(V_{2}\circ V_{2}\right)+\left(P_{s,a,b}V_{2}\right)^{2}\right|\nonumber \\
 & \leq\left|P_{s,a,b}\left(V_{1}\circ V_{1}-V_{2}\circ V_{2}\right)\right|+\left|\left(P_{s,a,b}V_{1}\right)^{2}-\left(P_{s,a,b}V_{2}\right)^{2}\right|\nonumber \\
 & \leq\left|P_{s,a,b}\left[\left(V_{1}+V_{2}\right)\circ\left(V_{1}-V_{2}\right)\right]\right|+\big|P_{s,a,b}\left(V_{1}-V_{2}\right)\big|\cdot\big|P_{s,a,b}\left(V_{1}+V_{2}\right)\big|\nonumber \\
 & \leq\frac{2}{1-\gamma}\big|P_{s,a,b}\left(V_{1}-V_{2}\right)\big|+\frac{2}{1-\gamma}\big|P_{s,a,b}\left(V_{1}-V_{2}\right)\big|\nonumber \\
 & \leq\frac{4}{1-\gamma}\big\| P_{s,a,b}\big\|_{1}\left\Vert V_{1}-V_{2}\right\Vert _{\infty}=\frac{4}{1-\gamma}\left\Vert V_{1}-V_{2}\right\Vert _{\infty},\label{eq:proof-basic-facts-1}
\end{align}
whereas the validity of the last line is guaranteed since $P_{s,a,b}\in\Delta(\mathcal{S})$.

Finally, we present the proof of (\ref{eq:b-V-bound}). Consider first
the case with
\[
\max\left\{ \mathsf{Var}_{\widehat{P}_{s,a,b}}\left(V_{1}\right),\mathsf{Var}_{\widehat{P}_{s,a,b}}\left(V_{1}\right)\right\} <\frac{4C_{\mathsf{b}}\log\frac{N}{(1-\gamma)\delta}}{\left(1-\gamma\right)^{2}N\left(s,a,b\right)}.
\]
In this case, the penalty terms reduce to (cf.~\eqref{eq:bonus-term-defn})
\[
\beta\left(s,a,b;V_{1}\right)=\beta\left(s,a,b;V_{2}\right)=\min\left\{ \frac{2C_{\mathsf{b}}\log\frac{N}{(1-\gamma)\delta}}{\left(1-\gamma\right)N\left(s,a,b\right)},\frac{1}{1-\gamma}\right\} ,
\]
and as a result, 
\[
\left|\beta\left(s,a,b;V_{1}\right)-\beta\left(s,a,b;V_{2}\right)\right|=0\leq2\left\Vert V_{1}-V_{2}\right\Vert _{\infty}.
\]
In contrast, in the other case where
\begin{equation}
\max\left\{ \mathsf{Var}_{\widehat{P}_{s,a,b}}\left(V_{1}\right),\mathsf{Var}_{\widehat{P}_{s,a,b}}\left(V_{1}\right)\right\} \geq\frac{4C_{\mathsf{b}}\log\frac{N}{(1-\gamma)\delta}}{\left(1-\gamma\right)^{2}N\left(s,a,b\right)},\label{eq:proof-basic-facts-2}
\end{equation}
we can obtain
\begin{align*}
\left|\beta\left(s,a,b;V_{1}\right)-\beta\left(s,a,b;V_{2}\right)\right| & \leq\sqrt{\frac{C_{\mathsf{b}}\log\frac{N}{(1-\gamma)\delta}}{N\left(s,a,b\right)}}\left(\sqrt{\mathsf{Var}_{\widehat{P}_{s,a,b}}\left(V_{1}\right)}-\sqrt{\mathsf{Var}_{\widehat{P}_{s,a,b}}\left(V_{2}\right)}\right)\\
 & =\sqrt{\frac{C_{\mathsf{b}}\log\frac{N}{(1-\gamma)\delta}}{N\left(s,a,b\right)}}\frac{\mathsf{Var}_{\widehat{P}_{s,a,b}}\left(V_{1}\right)-\mathsf{Var}_{\widehat{P}_{s,a,b}}\left(V_{2}\right)}{\sqrt{\mathsf{Var}_{\widehat{P}_{s,a,b}}\left(V_{1}\right)}+\sqrt{\mathsf{Var}_{\widehat{P}_{s,a,b}}\left(V_{2}\right)}}\\
 & \leq\frac{1-\gamma}{2}\left[\mathsf{Var}_{\widehat{P}_{s,a,b}}\left(V_{1}\right)-\mathsf{Var}_{\widehat{P}_{s,a,b}}\left(V_{2}\right)\right]\\
 & \leq2\left\Vert V_{1}-V_{2}\right\Vert _{\infty}.
\end{align*}
Here, the penultimate relation follows from (\ref{eq:proof-basic-facts-2}),
while the last line takes advantage of (\ref{eq:proof-basic-facts-1}).
The two cases taken together establish (\ref{eq:b-V-bound}). 

\subsection{Proof of Lemma \ref{lemma:gamma-contraction}\label{sec:proof-lemma-gamma-contraction}}

For convenience of presentation, let us define\begin{subequations}
\begin{align}
\widetilde{\mathcal{T}}_{\mathsf{pe}}^{+}\left(Q\right)\left(s,a,b\right) & \coloneqq\widehat{r}\left(s,a,b\right)+\gamma\widehat{P}_{s,a,b}V+\beta\left(s,a,b;V\right),\label{eq:defn-tilde-P-pe-plus}\\
\widetilde{\mathcal{T}}_{\mathsf{pe}}^{-}\left(Q\right)\left(s,a,b\right) & \coloneqq\widehat{r}\left(s,a,b\right)+\gamma\widehat{P}_{s,a,b}V-\beta\left(s,a,b;V\right),\label{eq:defn-tilde-P-pe-minus}
\end{align}
\end{subequations}where $V$ is the value function associated to
$Q$ (see (\ref{eq:V-defn-Bellman}) for the definition). It is readily
seen that\begin{subequations}\label{eq:T-widehat-tilde-relation}
\begin{align}
\widehat{\mathcal{T}}_{\mathsf{pe}}^{+}\left(Q\right)\left(s,a,b\right) & =\min\left\{ \widetilde{\mathcal{T}}_{\mathsf{pe}}^{+}\left(Q\right)\left(s,a,b\right),\frac{1}{1-\gamma}\right\} ,\\
\widehat{\mathcal{T}}_{\mathsf{pe}}^{-}\left(Q\right)\left(s,a,b\right) & =\max\left\{ \widetilde{\mathcal{T}}_{\mathsf{pe}}^{-}\left(Q\right)\left(s,a,b\right),0\right\} .
\end{align}
\end{subequations}

\paragraph{Property 1: monotonicity.}

In view of (\ref{eq:T-widehat-tilde-relation}), it suffices to show
that $\widetilde{\mathcal{T}}_{\mathsf{pe}}^{+}$ and $\widetilde{\mathcal{T}}_{\mathsf{pe}}^{-}$
are monotone. For any $Q_{1}$ and $Q_{2}$ satisfying $0\leq Q_{1}\leq Q_{2}\leq\frac{1}{1-\gamma}1$,
we denote by $V_{1}$ (resp.~$V_{2}$) the value function corresponds
to $Q_{1}$ (resp.~$Q_{2}$); see (\ref{eq:V-defn-Bellman}) for
the precise definition. It is straightforward to check that for all
$s\in\mathcal{S}$, one has $0\leq V_{1},V_{2}\leq\frac{1}{1-\gamma}1$
and
\begin{align*}
V_{1}\left(s\right) & =\max_{\mu(s)\in\Delta(\mathcal{A})}\min_{\nu(s)\in\Delta(\mathcal{B})}\mathop{\mathbb{E}}\limits _{a\sim\mu(s),b\sim\nu(s)}\big[Q_{1}\left(s,a,b\right)\big]\\
 & \geq\max_{\mu(s)\in\Delta(\mathcal{A})}\min_{\nu(s)\in\Delta(\mathcal{B})}\mathop{\mathbb{E}}\limits _{a\sim\mu(s),b\sim\nu(s)}\big[Q_{2}\left(s,a,b\right)\big]=V_{2}\left(s\right).
\end{align*}
For any $(s,a,b)\in\mathcal{S}\times\mathcal{A}\times\mathcal{B}$,
define two functions $f^{+},f^{-}:\mathbb{R}^{S}\to\mathbb{R}$ as
\begin{align*}
f^{+}\left(V\right) & \coloneqq\widehat{r}\left(s,a,b\right)+\gamma\widehat{P}_{s,a,b}V+\beta\left(s,a,b;V\right)\\
 & =\widehat{r}\left(s,a,b\right)+\gamma\widehat{P}_{s,a,b}V+\min\left\{ \max\left\{ \sqrt{\frac{C_{\mathsf{b}}\log\frac{N}{\delta}}{N\left(s,a,b\right)}\mathsf{Var}_{\widehat{P}_{s,a,b}}\left(V\right)},\frac{2C_{\mathsf{b}}\log\frac{N}{\delta}}{\left(1-\gamma\right)N\left(s,a,b\right)}\right\} ,\frac{1}{1-\gamma}\right\} +\frac{4}{N};\\
f^{-}\left(V\right) & \coloneqq\widehat{r}\left(s,a,b\right)+\gamma\widehat{P}_{s,a,b}V-\beta\left(s,a,b;V\right)\\
 & =\widehat{r}\left(s,a,b\right)+\gamma\widehat{P}_{s,a,b}V-\min\left\{ \max\left\{ \sqrt{\frac{C_{\mathsf{b}}\log\frac{N}{\delta}}{N\left(s,a,b\right)}\mathsf{Var}_{\widehat{P}_{s,a,b}}\left(V\right)},\frac{2C_{\mathsf{b}}\log\frac{N}{\delta}}{\left(1-\gamma\right)N\left(s,a,b\right)}\right\} ,\frac{1}{1-\gamma}\right\} -\frac{4}{N}.
\end{align*}
Recognizing that $\widetilde{\mathcal{T}}_{\mathsf{pe}}^{-}(Q)(s,a,b)=f^{-}(V)$
and $\widetilde{\mathcal{T}}_{\mathsf{pe}}^{+}(Q)(s,a,b)=f^{+}(V)$,
we only need to prove that $f^{+}$ and $f^{-}$ are monotone in $V$,
namely, for any $V_{1}\geq V_{2}$, we have $f^{+}(V_{1})\geq f^{+}(V_{2})$
and $f^{-}(V_{1})\geq f^{-}(V_{2})$. Since both $f^{+}$ and $f^{-}$
are continuous, it suffices to check that $\nabla f^{+}(V)\geq0$
and $\nabla f^{-}(V)\geq0$ hold almost everywhere. 
\begin{itemize}
\item We first consider the case where
\[
\mathsf{Var}_{\widehat{P}_{s,a,b}}\left(V\right)<\frac{4C_{\mathsf{b}}\log\frac{N}{\delta}}{\left(1-\gamma\right)^{2}N\left(s,a,b\right)}.
\]
In this case, it is seen that
\begin{align*}
f^{+}\left(V\right) & =\widehat{r}\left(s,a,b\right)+\gamma\widehat{P}_{s,a,b}V+\min\left\{ \frac{2C_{\mathsf{b}}\log\frac{N}{\delta}}{\left(1-\gamma\right)N\left(s,a,b\right)},\frac{1}{1-\gamma}\right\} +\frac{4}{N},\\
f^{-}\left(V\right) & =\widehat{r}\left(s,a,b\right)+\gamma\widehat{P}_{s,a,b}V-\min\left\{ \frac{2C_{\mathsf{b}}\log\frac{N}{\delta}}{\left(1-\gamma\right)N\left(s,a,b\right)},\frac{1}{1-\gamma}\right\} -\frac{4}{N},
\end{align*}
which immediately gives
\[
\nabla f^{+}\left(V\right)=\nabla f^{-}\left(V\right)=\gamma\widehat{P}_{s,a,b}^{\top}\geq0.
\]
\item Next, we look at the case where
\begin{equation}
\frac{4C_{\mathsf{b}}\log\frac{N}{\delta}}{\left(1-\gamma\right)^{2}N\left(s,a,b\right)}<\mathsf{Var}_{\widehat{P}_{s,a,b}}\left(V\right)<\frac{N\left(s,a,b\right)}{C_{\mathsf{b}}\left(1-\gamma\right)^{2}\log\frac{N}{\delta}}.\label{eq:proof-contraction-6}
\end{equation}
In this case, we can easily derive
\begin{align*}
f^{+}\left(V\right) & =\widehat{r}\left(s,a,b\right)+\gamma\widehat{P}_{s,a,b}V+\sqrt{\frac{C_{\mathsf{b}}\log\frac{N}{\delta}}{N\left(s,a,b\right)}\mathsf{Var}_{\widehat{P}_{s,a,b}}\left(V\right)}+\frac{4}{N}\\
 & =\widehat{r}\left(s,a,b\right)+\gamma\widehat{P}_{s,a,b}V+\sqrt{\frac{C_{\mathsf{b}}\log\frac{N}{\delta}}{N\left(s,a,b\right)}\left[\widehat{P}_{s,a,b}\left(V\circ V\right)-\left(\widehat{P}_{s,a,b}V\right)^{2}\right]}+\frac{4}{N},
\end{align*}
and similarly,
\[
f^{-}\left(V\right)=\widehat{r}\left(s,a,b\right)+\gamma\widehat{P}_{s,a,b}V-\sqrt{\frac{C_{\mathsf{b}}\log\frac{N}{\delta}}{N\left(s,a,b\right)}\left[\widehat{P}_{s,a,b}\left(V\circ V\right)-\left(\widehat{P}_{s,a,b}V\right)^{2}\right]}-\frac{4}{N}.
\]
Recognizing that (\ref{eq:proof-contraction-6}) guarantees
\[
\mathsf{Var}_{\widehat{P}_{s,a,b}}\left(V\right)=\widehat{P}_{s,a,b}\left(V\circ V\right)-\big(\widehat{P}_{s,a,b}V\big)^{2}>0,
\]
we can invoke a little linear algebra to show that
\begin{align*}
\nabla f^{+}\left(V\right) & =\gamma\widehat{P}_{s,a,b}^{\top}+\sqrt{\frac{C_{\mathsf{b}}\log\frac{N}{\delta}}{N\left(s,a,b\right)}}\frac{\widehat{P}_{s,a,b}^{\top}\circ V-\big(\widehat{P}_{s,a,b}V\big)\widehat{P}_{s,a,b}^{\top}}{\sqrt{\mathsf{Var}_{\widehat{P}_{s,a,b}}\left(V\right)}}\\
 & \overset{\text{(i)}}{\geq}\gamma\widehat{P}_{s,a,b}^{\top}-\frac{1}{2\left(1-\gamma\right)}\big(\widehat{P}_{s,a,b}V\big)\widehat{P}_{s,a,b}^{\top}\overset{\text{(ii)}}{\geq}\gamma\widehat{P}_{s,a,b}^{\top}-\frac{1}{2}\widehat{P}_{s,a,b}^{\top}\overset{\text{(iii)}}{\geq}0,
\end{align*}
where (i) follows from (\ref{eq:proof-contraction-6}), (ii) holds
since $\big|\widehat{P}_{s,a,b}V\big|\leq\|V\|_{\infty}\leq\frac{1}{1-\gamma}$,
and (iii) is valid as long as $\gamma\geq1/2$. Similarly, we can
also deduce that
\begin{align*}
\nabla f^{-}\left(V\right) & =\gamma\widehat{P}_{s,a,b}^{\top}-\sqrt{\frac{C_{\mathsf{b}}\log\frac{N}{\delta}}{N\left(s,a,b\right)}}\frac{\widehat{P}_{s,a,b}^{\top}\circ V-\big(\widehat{P}_{s,a,b}V\big)\widehat{P}_{s,a,b}^{\top}}{\sqrt{\mathsf{Var}_{\widehat{P}_{s,a,b}}\left(V\right)}}\\
 & \overset{\text{(iv)}}{\geq}\gamma\widehat{P}_{s,a,b}^{\top}-\frac{1}{2\left(1-\gamma\right)}\widehat{P}_{s,a,b}^{\top}\circ V\overset{\text{(v)}}{\geq}\gamma\widehat{P}_{s,a,b}^{\top}-\frac{1}{2}\widehat{P}_{s,a,b}^{\top}\overset{\text{(vi)}}{\geq}0,
\end{align*}
where (iv) arises from (\ref{eq:proof-contraction-6}), (v) holds
due to $\|V\|_{\infty}\leq\frac{1}{1-\gamma}$, and (vi) is guaranteed
to hold when $\gamma\geq1/2$. 
\item Lastly, we are left with the case with
\[
\mathsf{Var}_{\widehat{P}_{s,a,b}}\left(V\right)>\frac{N\left(s,a,b\right)}{C_{\mathsf{b}}\left(1-\gamma\right)^{2}\log\frac{N}{\delta}},
\]
which necessarily satisfies
\begin{align*}
f^{+}\left(V\right) & =\widehat{r}\left(s,a,b\right)+\gamma\widehat{P}_{s,a,b}V+\frac{1}{1-\gamma}+\frac{4}{N},\\
f^{-}\left(V\right) & =\widehat{r}\left(s,a,b\right)+\gamma\widehat{P}_{s,a,b}V-\frac{1}{1-\gamma}-\frac{4}{N}.
\end{align*}
This immediately gives
\[
\nabla f^{+}\left(V\right)=\nabla f^{-}\left(V\right)=\gamma\widehat{P}_{s,a,b}^{\top}\geq0.
\]
\end{itemize}
Thus far, we have verified that $\nabla f^{+}(V)\geq0$ and $\nabla f^{-}(V)\geq0$
hold almost everywhere, except when $V$ is in a Legesgue zero-measure
set
\[
\left\{ V\in\mathbb{R}^{S}:\mathsf{Var}_{\widehat{P}_{s,a,b}}\left(V\right)=\frac{N\left(s,a,b\right)}{C_{\mathsf{b}}\left(1-\gamma\right)^{2}\log\frac{N}{\delta}}\quad\text{or}\quad\mathsf{Var}_{\widehat{P}_{s,a,b}}\left(V\right)=\frac{4C_{\mathsf{b}}\log\frac{N}{\delta}}{\left(1-\gamma\right)^{2}N\left(s,a,b\right)}\right\} .
\]
This in turn completes the proof of the claimed monotonicity property. 

\paragraph{Property 2: $\gamma$-contraction.}

We shall only prove the $\gamma$-contraction property for $\widehat{\mathcal{T}}_{\mathsf{pe}}^{+}$;
the proof w.r.t.~$\widehat{\mathcal{T}}_{\mathsf{pe}}^{-}$ follows
from an analogous argument and is hence omitted. In view of the identities
(\ref{eq:T-widehat-tilde-relation}), it is straightforward to check
that
\[
\left\Vert \widehat{\mathcal{T}}_{\mathsf{pe}}^{+}\left(Q_{1}\right)-\widehat{\mathcal{T}}_{\mathsf{pe}}^{+}\left(Q_{2}\right)\right\Vert _{\infty}\leq\left\Vert \widetilde{\mathcal{T}}_{\mathsf{pe}}^{+}\left(Q_{1}\right)-\widetilde{\mathcal{T}}_{\mathsf{pe}}^{+}\left(Q_{2}\right)\right\Vert _{\infty}
\]
for any $Q_{1},Q_{2}:\mathcal{S}\times\mathcal{A}\times\mathcal{B}\to\big[0,\frac{1}{1-\gamma}\big]$.
Therefore, it suffices to show that 
\begin{equation}
\left\Vert \widetilde{\mathcal{T}}_{\mathsf{pe}}^{+}\left(Q_{1}\right)-\widetilde{\mathcal{T}}_{\mathsf{pe}}^{+}\left(Q_{2}\right)\right\Vert _{\infty}\leq\gamma\left\Vert Q_{1}-Q_{2}\right\Vert _{\infty}.\label{eq:proof-contraction-1}
\end{equation}
Recalling that we have already shown above that $\widetilde{\mathcal{T}}_{\mathsf{pe}}^{+}$
is monotone, namely,
\begin{equation}
\widetilde{\mathcal{T}}_{\mathsf{pe}}^{+}\left(Q_{1}\right)\geq\widetilde{\mathcal{T}}_{\mathsf{pe}}^{+}\left(Q_{2}\right),\qquad\forall\,Q_{1}\geq Q_{2},\label{eq:proof-contraction-2}
\end{equation}
we can immediately apply it and the triangle inequality to obtain
\begin{align}
\widetilde{\mathcal{T}}_{\mathsf{pe}}^{+}\left(Q_{1}\right)-\widetilde{\mathcal{T}}_{\mathsf{pe}}^{+}\left(Q_{2}\right) & \leq\widetilde{\mathcal{T}}_{\mathsf{pe}}^{+}\big(Q_{2}+\left\Vert Q_{1}-Q_{2}\right\Vert _{\infty}1\big)-\widetilde{\mathcal{T}}_{\mathsf{pe}}^{+}\left(Q_{2}\right).\label{eq:proof-contraction-3}
\end{align}
Letting $V_{2}$ be the value function corresponding to $Q_{2}$,
we can straightforwardly check that the value function $V^{\mathsf{upper}}$
associated with the vector $Q^{\mathsf{upper}}\coloneqq Q_{2}+\Vert Q_{1}-Q_{2}\Vert_{\infty}1$
(cf.~\eqref{eq:V-defn-Bellman}) is given by
\begin{align*}
V^{\mathsf{upper}}(s) & =\max_{\mu(s)\in\Delta(\mathcal{A})}\min_{\nu(s)\in\Delta(\mathcal{B})}\mathop{\mathbb{E}}\limits _{a\sim\mu(s),b\sim\nu(s)}\big[Q_{2}\left(s,a,b\right)+\left\Vert Q_{1}-Q_{2}\right\Vert _{\infty}1\big]\\
 & =\left\Vert Q_{1}-Q_{2}\right\Vert _{\infty}+\max_{\mu(s)\in\Delta(\mathcal{A})}\min_{\nu(s)\in\Delta(\mathcal{B})}\mathop{\mathbb{E}}\limits _{a\sim\mu(s),b\sim\nu(s)}\left[Q_{2}\left(s,a,b\right)\right]\\
 & =V_{2}\left(s\right)+\left\Vert Q_{1}-Q_{2}\right\Vert _{\infty}.
\end{align*}
Additionally, we can further rewrite the penalty term as
\begin{align*}
\beta\left(s,a,b;V_{2}\right) & =\min\left\{ \max\left\{ \sqrt{\frac{C_{\mathsf{b}}\log\frac{N}{(1-\gamma)\delta}}{N\left(s,a,b\right)}\mathsf{Var}_{\widehat{P}_{s,a,b}}\left(V_{2}\right)},\frac{2C_{\mathsf{b}}\log\frac{N}{\delta}}{\left(1-\gamma\right)N\left(s,a,b\right)}\right\} ,\frac{1}{1-\gamma}\right\} +\frac{5}{N}\\
 & =\min\left\{ \max\left\{ \sqrt{\frac{C_{\mathsf{b}}\log\frac{N}{(1-\gamma)\delta}}{N\left(s,a,b\right)}\mathsf{Var}_{\widehat{P}_{s,a,b}}\left(V_{2}+\left\Vert Q_{1}-Q_{2}\right\Vert _{\infty}1\right)},\frac{2C_{\mathsf{b}}\log\frac{N}{\delta}}{\left(1-\gamma\right)N\left(s,a,b\right)}\right\} ,\frac{1}{1-\gamma}\right\} +\frac{5}{N}\\
 & =\beta\left(s,a,b;V_{2}+\left\Vert Q_{1}-Q_{2}\right\Vert _{\infty}1\right),
\end{align*}
which combined with \eqref{eq:defn-tilde-P-pe-plus} gives
\begin{equation}
\left\Vert \widetilde{\mathcal{T}}_{\mathsf{pe}}^{+}\left(Q_{2}+\left\Vert Q_{1}-Q_{2}\right\Vert _{\infty}1\right)-\widetilde{\mathcal{T}}_{\mathsf{pe}}^{+}\left(Q_{2}\right)\right\Vert _{\infty}=\gamma\left\Vert \widehat{P}_{s,a,b}\left(\left\Vert Q_{1}-Q_{2}\right\Vert _{\infty}1\right)\right\Vert _{\infty}=\gamma\left\Vert Q_{1}-Q_{2}\right\Vert _{\infty}.\label{eq:proof-contraction-4}
\end{equation}
Taking (\ref{eq:proof-contraction-3}) and (\ref{eq:proof-contraction-4})
collectively yields
\begin{equation}
\widetilde{\mathcal{T}}_{\mathsf{pe}}^{+}\left(Q_{1}\right)-\widetilde{\mathcal{T}}_{\mathsf{pe}}^{+}\left(Q_{2}\right)\leq\gamma\left\Vert Q_{1}-Q_{2}\right\Vert _{\infty}1.\label{eq:proof-contraction-5}
\end{equation}
Similarly, we can also show that
\[
\widetilde{\mathcal{T}}_{\mathsf{pe}}^{+}\left(Q_{1}\right)-\widetilde{\mathcal{T}}_{\mathsf{pe}}^{+}\left(Q_{2}\right)\geq\widetilde{\mathcal{T}}_{\mathsf{pe}}^{+}\left(Q_{2}-\left\Vert Q_{1}-Q_{2}\right\Vert _{\infty}1\right)-\widetilde{\mathcal{T}}_{\mathsf{pe}}^{+}\left(Q_{2}\right)\geq-\gamma\left\Vert Q_{1}-Q_{2}\right\Vert _{\infty}1.
\]
Combine (\ref{eq:proof-contraction-4}) and (\ref{eq:proof-contraction-5})
to establish (\ref{eq:proof-contraction-1}).

\paragraph{Property 3: existence and uniqueness of the fixed point.}

For any $0\leq Q\leq\frac{1}{1-\gamma}1$, we know that its associated
value function $V$ (cf.~\eqref{eq:V-defn-Bellman}) satisfies $0\leq V(s)\leq\frac{1}{1-\gamma}$
for any $s\in\mathcal{S}$. Therefore, it is seen that
\begin{align*}
\widetilde{\mathcal{T}}_{\mathsf{pe}}^{+}\left(Q\right)(s,a,b) & \geq\widehat{r}\left(s,a,b\right)+\gamma\widehat{P}_{s,a,b}V\geq0,\\
\widetilde{\mathcal{T}}_{\mathsf{pe}}^{-}\left(Q\right)(s,a,b) & \leq\widehat{r}\left(s,a,b\right)+\gamma\widehat{P}_{s,a,b}V\leq1+\frac{\gamma}{1-\gamma}\leq\frac{1}{1-\gamma}.
\end{align*}
Combining the above relations with (\ref{eq:T-widehat-tilde-relation})
gives
\[
0\leq\widehat{\mathcal{T}}_{\mathsf{pe}}^{+}(Q)(s,a,b)\leq\frac{1}{1-\gamma}\qquad\text{and}\qquad0\leq\widehat{\mathcal{T}}_{\mathsf{pe}}^{-}(Q)(s,a,b)\leq\frac{1}{1-\gamma}.
\]
This together with the $\gamma$-contraction property indicates that
$\widehat{\mathcal{T}}_{\mathsf{pe}}^{+}$ and $\widehat{\mathcal{T}}_{\mathsf{pe}}^{-}$
are both contraction mappings on the complete metric space $\big(\big[0,\frac{1}{1-\gamma}\big]^{SAB},\Vert\cdot\Vert_{\infty}\big)$.
In view of the Banach fixed point theorem (see, e.g., \citet[Theorem 3.1]{ciesielski2007stefan}),
both operators admit unique fixed points, which we shall denote by
$Q_{\mathsf{pe}}^{+\star}$ and $Q_{\mathsf{pe}}^{-\star}$ respectively.
Clearly, it holds that $0\leq Q_{\mathsf{pe}}^{+\star}\leq\frac{1}{1-\gamma}1$
and $0\leq Q_{\mathsf{pe}}^{-\star}\leq\frac{1}{1-\gamma}1$. 

\subsection{Proof of Lemma \ref{lemma:loo}\label{sec:proof-lemma-loo}}

The proof of Lemma \ref{lemma:loo} is motivated by the leave-one-out
analysis framework \citep{li2020breaking,agarwal2020model,li2022settling,cui2021minimax,pananjady2020instance}
previously proposed for single-agent MDPs. Here, we only prove the
results (\ref{eq:Bernstein-type-lemma-loo}) and (\ref{eq:Var-Phat-P-diff})
for all for all $\widetilde{V}\in\mathbb{R}^{S}$ satisfying $0\leq\widetilde{V}\leq\frac{1}{1-\gamma}1$
and $\Vert\widetilde{V}-V_{\mathsf{pe}}^{-\star}\Vert_{\infty}\leq1/N$;
for all $\widetilde{V}\in\mathbb{R}^{S}$ satisfying $0\leq\widetilde{V}\leq\frac{1}{1-\gamma}1$
and $\Vert\widetilde{V}-V_{\mathsf{pe}}^{+\star}\Vert_{\infty}\leq1/N$,
the results follow from an analogous argument and hence we omit the
proof for the sake of brevity.

\paragraph{Step 1: constructing auxiliary Markov Games.}

We first construct a set of auxiliary Markov games that facilitates
analysis. For any $s_{0}\in\mathcal{S}$ and any $u>0$, we define
a two-player zero-sum Markov game $\mathcal{MG}^{s_{0},u}=(\mathcal{S},\mathcal{A},\mathcal{B},P^{s_{0},u},r^{s_{0},u},\gamma)$
such that: 
\begin{itemize}
\item the transition kernel $P^{s_{0},u}:\mathcal{S}\times\mathcal{A}\times\mathcal{B}\rightarrow\Delta(\mathcal{S})$
is defined such that \begin{subequations}\label{eq:proof-loo-P}
\begin{align}
P^{s_{0},u}\left(s'\mymid s_{0},a,b\right) & =\ind\{s'=s_{0}\},\qquad\qquad\qquad\text{for all }\left(s',a,b\right)\in\mathcal{S}\times\mathcal{A}\times\mathcal{B},\\
P^{s_{0},u}\left(\cdot\mymid s,a,b\right) & =\widehat{P}\left(\cdot\mymid s,a,b\right),\qquad\,\,\;\quad\quad\quad\text{for all }\left(s,a,b\right)\in\mathcal{S}\times\mathcal{A}\times\mathcal{B}\text{ with }s\neq s_{0};
\end{align}
\end{subequations}we shall adopt the abbreviation $P_{s,a,b}^{s_{0},u}=P^{s_{0},u}(\cdot\mymid s,a,b)$
throughout this section; 
\item the reward function $r^{s_{0},u}$ is defined such that\begin{subequations}\label{eq:proof-loo-r}
\begin{align}
r^{s_{0},u}\left(s_{0},a,b\right) & =u+\min\left\{ \frac{2C_{\mathsf{b}}\log\frac{N}{\delta}}{\left(1-\gamma\right)N\left(s_{0},a,b\right)},\frac{1}{1-\gamma}\right\} ,\qquad\text{for all }\left(a,b\right)\in\mathcal{A}\times\mathcal{B},\\
r^{s_{0},u}\left(s,a,b\right) & =\widehat{r}\left(s,a,b\right),\qquad\qquad\qquad\quad\quad\quad\quad\;\;\text{for all }\left(s,a,b\right)\in\mathcal{S}\times\mathcal{A}\times\mathcal{B},\text{ and }s\neq s_{0}.
\end{align}
\end{subequations}
\end{itemize}
Given that $P^{s_{0},u}$ is constructed without using any sample
transition from $s_{0}$, the auxiliary MG $\mathcal{MG}^{s_{0},u}$
is statistically independent from any randomness resulting from transition
from state $s_{0}$, a key observation that plays a central role in
decoupling the statistical dependency. In addition, let us introduce
the corresponding pessimistic operators $\widehat{\mathcal{T}}_{\mathsf{pe}}^{-,s_{0},u}:\mathbb{R}^{SAB}\to\mathbb{R}^{SAB}$
for the max-player such that: for any $Q:\mathcal{S}\times\mathcal{A}\times\mathcal{B}\to\mathbb{R}$
and any $(s,a,b)\in\mathcal{S}\times\mathcal{A}\times\mathcal{B}$,
\begin{align}
\widehat{\mathcal{T}}_{\mathsf{pe}}^{-,s_{0},u}\left(Q\right)\left(s,a,b\right) & \coloneqq\max\left\{ r^{s_{0},u}\left(s,a,b\right)+\gamma P_{s,a,b}^{s_{0},u}V-\beta^{s_{0},u}\left(s,a,b;V\right),0\right\} .\label{eq:proof-loo-T}
\end{align}
Here, $V:\mathcal{S}\to\mathbb{R}$ is the value function associated
with the vector $Q$ (see (\ref{eq:V-defn-Bellman}) for the definition),
whereas
\begin{equation}
\beta^{s_{0},u}\left(s,a,b;V\right)\coloneqq\min\left\{ \max\left\{ \sqrt{\frac{C_{\mathsf{b}}\log\frac{N}{\delta}}{N\left(s,a,b\right)}\mathsf{Var}_{P_{s,a,b}^{s_{0},u}}\left(V\right)},\frac{2C_{\mathsf{b}}\log\frac{N}{\delta}}{\left(1-\gamma\right)N\left(s,a,b\right)}\right\} ,\frac{1}{1-\gamma}\right\} +\frac{4}{N}.\label{eq:proof-loo-0}
\end{equation}

\paragraph{Step 2: identifying auxiliary MGs related to the empirical MG $\widehat{\mathcal{MG}}$.}

With the above construction in place, we will show in this step that
by taking
\begin{align*}
u=u^{\star} & \coloneqq\left(1-\gamma\right)V_{\mathsf{pe}}^{-\star}\left(s_{0}\right)+\frac{4}{N},
\end{align*}
there exists a fixed point $Q_{s_{0},u^{\star}}^{-\star}$ of $\widehat{\mathcal{T}}_{\mathsf{pe}}^{-,s_{0},u^{\star}}$
whose corresponding value function $V_{s_{0},u^{\star}}^{-\star}$
coincides with $V_{\mathsf{pe}}^{-\star}$. To see why this is true,
define
\[
Q_{s_{0},u^{\star}}^{-\star}\left(s,a,b\right)\coloneqq\begin{cases}
V_{\mathsf{pe}}^{-\star}\left(s_{0}\right), & \text{for }s=s_{0}\text{ and all }\left(a,b\right)\in\mathcal{A}\times\mathcal{B}\\
Q_{\mathsf{pe}}^{-\star}\left(s,a,b\right),\quad & \text{for all }\left(s,a,b\right)\in\mathcal{S}\times\mathcal{A}\times\mathcal{B}\text{ with }s\neq s_{0}.
\end{cases}
\]
It is straightforward to check that the value function associated
with $Q_{s_{0},u^{\star}}^{-\star}$ is precisely $V_{\mathsf{pe}}^{-\star}$.
Additionally, let us calculate $\widehat{\mathcal{T}}_{\mathsf{pe}}^{-,s_{0},u^{\star}}\left(Q_{\mathsf{pe}}^{-\star}\right)$. 
\begin{itemize}
\item For state $s_{0}$ and any action pair $(a,b)\in\mathcal{A}\times\mathcal{B}$,
it holds that
\begin{align}
 & \widehat{\mathcal{T}}_{\mathsf{pe}}^{-,s_{0},u^{\star}}\left(Q_{\mathsf{pe}}^{-\star}\right)\left(s_{0},a,b\right)=\max\left\{ r^{s_{0},u^{\star}}\left(s_{0},a,b\right)+\gamma P_{s_{0},a,b}^{s_{0},u^{\star}}V_{\mathsf{pe}}^{-\star}-\beta^{s_{0},u^{\star}}\left(s_{0},a,b;V_{\mathsf{pe}}^{-\star}\right),0\right\} \nonumber \\
 & \quad=\max\left\{ u^{\star}+\min\left\{ \frac{2C_{\mathsf{b}}\log\frac{N}{\delta}}{\left(1-\gamma\right)N\left(s,a,b\right)},\frac{1}{1-\gamma}\right\} +\gamma V_{\mathsf{pe}}^{-\star}\left(s_{0}\right)-\min\left\{ \frac{2C_{\mathsf{b}}\log\frac{N}{\delta}}{\left(1-\gamma\right)N\left(s,a,b\right)},\frac{1}{1-\gamma}\right\} -\frac{4}{N},0\right\} \nonumber \\
 & \quad=\max\left\{ \left(1-\gamma\right)V_{\mathsf{pe}}^{-\star}\left(s_{0}\right)+\frac{4}{N}+\gamma V_{\mathsf{pe}}^{-\star}\left(s_{0}\right)-\frac{4}{N},0\right\} =V_{\mathsf{pe}}^{-\star}\left(s_{0}\right)=Q_{\mathsf{pe}}^{-\star}\left(s_{0},a,b\right).\label{eq:proof-loo-1}
\end{align}
To see why the second equality holds, note that $P^{s_{0},u}(s'\mymid s_{0},a,b)=\ind\{s'=s_{0}\}$,
and therefore $P_{s_{0},a,b}^{s_{0},u^{\star}}V_{\mathsf{pe}}^{-\star}=V_{\mathsf{pe}}^{-\star}\left(s_{0}\right)$
and $\mathsf{Var}_{P_{s_{0},a,b}^{s_{0},u^{\star}}}(V_{\mathsf{pe}}^{-\star})=0$,
which in turn allows us to simplify the expression of $\beta^{s_{0},u^{\star}}\left(s_{0},a,b;V_{\mathsf{pe}}^{-\star}\right)$
in \eqref{eq:proof-loo-0} to $\min\left\{ \frac{2C_{\mathsf{b}}\log\frac{N}{\delta}}{\left(1-\gamma\right)N\left(s,a,b\right)},\frac{1}{1-\gamma}\right\} +\frac{4}{N}$. 
\item For any state $s\neq s_{0}$ and any action pairs $(a,b)\in\mathcal{A}\times\mathcal{B}$,
we have
\begin{align}
\widehat{\mathcal{T}}_{\mathsf{pe}}^{-,s_{0},u^{\star}}\left(Q_{\mathsf{pe}}^{-\star}\right)\left(s,a,b\right) & =\max\left\{ r^{s_{0},u^{\star}}\left(s,a,b\right)+\gamma P_{s,a,b}^{s_{0},u^{\star}}V_{\mathsf{pe}}^{-\star}-\beta^{s_{0},u^{\star}}\left(s,a,b;V_{\mathsf{pe}}^{-\star}\right),0\right\} \nonumber \\
 & \overset{\text{(i)}}{=}\max\left\{ \widehat{r}\left(s,a,b\right)+\gamma\widehat{P}_{s,a,b}V_{\mathsf{pe}}^{-\star}-\beta\left(s,a,b;V_{\mathsf{pe}}^{-\star}\right),0\right\} \nonumber \\
 & =\widetilde{\mathcal{T}}_{\mathsf{pe}}^{-}\left(Q_{\mathsf{pe}}^{-\star}\right)\left(s,a,b\right)\overset{\text{(ii)}}{=}Q_{\mathsf{pe}}^{-\star}\left(s,a,b\right).\label{eq:proof-loo-2}
\end{align}
Here, (i) holds since $r^{s_{0},u^{\star}}(s,a,b)=\widehat{r}(s,a,b)$
and $P^{s_{0},u}(\cdot\mymid s,a,b)=\widehat{P}_{s,a,b}(\cdot\mymid s,a,b)$
whenever $s\neq s_{0}$; (ii) follows from the fact that $Q_{\mathsf{pe}}^{-\star}$
is the fixed point of $\widehat{\mathcal{T}}_{\mathsf{pe}}^{-,s_{0},u^{\star}}$. 
\end{itemize}
Taking (\ref{eq:proof-loo-1}) and (\ref{eq:proof-loo-2}) collectively
confirms that $Q_{s_{0},u^{\star}}^{-\star}$ is a fixed point of
$\widehat{\mathcal{T}}_{\mathsf{pe}}^{-,s_{0},u^{\star}}$.

\paragraph{Step 3: constructing an $\varepsilon$-net.}

Before proceeding, we note that repeating the same analysis in the
proof of Lemma \ref{lemma:gamma-contraction} (cf.~Appendix~\ref{sec:proof-lemma-gamma-contraction})
reveals that: for any $s_{0}\in\mathcal{S}$ and $u>0$, the operator
$\widehat{\mathcal{T}}_{\mathsf{pe}}^{-,s_{0},u}$ admits a unique
fixed point $Q_{s_{0},u}^{-\star}$ satisfying $0\leq Q_{s_{0},u}^{-\star}(s,a,b)\leq\frac{1}{1-\gamma}$
for all $(s,a,b)\in\mathcal{S}\times\mathcal{A}\times\mathcal{B}$;
we shall let $V_{s_{0},u}^{-\star}$ be the value function associated
with $Q_{s_{0},u}^{-\star}$ (see (\ref{eq:V-defn-Bellman}) for the
definition) in what follows. We also note that all the probabilistic
arguments below are conditional on $\{N(s,a,b):(s,a,b)\in\mathcal{S}\times\mathcal{A}\times\mathcal{B}\}$. 

The main purpose of this step is to establish the desired concentration
bounds for a proper covering of a range of $u$ of interest. Towards
this end, we first make note of the following useful Bernstein-style
concentration bound. 

\begin{claim}\label{claim:loo-fixed-concentration}For any $(a,b)\in\mathcal{A}\times\mathcal{B}$
and any given function $V:\mathcal{S}\to[0,\frac{1}{1-\gamma}]$ independent
of $\widehat{P}_{s_{0},a,b}$, with probability exceeding $1-\delta$
we have
\begin{equation}
\left|\big(\widehat{P}_{s_{0},a,b}-P_{s_{0},a,b}\big)V\right|\leq C_{0}\sqrt{\frac{1}{N\left(s_{0},a,b\right)}\mathsf{Var}_{\widehat{P}_{s_{0},a,b}}\left(V\right)\log\frac{1}{\delta}}+\frac{C_{0}\log\frac{1}{\delta}}{\left(1-\gamma\right)N\left(s_{0},a,b\right)}\label{eq:proof-loo-claim-1}
\end{equation}
and
\begin{equation}
\mathsf{Var}_{\widehat{P}_{s_{0},a,b}}\left(V\right)\leq2\mathsf{Var}_{P_{s_{0},a,b}}\left(V\right)+O\left(\frac{1}{\left(1-\gamma\right)^{2}N\left(s_{0},a,b\right)}\log\frac{1}{\delta}\right)\label{eq:proof-loo-claim-2}
\end{equation}
for some universal constant $C_{0}>0$.\end{claim}\begin{proof}See
Appendix \ref{subsec:proof-claim-loo-fixed}.\end{proof}

In order to extend this result to those vectors $V$ that might be
statistically dependent on $\widehat{P}_{s_{0},a,b}$, we construct
an $\epsilon$-net (see, e.g., \citet{vershynin2018high}) with $\epsilon=1/N$
for the interval $[\frac{4}{N},1+\frac{4}{N}]$ as follows: 
\[
\mathcal{N}\coloneqq\left\{ \frac{i}{N}:4\leq i\leq N+3\right\} .
\]
It is clearly seen that
\[
u^{\star}\in\bigg[\frac{4}{N},1+\frac{4}{N}\bigg]\qquad\text{and}\qquad\vert\mathcal{N}\vert=N.
\]
Given that $V_{s_{0},u}^{-\star}$ is statistically independent of
$\widehat{P}_{s_{0},a,b}$ for any $u\in\mathcal{N}$ and any $(a,b)\in\mathcal{A}\times\mathcal{B}$,
we can readily apply Claim \ref{claim:loo-fixed-concentration} (replacing
$\delta$ with $\delta/N$) in conjunction with the union bound to
show that: for any $(a,b)\in\mathcal{A}\times\mathcal{B}$, with probability
exceeding $1-\delta$, 
\begin{equation}
\left|\big(\widehat{P}_{s_{0},a,b}-P_{s_{0},a,b}\big)V_{s_{0},u}^{-\star}\right|\leq C_{0}\sqrt{\frac{1}{N\left(s_{0},a,b\right)}\mathsf{Var}_{\widehat{P}_{s_{0},a,b}}\big(V_{s_{0},u}^{-\star}\big)\log\frac{N}{\delta}}+\frac{C_{0}\log\frac{N}{\delta}}{\left(1-\gamma\right)N\left(s_{0},a,b\right)}\label{eq:proof-loo-3}
\end{equation}
and
\begin{equation}
\mathsf{Var}_{\widehat{P}_{s_{0},a,b}}\big(V_{s_{0},u}^{-\star}\big)\leq2\mathsf{Var}_{P_{s_{0},a,b}}\big(V_{s_{0},u}^{-\star}\big)+O\left(\frac{1}{\left(1-\gamma\right)^{2}N\left(s_{0},a,b\right)}\log\frac{N}{\delta}\right)\label{eq:proof-loo-4}
\end{equation}
hold simultaneously for all $u\in\mathcal{N}$.

\paragraph{Step 4: a covering argument.}

From the definition of $\mathcal{N}$, there exists a point $u_{0}\in\mathcal{N}$
such that $\vert u_{0}-u^{\star}\vert\leq1/N$. In view of (\ref{eq:proof-loo-T}),
(\ref{eq:proof-loo-P}), (\ref{eq:proof-loo-r}) and (\ref{eq:proof-loo-0}),
for any $Q:\mathbb{R}^{SAB}\to\mathbb{R}^{SAB}$ one has
\begin{equation}
\big\Vert\widehat{\mathcal{T}}_{\mathsf{pe}}^{-,s_{0},u^{\star}}\left(Q\right)-\widehat{\mathcal{T}}_{\mathsf{pe}}^{-,s_{0},u_{0}}\left(Q\right)\big\Vert_{\infty}\leq\left|u^{\star}-u_{0}\right|\leq\frac{1}{N}.\label{eq:proof-loo-5}
\end{equation}
Akin to Lemma \ref{lemma:gamma-contraction} (which shows that $\widehat{\mathcal{T}}_{\mathsf{pe}}^{-}$
is $\gamma$-contractive), we can show that $\widehat{\mathcal{T}}_{\mathsf{pe}}^{-,s_{0},u}$
is also $\gamma$-contractive for all $u>0$ using the same analysis.
Consequently,
\begin{align*}
\left\Vert Q_{s_{0},u_{0}}^{-\star}-Q_{s_{0},u^{\star}}^{-\star}\right\Vert _{\infty} & =\big\Vert\widehat{\mathcal{T}}_{\mathsf{pe}}^{-,s_{0},u_{0}}\left(Q_{s_{0},u_{0}}^{-\star}\right)-\widehat{\mathcal{T}}_{\mathsf{pe}}^{-,s_{0},u^{\star}}\left(Q_{s_{0},u^{\star}}^{-\star}\right)\big\Vert_{\infty}\\
 & \leq\big\Vert\widehat{\mathcal{T}}_{\mathsf{pe}}^{-,s_{0},u_{0}}\left(Q_{s_{0},u_{0}}^{-\star}\right)-\widehat{\mathcal{T}}_{\mathsf{pe}}^{-,s_{0},u^{\star}}\left(Q_{s_{0},u_{0}}^{-\star}\right)\big\Vert_{\infty}+\big\Vert\widehat{\mathcal{T}}_{\mathsf{pe}}^{-,s_{0},u^{\star}}\left(Q_{s_{0},u_{0}}^{-\star}\right)-\widehat{\mathcal{T}}_{\mathsf{pe}}^{-,s_{0},u^{\star}}\left(Q_{s_{0},u^{\star}}^{-\star}\right)\big\Vert_{\infty}\\
 & \leq\frac{1}{N}+\gamma\left\Vert Q_{s_{0},u_{0}}^{-\star}-Q_{s_{0},u^{\star}}^{-\star}\right\Vert _{\infty},
\end{align*}
which immediately gives
\[
\left\Vert Q_{s_{0},u_{0}}^{-\star}-Q_{s_{0},u^{\star}}^{-\star}\right\Vert _{\infty}\leq\frac{1}{N\left(1-\gamma\right)}.
\]
Invoke (\ref{eq:Q-V-bound}) to show that
\[
\left\Vert V_{s_{0},u_{0}}^{-\star}-V_{\mathsf{pe}}^{-\star}\right\Vert _{\infty}=\left\Vert V_{s_{0},u_{0}}^{-\star}-V_{s_{0},u^{\star}}^{-\star}\right\Vert _{\infty}\leq\left\Vert Q_{s_{0},u_{0}}^{-\star}-Q_{s_{0},u^{\star}}^{-\star}\right\Vert _{\infty}\leq\frac{1}{N\left(1-\gamma\right)}.
\]
As a result, for all $\widetilde{V}$ satisfying $0\leq\widetilde{V}\leq\frac{1}{1-\gamma}1$
and $\Vert\widetilde{V}-V_{\mathsf{pe}}^{-\star}\Vert_{\infty}\leq1/N$,
the triangle inequality tells us that
\begin{equation}
\big\Vert V_{s_{0},u_{0}}^{-\star}-\widetilde{V}\big\Vert_{\infty}\leq\left\Vert V_{s_{0},u_{0}}^{-\star}-V_{\mathsf{pe}}^{-\star}\right\Vert _{\infty}+\big\Vert V_{\mathsf{pe}}^{-\star}-\widetilde{V}\big\Vert_{\infty}\leq\frac{2}{N\left(1-\gamma\right)}.\label{eq:proof-loo-6}
\end{equation}
In addition, we can use (\ref{eq:variance-V-bound}) to show that
\begin{align}
\big|\mathsf{Var}_{\widehat{P}_{s_{0},a,b}}\left(V_{s_{0},u_{0}}^{-\star}\right)-\mathsf{Var}_{\widehat{P}_{s_{0},a,b}}\big(\widetilde{V}\big)\big| & \leq\frac{4}{1-\gamma}\big\Vert V_{s_{0},u_{0}}^{-\star}-\widetilde{V}\big\Vert_{\infty}\leq\frac{8}{N\left(1-\gamma\right)^{2}},\label{eq:proof-loo-7}\\
\big|\mathsf{Var}_{P_{s_{0},a,b}}\left(V_{s_{0},u_{0}}^{-\star}\right)-\mathsf{Var}_{P_{s_{0},a,b}}\big(\widetilde{V}\big)\big| & \leq\frac{4}{1-\gamma}\big\Vert V_{s_{0},u_{0}}^{-\star}-\widetilde{V}\big\Vert_{\infty}\leq\frac{8}{N\left(1-\gamma\right)^{2}}.\label{eq:proof-loo-8}
\end{align}
Therefore, for any $(a,b)\in\mathcal{A}\times\mathcal{B}$, with probability
exceeding $1-\delta$ one has 
\begin{align}
\big|\big(\widehat{P}_{s_{0},a,b}-P_{s_{0},a,b}\big)\widetilde{V}\big| & \leq\big|\big(\widehat{P}_{s_{0},a,b}-P_{s_{0},a,b}\big)V_{s_{0},u_{0}}^{-\star}\big|+\big|\big(\widehat{P}_{s_{0},a,b}-P_{s_{0},a,b}\big)\big(\widetilde{V}-V_{s_{0},u_{0}}^{-\star}\big)\big|\nonumber \\
 & \leq\left|\big(\widehat{P}_{s_{0},a,b}-P_{s_{0},a,b}\big)V_{s_{0},u_{0}}^{-\star}\right|+\big\Vert\widehat{P}_{s_{0},a,b}\big\Vert_{1}\big\Vert V_{s_{0},u_{0}}^{-\star}-\widetilde{V}\big\Vert_{\infty}+\left\Vert P_{s_{0},a,b}\right\Vert _{1}\big\Vert V_{s_{0},u_{0}}^{-\star}-\widetilde{V}\big\Vert_{\infty}\nonumber \\
 & \overset{\text{(i)}}{\leq}C_{0}\sqrt{\frac{1}{N\left(s_{0},a,b\right)}\mathsf{Var}_{\widehat{P}_{s_{0},a,b}}\big(V_{s_{0},u_{0}}^{-\star}\big)\log\frac{N}{\delta}}+\frac{C_{0}\log\frac{N}{\delta}}{\left(1-\gamma\right)N\left(s_{0},a,b\right)}+\frac{2}{N\left(1-\gamma\right)}\nonumber \\
 & \overset{\text{(ii)}}{\leq}C_{0}\sqrt{\frac{1}{N\left(s_{0},a,b\right)}\left[\mathsf{Var}_{\widehat{P}_{s_{0},a,b}}\big(\widetilde{V}\big)+\frac{8}{N\left(1-\gamma\right)^{2}}\right]\log\frac{N}{\delta}}+\frac{\left(C_{0}+2\right)\log\frac{N}{\delta}}{\left(1-\gamma\right)N\left(s_{0},a,b\right)}\nonumber \\
 & \overset{\text{(iii)}}{\leq}C_{0}\sqrt{\frac{1}{N\left(s_{0},a,b\right)}\mathsf{Var}_{\widehat{P}_{s_{0},a,b}}\big(\widetilde{V}\big)\log\frac{N}{\delta}}+\frac{\left(4C_{0}+2\right)\log\frac{N}{\delta}}{\left(1-\gamma\right)N\left(s_{0},a,b\right)}\label{eq:proof-loo-9}
\end{align}
holds for all $\widetilde{V}$ satisfying $0\leq\widetilde{V}\leq\frac{1}{1-\gamma}1$
and $\Vert\widetilde{V}-V_{\mathsf{pe}}^{-\star}\Vert_{\infty}\leq1/N$.
Here (i) follows from (\ref{eq:proof-loo-3}), (\ref{eq:proof-loo-6})
and the facts that $\Vert\widehat{P}_{s_{0},a,b}\Vert_{1}=\Vert P_{s_{0},a,b}\Vert_{1}=1$;
(ii) utilizes (\ref{eq:proof-loo-7}) and the fact that $N(s_{0},a,b)\leq N$;
and (iii) holds since $N\geq N(s_{0},a,b)$. In addition, it is also
seen that
\begin{align}
\mathsf{Var}_{\widehat{P}_{s_{0},a,b}}\big(\widetilde{V}\big) & =\mathsf{Var}_{\widehat{P}_{s_{0},a,b}}\big(V_{s_{0},u_{0}}^{-\star}\big)+\mathsf{Var}_{\widehat{P}_{s_{0},a,b}}\big(\widetilde{V}\big)-\mathsf{Var}_{\widehat{P}_{s_{0},a,b}}\big(V_{s_{0},u_{0}}^{-\star}\big)\nonumber \\
 & \overset{\text{(i)}}{\leq}2\mathsf{Var}_{P_{s_{0},a,b}}\big(V_{s_{0},u_{0}}^{-\star}\big)+O\left(\frac{1}{\left(1-\gamma\right)^{2}N\left(s_{0},a,b\right)}\log\frac{N}{\delta}\right)+\frac{8}{N\left(1-\gamma\right)^{2}}\nonumber \\
 & \overset{\text{(ii)}}{\leq}2\mathsf{Var}_{P_{s_{0},a,b}}\big(\widetilde{V}\big)+O\left(\frac{1}{\left(1-\gamma\right)^{2}N\left(s_{0},a,b\right)}\log\frac{N}{\delta}\right)+\frac{24}{N\left(1-\gamma\right)^{2}}\nonumber \\
 & \overset{\text{(iii)}}{=}2\mathsf{Var}_{P_{s_{0},a,b}}\big(\widetilde{V}\big)+O\left(\frac{1}{\left(1-\gamma\right)^{2}N\left(s_{0},a,b\right)}\log\frac{N}{\delta}\right).\label{eq:proof-loo-10}
\end{align}
Here, (i) follows from (\ref{eq:proof-loo-4}) and (\ref{eq:proof-loo-7});
(ii) utilizes (\ref{eq:proof-loo-8}); and (iii) holds since $N\geq N(s_{0},a,b)$.
By taking the union bound over all $(s_{0},a,b)\in\mathcal{S}\times\mathcal{A}\times\mathcal{B}$
satisfying $N(s_{0},a,b)>0$ (notice that the number of such state-action
pairs cannot exceed $N$), we conclude from (\ref{eq:proof-loo-9})
and (\ref{eq:proof-loo-10}) (replacing $\delta$ with $\delta/N$)
that with probability exceeding $1-\delta$,
\begin{align*}
\big|\big(\widehat{P}_{s_{0},a,b}-P_{s_{0},a,b}\big)\widetilde{V}\big| & \leq C_{0}\sqrt{\frac{1}{N\left(s_{0},a,b\right)}\mathsf{Var}_{\widehat{P}_{s_{0},a,b}}\big(\widetilde{V}\big)\log\frac{N^{2}}{\delta}}+\frac{\left(4C_{0}+2\right)\log\frac{N^{2}}{\delta}}{\left(1-\gamma\right)N\left(s_{0},a,b\right)}\\
 & \leq2C_{0}\sqrt{\frac{1}{N\left(s_{0},a,b\right)}\mathsf{Var}_{\widehat{P}_{s_{0},a,b}}\big(\widetilde{V}\big)\log\frac{N}{\delta}}+\frac{\left(8C_{0}+4\right)\log\frac{N}{\delta}}{\left(1-\gamma\right)N\left(s_{0},a,b\right)}
\end{align*}
and
\begin{align*}
\mathsf{Var}_{\widehat{P}_{s_{0},a,b}}\big(\widetilde{V}\big) & \leq2\mathsf{Var}_{P_{s_{0},a,b}}\big(\widetilde{V}\big)+O\left(\frac{1}{\left(1-\gamma\right)^{2}N\left(s,a,b\right)}\log\frac{N^{2}}{\delta}\right)\\
 & =2\mathsf{Var}_{P_{s_{0},a,b}}\big(\widetilde{V}\big)+O\left(\frac{1}{\left(1-\gamma\right)^{2}N\left(s,a,b\right)}\log\frac{N}{\delta}\right)
\end{align*}
hold for all $(s,a,b)\in\mathcal{S}\times\mathcal{A}\times\mathcal{B}$
satisfying $N(s_{0},a,b)>0$ and all $\widetilde{V}$ obeying $0\leq\widetilde{V}\leq\frac{1}{1-\gamma}1$
and $\Vert\widetilde{V}-V_{\mathsf{pe}}^{-\star}\Vert_{\infty}\leq1/N$.

\subsubsection{Proof of Claim \ref{claim:loo-fixed-concentration} \label{subsec:proof-claim-loo-fixed}}

\paragraph{Part 1: proof of inequality (\ref{eq:proof-loo-claim-1}).}

Conditional on $N(s_{0},a,b)$ and $V$, we can write
\[
\big(\widehat{P}_{s_{0},a,b}-P_{s_{0},a,b}\big)V=\sum_{s'\in\mathcal{S}}\underbrace{V\left(s'\right)\left[\frac{\sum_{i=1}^{N}\ind\left\{ s_{i}=s_{0},a_{i}=a,b_{i}=b,s_{i}'=s'\right\} }{N\left(s_{0},a,b\right)}-P\left(s'\mymid s_{0},a,b\right)\right]}_{\eqqcolon X_{s'}}
\]
as a sum of independent random variables. It is straightforward to
check that $\mathbb{E}[X_{s'}]=0$ and $\vert X_{s'}\vert\leq\frac{1}{1-\gamma}$
for all $s'\in\mathcal{S}$. Apply the Bernstein inequality \citep[Theorem 2.8.4]{vershynin2018high}
to show that with probability exceeding $1-\delta$,
\begin{equation}
\big|\big(\widehat{P}_{s_{0},a,b}-P_{s_{0},a,b}\big)V\big|\leq C_{\mathsf{bern}}\sqrt{\frac{1}{N\left(s_{0},a,b\right)}\mathsf{Var}_{P_{s_{0},a,b}}\left(V\right)\log\frac{1}{\delta}}+\frac{C_{\mathsf{bern}}\log\frac{1}{\delta}}{\left(1-\gamma\right)N\left(s_{0},a,b\right)}\label{eq:proof-loo-fixed-11}
\end{equation}
for some universal constant $C_{\mathsf{bern}}>0$. 

Next, denoting
\[
\overline{V}\coloneqq V-\left(P_{s,a,b}V\right)1,
\]
we observe that
\begin{align}
\mathsf{Var}_{P_{s_{0},a,b}}\left(V\right) & =P_{s_{0},a,b}\big(\overline{V}\circ\overline{V}\big)=\widehat{P}_{s_{0},a,b}\big(\overline{V}\circ\overline{V}\big)+\big(P_{s_{0},a,b}-\widehat{P}_{s_{0},a,b}\big)\big(\overline{V}\circ\overline{V}\big)\nonumber \\
 & =\mathsf{Var}_{\widehat{P}_{s_{0},a,b}}\left(V\right)+\big[\big(P_{s_{0},a,b}-\widehat{P}_{s_{0},a,b}\big)V\big]^{2}+\big(P_{s_{0},a,b}-\widehat{P}_{s_{0},a,b}\big)\big(\overline{V}\circ\overline{V}\big),\label{eq:proof-loo-fixed-3}
\end{align}
where we have used the fact that
\begin{align*}
\widehat{P}_{s_{0},a,b}\big(\overline{V}\circ\overline{V}\big) & =\widehat{P}_{s_{0},a,b}\left(\big[V-\left(P_{s_{0},a,b}V\right)1\big]\circ\big[V-\left(P_{s_{0},a,b}V\right)1\big]\right)\\
 & =\widehat{P}_{s_{0},a,b}\left(V\circ V\right)-2\left(P_{s_{0},a,b}V\right)\left(\widehat{P}_{s_{0},a,b}V\right)+\left(P_{s_{0},a,b}V\right)^{2}\\
 & =\widehat{P}_{s_{0},a,b}\left(\big[V-\left(\widehat{P}_{s_{0},a,b}V\right)1\big]\circ\big[V-\left(\widehat{P}_{s_{0},a,b}V\right)1\big]\right)+\left(\widehat{P}_{s_{0},a,b}V\right)^{2}\\
 & \qquad-2\left(P_{s_{0},a,b}V\right)\left(\widehat{P}_{s_{0},a,b}V\right)+\left(P_{s_{0},a,b}V\right)^{2}\\
 & =\mathsf{Var}_{\widehat{P}_{s_{0},a,b}}\left(V\right)+\big[\big(P_{s_{0},a,b}-\widehat{P}_{s_{0},a,b}\big)V\big]^{2}.
\end{align*}
Similar to (\ref{eq:proof-loo-fixed-11}), we can show that with probability
exceeding $1-\delta$,
\begin{align}
\big|\big(\widehat{P}_{s_{0},a,b}-P_{s_{0},a,b}\big)\big(\overline{V}\circ\overline{V}\big)\big| & \leq C_{\mathsf{bern}}\sqrt{\frac{1}{N\left(s_{0},a,b\right)}\mathsf{Var}_{P_{s_{0},a,b}}\big(\overline{V}\circ\overline{V}\big)\log\frac{1}{\delta}}+\frac{C_{\mathsf{bern}}\log\frac{1}{\delta}}{\left(1-\gamma\right)^{2}N\left(s_{0},a,b\right)}\nonumber \\
 & \leq C_{\mathsf{bern}}\sqrt{\frac{1}{\left(1-\gamma\right)^{2}N\left(s_{0},a,b\right)}\mathsf{Var}_{P_{s_{0},a,b}}\big(V\big)\log\frac{1}{\delta}}+\frac{C_{\mathsf{bern}}\log\frac{1}{\delta}}{\left(1-\gamma\right)^{2}N\left(s_{0},a,b\right)},\label{eq:proof-loo-fixed-4}
\end{align}
where the second relation holds since
\begin{align*}
\mathsf{Var}_{P_{s_{0},a,b}}\big(\overline{V}\circ\overline{V}\big) & \leq P_{s_{0},a,b}\big(\overline{V}\circ\overline{V}\circ\overline{V}\circ\overline{V}\big)\leq\frac{1}{(1-\gamma)^{2}}P_{s_{0},a,b}\big(\overline{V}\circ\overline{V}\big)\\
 & =\frac{1}{\left(1-\gamma\right)^{2}}\mathsf{Var}_{P_{s_{0},a,b}}\left(V\right).
\end{align*}
Equipped with (\ref{eq:proof-loo-fixed-4}), we can further bound
(\ref{eq:proof-loo-fixed-3}) as follows:
\begin{align*}
\mathsf{Var}_{P_{s_{0},a,b}}\left(V\right) & \leq\mathsf{Var}_{\widehat{P}_{s_{0},a,b}}\left(V\right)+\big[\big(P_{s_{0},a,b}-\widehat{P}_{s_{0},a,b}\big)V\big]^{2}+\\
 & \quad+C_{\mathsf{bern}}\sqrt{\frac{\log\frac{1}{\delta}}{\left(1-\gamma\right)^{2}N\left(s_{0},a,b\right)}\mathsf{Var}_{P_{s_{0},a,b}}\big(V\big)}+\frac{C_{\mathsf{bern}}\log\frac{1}{\delta}}{\left(1-\gamma\right)^{2}N\left(s_{0},a,b\right)}\\
 & \leq\mathsf{Var}_{\widehat{P}_{s_{0},a,b}}\left(V\right)+\big[\big(P_{s_{0},a,b}-\widehat{P}_{s_{0},a,b}\big)V\big]^{2}+\frac{C_{\mathsf{bern}}\log\frac{1}{\delta}}{\left(1-\gamma\right)^{2}N\left(s_{0},a,b\right)}\\
 & \quad+\frac{1}{2}\mathsf{Var}_{P_{s_{0},a,b}}\left(V\right)+\frac{C_{\mathsf{bern}}^{2}}{2}\frac{\log\frac{1}{\delta}}{\left(1-\gamma\right)^{2}N\left(s_{0},a,b\right)},
\end{align*}
where the last relation follows from the AM-GM inequality. By rearranging
terms, we arrive at
\begin{equation}
\mathsf{Var}_{P_{s_{0},a,b}}\left(V\right)\leq2\mathsf{Var}_{\widehat{P}_{s_{0},a,b}}\left(V\right)+2\big[\big(P_{s_{0},a,b}-\widehat{P}_{s_{0},a,b}\big)V\big]^{2}+\frac{\left(C_{\mathsf{bern}}^{2}+2C_{\mathsf{bern}}\right)\log\frac{1}{\delta}}{\left(1-\gamma\right)^{2}N\left(s_{0},a,b\right)}.\label{eq:proof-loo-fixed-5}
\end{equation}
Taking (\ref{eq:proof-loo-fixed-5}) and (\ref{eq:proof-loo-fixed-11})
collectively yields
\begin{align}
\big|\big(\widehat{P}_{s_{0},a,b}-P_{s_{0},a,b}\big)V\big| & \leq\sqrt{\frac{2C_{\mathsf{bern}}^{2}}{N\left(s_{0},a,b\right)}\mathsf{Var}_{\widehat{P}_{s_{0},a,b}}\left(V\right)\log\frac{1}{\delta}}+\frac{C_{\mathsf{bern}}\log\frac{1}{\delta}}{\left(1-\gamma\right)N\left(s_{0},a,b\right)}\nonumber \\
 & \quad+\sqrt{\frac{2C_{\mathsf{bern}}^{2}}{N\left(s_{0},a,b\right)}\log\frac{1}{\delta}}\,\big|\big(\widehat{P}_{s_{0},a,b}-P_{s_{0},a,b}\big)V\big|+\frac{\sqrt{C_{\mathsf{bern}}^{2}\left(C_{\mathsf{bern}}^{2}+2C_{\mathsf{bern}}\right)\log\frac{1}{\delta}}}{\left(1-\gamma\right)N\left(s_{0},a,b\right)}.\label{eq:proof-loo-fixed-6}
\end{align}

When $N(s_{0},a,b)\leq\frac{1}{8C_{\mathsf{bern}}^{2}}\log\frac{1}{\delta}$,
we know that (\ref{eq:proof-loo-claim-1}) holds trivially since
\[
\left|\big(\widehat{P}_{s_{0},a,b}-P_{s_{0},a,b}\big)V\right|\leq\frac{2}{1-\gamma}=O\bigg(\frac{\log\frac{1}{\delta}}{\left(1-\gamma\right)N\left(s_{0},a,b\right)}\bigg).
\]
In comparison, if $N(s_{0},a,b)>\frac{1}{8C_{\mathsf{bern}}^{2}}\log\frac{1}{\delta}$,
then one can see from (\ref{eq:proof-loo-fixed-6}) that
\begin{align*}
\big|\big(\widehat{P}_{s_{0},a,b}-P_{s_{0},a,b}\big)V\big| & \leq\frac{1}{2}\big|\big(\widehat{P}_{s_{0},a,b}-P_{s_{0},a,b}\big)V\big|+\sqrt{\frac{2C_{\mathsf{bern}}^{2}}{N\left(s_{0},a,b\right)}\mathsf{Var}_{\widehat{P}_{s_{0},a,b}}\left(V\right)\log\frac{1}{\delta}}\\
 & \quad+\frac{C_{\mathsf{bern}}+\sqrt{C_{\mathsf{bern}}^{2}\left(C_{\mathsf{bern}}^{2}+2C_{\mathsf{bern}}\right)}}{\left(1-\gamma\right)N\left(s_{0},a,b\right)}\log\frac{1}{\delta}.
\end{align*}
Rearranging terms, we are left with
\begin{equation}
\big|\big(\widehat{P}_{s_{0},a,b}-P_{s_{0},a,b}\big)V\big|\leq\sqrt{\frac{8C_{\mathsf{bern}}^{2}}{N\left(s_{0},a,b\right)}\mathsf{Var}_{\widehat{P}_{s_{0},a,b}}\left(V\right)\log\frac{1}{\delta}}+2\frac{C_{\mathsf{bern}}+\sqrt{C_{\mathsf{bern}}^{2}\left(C_{\mathsf{bern}}^{2}+2C_{\mathsf{bern}}\right)}}{\left(1-\gamma\right)N\left(s_{0},a,b\right)}\log\frac{1}{\delta}.\label{eq:proof-loo-fixed-7}
\end{equation}
Putting the above bounds together concludes the proof of (\ref{eq:proof-loo-claim-1}).

\paragraph{Part 2: proof of inequality (\ref{eq:proof-loo-claim-2}).}

When $N(s_{0},a,b)\geq16C_{\mathsf{bern}}^{2}\log\frac{1}{\delta}$,
it holds that
\begin{align*}
\mathsf{Var}_{\widehat{P}_{s_{0},a,b}}\left(V\right) & \overset{\text{(i)}}{=}\mathsf{Var}_{P_{s_{0},a,b}}\left(V\right)-\big[\big(P_{s_{0},a,b}-\widehat{P}_{s_{0},a,b}\big)V\big]^{2}-\big(P_{s_{0},a,b}-\widehat{P}_{s_{0},a,b}\big)\big(\overline{V}\circ\overline{V}\big)\\
 & \overset{\text{(ii)}}{\leq}\mathsf{Var}_{P_{s_{0},a,b}}\left(V\right)+C_{\mathsf{bern}}\sqrt{\frac{1}{\left(1-\gamma\right)^{2}N\left(s_{0},a,b\right)}\mathsf{Var}_{P_{s_{0},a,b}}\big(V\big)\log\frac{1}{\delta}}+\frac{C_{\mathsf{bern}}\log\frac{1}{\delta}}{\left(1-\gamma\right)^{2}N\left(s_{0},a,b\right)}\\
 & \overset{\text{(iii)}}{\leq}2\mathsf{Var}_{P_{s_{0},a,b}}\left(V\right)+\frac{\left(C_{\mathsf{bern}}^{2}/4+C_{\mathsf{bern}}\right)\log\frac{1}{\delta}}{\left(1-\gamma\right)^{2}N\left(s_{0},a,b\right)}\\
 & =2\mathsf{Var}_{P_{s_{0},a,b}}\left(V\right)+O\left(\frac{\log\frac{1}{\delta}}{\left(1-\gamma\right)^{2}N\left(s_{0},a,b\right)}\right),
\end{align*}
which establishes (\ref{eq:proof-loo-claim-2}) in this case. Here,
(i) follows from (\ref{eq:proof-loo-fixed-3}), (ii) utilizes (\ref{eq:proof-loo-fixed-4}),
and (iii) applies the AM-GM inequality. Instead, if $N(s_{0},a,b)<16C_{\mathsf{bern}}^{2}\log\frac{1}{\delta}$,
then we know that (\ref{eq:proof-loo-claim-2}) is trivially true
because
\[
\mathsf{Var}_{P_{s_{0},a,b}}\left(V\right)\leq\frac{1}{\left(1-\gamma\right)^{2}}=O\left(\frac{\log\frac{1}{\delta}}{\left(1-\gamma\right)^{2}N\left(s_{0},a,b\right)}\right).
\]
Combining these two cases concludes the proof of (\ref{eq:proof-loo-claim-2}). 

\subsection{Proof of Lemma \ref{lemma:Q-monononicity} \label{sec:proof-lemma-Q-monotonicity}}

To begin with, we would like to justify that for any $(s,a,b)\in\mathcal{S}\times\mathcal{A}\times\mathcal{B}$,
it holds that
\[
Q^{\widehat{\mu},\star}\left(s,a,b\right)\geq Q_{\mathsf{pe}}^{-}\left(s,a,b\right).
\]
Recall from Lemma~\ref{lemma:gamma-contraction} that $Q_{\mathsf{pe}}^{-\star}$
is the fixed point of $\widehat{\mathcal{T}}_{\mathsf{pe}}^{-}$ and
hence
\begin{equation}
Q_{\mathsf{pe}}^{-\star}\left(s,a,b\right)=\big(\widehat{\mathcal{T}}_{\mathsf{pe}}^{-}Q_{\mathsf{pe}}^{-\star}\big)\left(s,a,b\right)=\max\left\{ \widehat{r}\left(s,a,b\right)+\gamma\widehat{P}_{s,a,b}V_{\mathsf{pe}}^{-\star}-\beta\left(s,a,b;V_{\mathsf{pe}}^{-\star}\right),0\right\} .\label{eq:proof-Q-monotonic-1}
\end{equation}
If $Q_{\mathsf{pe}}^{-\star}(s,a,b)=0$, then we can utilize the facts
that $Q_{\mathsf{pe}}^{-}\leq Q_{\mathsf{pe}}^{-\star}$ (as shown
in (\ref{eq:Q-pe-t-minus-star-all-t})) and $Q^{\widehat{\mu},\star}(s,a,b)\geq0$
to ensure that
\[
Q_{\mathsf{pe}}^{-}\left(s,a,b\right)\leq Q_{\mathsf{pe}}^{-\star}\left(s,a,b\right)=0\leq Q^{\widehat{\mu},\star}\left(s,a,b\right).
\]
Therefore, it suffices to look at the case when $Q_{\mathsf{pe}}^{-\star}(s,a,b)>0$
in the sequel.

Note that the condition $Q_{\mathsf{pe}}^{-\star}(s,a,b)>0$ taken
collectively with (\ref{eq:proof-Q-monotonic-1}) gives
\begin{equation}
Q_{\mathsf{pe}}^{-\star}\left(s,a,b\right)=\widehat{r}\left(s,a,b\right)+\gamma\widehat{P}_{s,a,b}V_{\mathsf{pe}}^{-\star}-\beta\left(s,a,b;V_{\mathsf{pe}}^{-\star}\right).\label{eq:proof-Q-monotonic-2}
\end{equation}
In addition, in this case we must have $N(s,a,b)>0$: otherwise we
have
\[
\beta\left(s,a,b;V_{\mathsf{pe}}^{-\star}\right)=\frac{1}{1-\gamma}+\frac{3}{N}>\frac{1}{1-\gamma},
\]
which taken collectively with (\ref{eq:proof-Q-monotonic-2}) leads
to the following contradiction
\begin{align*}
Q_{\mathsf{pe}}^{-\star}\left(s,a,b\right) & =\widehat{r}\left(s,a,b\right)+\gamma\widehat{P}_{s,a,b}V_{\mathsf{pe}}^{-\star}-\beta\left(s,a,b;V_{\mathsf{pe}}^{-\star}\right)<\frac{\gamma}{1-\gamma}-\frac{1}{1-\gamma}<0.
\end{align*}
With the condition $N(s,a,b)>0$ in mind, we can proceed to bound
\begin{align}
Q_{\mathsf{pe}}^{-}\left(s,a,b\right) & \overset{\text{(i)}}{\leq}Q_{\mathsf{pe}}^{-\star}\left(s,a,b\right)\overset{\text{(ii)}}{=}\widehat{r}\left(s,a,b\right)+\gamma\widehat{P}_{s,a,b}V_{\mathsf{pe}}^{-\star}-\beta\left(s,a,b;V_{\mathsf{pe}}^{-\star}\right)\nonumber \\
 & \leq\widehat{r}\left(s,a,b\right)+\gamma\widehat{P}_{s,a,b}V_{\mathsf{pe}}^{-}-\beta\left(s,a,b;V_{\mathsf{pe}}^{-\star}\right)+\gamma\left\Vert V_{\mathsf{pe}}^{-}-V_{\mathsf{pe}}^{-\star}\right\Vert _{\infty}\nonumber \\
 & \overset{\text{(iii)}}{\leq}\widehat{r}\left(s,a,b\right)+\gamma P_{s,a,b}V_{\mathsf{pe}}^{-}-\beta\left(s,a,b;V_{\mathsf{pe}}^{-\star}\right)+\gamma\big(\widehat{P}_{s,a,b}-P_{s,a,b}\big)V_{\mathsf{pe}}^{-}+\frac{1}{N}\nonumber \\
 & \overset{\text{(iv)}}{\leq}\widehat{r}\left(s,a,b\right)+\gamma P_{s,a,b}V_{\mathsf{pe}}^{-}-\beta\left(s,a,b;V_{\mathsf{pe}}^{-}\right)+\gamma\big(\widehat{P}_{s,a,b}-P_{s,a,b}\big)V_{\mathsf{pe}}^{-}+\frac{3}{N}\nonumber \\
 & \overset{\text{(v)}}{\leq}r\left(s,a,b\right)+\gamma P_{s,a,b}V_{\mathsf{pe}}^{-}.\label{eq:proof-Q-monotonic-3}
\end{align}
Here (i) follows from (\ref{eq:Q-pe-t-minus-star-all-t}); (ii) holds
due to (\ref{eq:proof-Q-monotonic-2}); (iii) invokes (\ref{eq:Q_pe_approx_err});
(iv) makes use of (\ref{eq:b-V-bound}) and (\ref{eq:Q_pe_approx_err});
and (v) comes from (\ref{eq:b-dominance-1}) and the fact that $\widehat{r}(s,a,b)=r(s,a,b)$
when $N(s,a,b)\geq1$. As a result, we obtain
\begin{align}
Q^{\widehat{\mu},\star}\left(s,a,b\right)-Q_{\mathsf{pe}}^{-}\left(s,a,b\right) & =r\left(s,a,b\right)+\gamma P_{s,a,b}V^{\widehat{\mu},\star}-Q_{\mathsf{pe}}^{-}\left(s,a,b\right)\nonumber \\
 & \geq\gamma P_{s,a,b}\big(V^{\widehat{\mu},\star}-V_{\mathsf{pe}}^{-}\big),\label{eq:proof-Q-monotonic-4}
\end{align}
where the first identity follows from the Bellman equation, and the
last relation holds due to (\ref{eq:proof-Q-monotonic-3}). Let us
take
\begin{equation}
\left(s_{0},a_{0},b_{0}\right)=\underset{(s,a,b)\in\mathcal{S\times\mathcal{A}\times\mathcal{B}}}{\arg\min}\left\{ Q^{\widehat{\mu},\star}\left(s,a,b\right)-Q_{\mathsf{pe}}^{-}\left(s,a,b\right)\right\} ,\label{eq:proof-Q-monotonic-5}
\end{equation}
and use the notation $\nu_{\mathsf{br}}$ to denote the best-response
policy of the min-player when the max-player adopts policy $\widehat{\mu}$,
i.e., $\nu_{\mathsf{br}}\coloneqq\arg\min_{\nu}V^{\widehat{\mu},\nu}$
(whose existence is guaranteed by the existence of optimal policy
in MDP). Consequently, we arrive at
\begin{align}
Q^{\widehat{\mu},\star}\left(s_{0},a_{0},b_{0}\right)-Q_{\mathsf{pe}}^{-}\left(s_{0},a_{0},b_{0}\right) & \overset{\text{(i)}}{\geq}\gamma P_{s_{0},a_{0},b_{0}}\big(V^{\widehat{\mu},\star}-V_{\mathsf{pe}}^{-}\big)\geq\gamma\min_{s\in\mathcal{S}}\left[V^{\widehat{\mu},\star}\left(s\right)-V_{\mathsf{pe}}^{-}\left(s\right)\right]\nonumber \\
 & =\gamma\min_{s\in\mathcal{S}}\left\{ \mathop{\mathbb{E}}\limits _{a\sim\widehat{\mu}(s),b\sim\nu_{\mathsf{br}}(s)}\left[Q^{\widehat{\mu},\star}\left(s,a,b\right)\right]-\mathop{\mathbb{E}}\limits _{a\sim\mu_{T}^{-}(s),b\sim\nu_{T}^{-}(s)}\left[Q_{\mathsf{pe}}^{-}\left(s,a,b\right)\right]\right\} \nonumber \\
 & \overset{\text{(ii)}}{\geq}\gamma\min_{s\in\mathcal{S}}\mathop{\mathbb{E}}\limits _{a\sim\widehat{\mu}(s),b\sim\nu_{\mathsf{br}}(s)}\left[Q^{\widehat{\mu},\star}\left(s,a,b\right)-Q_{\mathsf{pe}}^{-}\left(s,a,b\right)\right]\nonumber \\
 & \geq\gamma\min_{(s,a,b)\in\mathcal{S\times\mathcal{A}\times\mathcal{B}}}\left[Q^{\widehat{\mu},\star}\left(s,a,b\right)-Q_{\mathsf{pe}}^{-}\left(s,a,b\right)\right]\nonumber \\
 & \overset{\text{(iii)}}{=}\gamma\left[Q^{\widehat{\mu},\star}\left(s_{0},a_{0},b_{0}\right)-Q_{\mathsf{pe}}^{-}\left(s_{0},a_{0},b_{0}\right)\right].\label{eq:proof-Q-monotonic-7}
\end{align}
Here (i) relies on (\ref{eq:proof-Q-monotonic-4}); (ii) holds true
since $(\mu_{T}^{-}(s),\nu_{T}^{-}(s))$ is the Nash equilibrium of
$Q_{\mathsf{pe}}^{-}(s,\cdot,\cdot)$ and recall that $\widehat{\mu}=\mu_{T}^{-}$
and hence
\[
\mathop{\mathbb{E}}\limits _{a\sim\mu_{T}^{-}(s),b\sim\nu_{T}^{-}(s)}\left[Q_{\mathsf{pe}}^{-}\left(s,a,b\right)\right]\leq\mathop{\mathbb{E}}\limits _{a\sim\widehat{\mu}(s),b\sim\nu_{\mathsf{br}}(s)}\left[Q^{\widehat{\mu},\star}\left(s,a,b\right)-Q_{\mathsf{pe}}^{-}\left(s,a,b\right)\right];
\]
and (iii) arises from the definition of $(s_{0},a_{0},b_{0})$ (cf.~(\ref{eq:proof-Q-monotonic-5})).
Given that $\gamma<1$, Condition (\ref{eq:proof-Q-monotonic-7})
can only hold if
\[
Q^{\widehat{\mu},\star}\left(s_{0},a_{0},b_{0}\right)-Q_{\mathsf{pe}}^{-}\left(s_{0},a_{0},b_{0}\right)\geq0,
\]
Therefore, we can conclude that $Q^{\widehat{\mu},\star}\geq Q_{\mathsf{pe}}^{-}$. 

Similarly, we can also show that $Q_{\mathsf{pe}}^{+}\geq Q^{\star,\widehat{\nu}}$
via the same argument, which we omit here for the sake of conciseness.
The claimed properties regarding the V-function are therefore immediately
consequences from $Q^{\widehat{\mu},\star}\geq Q_{\mathsf{pe}}^{-}$
and $Q_{\mathsf{pe}}^{+}\geq Q^{\star,\widehat{\nu}}$. 

\subsection{Proof of Lemma \ref{lemma:V_pe_lower_bound}\label{subsec:proof-lemma_V_pe_lower_bound}}

It is first observed from \eqref{eq:Vpe-minus-lower-bound-min} and
\eqref{eq:defn-nu0-min} that
\begin{equation}
V_{\mathsf{pe}}^{-}\left(s\right)\geq\mathop{\mathbb{E}}\limits _{a\sim\mu^{\star}(s),b\sim\nu_{0}(s)}\left[Q_{\mathsf{pe}}^{-}\left(s,a,b\right)\right]\ge\mathop{\mathbb{E}}\limits _{a\sim\mu^{\star}(s),b\sim\nu_{0}(s)}\left[Q_{\mathsf{pe}}^{-\star}\left(s,a,b\right)\right]-\frac{1}{N},\label{eq:proof_V_pe_lower_bound_1}
\end{equation}
where the last inequality comes from (\ref{eq:Q_pe_approx_err}). 
\begin{itemize}
\item For any $(s,a,b)\in\mathcal{S}\times\mathcal{A}\times\mathcal{B}$
obeying $N(s,a,b)\geq1$, it is seen that
\begin{align*}
Q_{\mathsf{pe}}^{-\star}\left(s,a,b\right) & \overset{\text{(i)}}{=}\max\left\{ \widehat{r}\left(s,a,b\right)+\gamma\widehat{P}_{s,a,b}V_{\mathsf{pe}}^{-\star}-\beta\left(s,a,b;V_{\mathsf{pe}}^{-\star}\right),0\right\} \\
 & \geq\widehat{r}\left(s,a,b\right)+\gamma\widehat{P}_{s,a,b}V_{\mathsf{pe}}^{-\star}-\beta\left(s,a,b;V_{\mathsf{pe}}^{-\star}\right)\\
 & \overset{\text{(ii)}}{\geq}\widehat{r}\left(s,a,b\right)+\gamma\widehat{P}_{s,a,b}V_{\mathsf{pe}}^{-}-\beta\left(s,a,b;V_{\mathsf{pe}}^{-}\right)-\frac{3}{N}\\
 & \overset{\text{(iii)}}{\geq}r\left(s,a,b\right)+\gamma P_{s,a,b}V_{\mathsf{pe}}^{-}-2\beta\left(s,a,b;V_{\mathsf{pe}}^{-}\right)+\frac{1}{N}
\end{align*}
holds with probability exceeding $1-\delta$. Here, (i) holds true
since $\widehat{Q}_{\mathsf{pe}}^{\star}$ is the fixed point of $\widehat{\mathcal{T}}_{\mathsf{pe}}^{-}$;
(ii) follows from (\ref{eq:Q_pe_approx_err}) and (\ref{eq:b-V-bound});
and (iii) follows from (\ref{eq:b-dominance-1}) and the fact that
$\widehat{r}(s,a,b)=r(s,a,b)$ when $N(s,a,b)\geq1$. 
\item If instead $N(s,a,b)=0$, then by definition we have
\[
\beta\left(s,a,b;V_{\mathsf{pe}}^{-\star}\right)=\frac{1}{1-\gamma}+\frac{4}{N},
\]
and therefore,
\[
Q_{\mathsf{pe}}^{-\star}\left(s,a,b\right)\geq0\geq r\left(s,a,b\right)+\gamma P_{s,a,b}V_{\mathsf{pe}}^{-}-2\beta\left(s,a,b;V_{\mathsf{pe}}^{-}\right)+\frac{1}{N}
\]
holds as well. 
\end{itemize}
To summarize, with probability exceeding $1-\delta$, one has
\[
Q_{\mathsf{pe}}^{-\star}\left(s,a,b\right)\geq r\left(s,a,b\right)+\gamma P_{s,a,b}V_{\mathsf{pe}}^{-}-2\beta\left(s,a,b;V_{\mathsf{pe}}^{-}\right)+\frac{1}{N}
\]
holds simultaneously for all $(s,a,b)\in\mathcal{S}\times\mathcal{A}\times\mathcal{B}$.
This taken collectively with (\ref{eq:proof_V_pe_lower_bound_1})
yields
\[
V_{\mathsf{pe}}^{-}\left(s\right)\geq\mathop{\mathbb{E}}\limits _{a\sim\mu^{\star}(s),b\sim\nu_{0}(s)}\left[r\left(s,a,b\right)+\gamma P_{s,a,b}V_{\mathsf{pe}}^{-}-2\beta\left(s,a,b;V_{\mathsf{pe}}^{-}\right)\right]
\]
holds for all $s\in\mathcal{S}$. This concludes the proof of \eqref{eq:main-proof-2}. 

\subsection{Proof of Lemma \ref{lemma:d_b_upper_bound}\label{sec:proof_lemma_d_b_upper_bound}}

Before proceeding, we make note of a key result will be useful in
the proof.

\begin{claim}\label{claim:chernoff-type}With probability exceeding
$1-\delta$, we have
\begin{equation}
\max\left\{ N\left(s,a,b\right),\log\frac{N}{\delta}\right\} \geq\frac{1}{2}Nd_{\mathsf{b}}\left(s,a,b\right)\label{eq:chernoff-type-bound}
\end{equation}
holds for all $(s,a,b)\in\mathcal{S}\times\mathcal{A}\times\mathcal{B}$.
\end{claim}\begin{proof}See Appendix \ref{subsec:proof-claim-chernoff}.\end{proof}

With the above claim in place, we can readily make the observation
that
\begin{align*}
\beta\left(s,a,b;V_{\mathsf{pe}}^{-}\right) & =\min\left\{ \max\left\{ \sqrt{\frac{C_{\mathsf{b}}\log\frac{N}{\delta}}{N\left(s,a,b\right)}\mathsf{Var}_{\widehat{P}_{s,a,b}}\left(V_{\mathsf{pe}}^{-}\right)},\frac{2C_{\mathsf{b}}\log\frac{N}{\delta}}{\left(1-\gamma\right)N\left(s,a,b\right)}\right\} ,\frac{1}{1-\gamma}\right\} +\frac{4}{N}\\
 & \overset{\text{(i)}}{\leq}\max\left\{ \sqrt{\frac{C_{\mathsf{b}}\log\frac{N}{\delta}}{\max\big\{ N\left(s,a,b\right),\log\frac{N}{\delta}\big\}}\mathsf{Var}_{\widehat{P}_{s,a,b}}\left(V_{\mathsf{pe}}^{-}\right)},\frac{2C_{\mathsf{b}}\log\frac{N}{\delta}}{\left(1-\gamma\right)\max\big\{ N\left(s,a,b\right),\log\frac{N}{\delta}\big\}}\right\} +\frac{4}{N}\\
 & \overset{\text{(ii)}}{\leq}\max\left\{ \sqrt{\frac{2C_{\mathsf{b}}\log\frac{N}{\delta}}{Nd_{\mathsf{b}}\left(s,a,b\right)}\mathsf{Var}_{\widehat{P}_{s,a,b}}\left(V_{\mathsf{pe}}^{-}\right)},\frac{4C_{\mathsf{b}}\log\frac{N}{\delta}}{\left(1-\gamma\right)Nd_{\mathsf{b}}\left(s,a,b\right)}\right\} +\frac{4}{N}\\
 & \leq c_{3}\sqrt{\frac{\log\frac{N}{\delta}}{Nd_{\mathsf{b}}\left(s,a,b\right)}\mathsf{Var}_{\widehat{P}_{s,a,b}}\left(V_{\mathsf{pe}}^{-}\right)}+\frac{c_{3}\log\frac{N}{\delta}}{\left(1-\gamma\right)Nd_{\mathsf{b}}\left(s,a,b\right)}+\frac{4}{N}
\end{align*}
for some universal constant $c_{3}\geq\max\big\{\sqrt{2C_{\mathsf{b}}},4C_{\mathsf{b}}\big\}$.
To see why (i) holds, it suffices to note that in the case where $N(s,a,b)\leq\log(N/\delta)$,
it holds that
\[
\frac{1}{1-\gamma}\leq\frac{2C_{\mathsf{b}}\log\frac{N}{\delta}}{\left(1-\gamma\right)\max\big\{ N\left(s,a,b\right),\log\frac{N}{\delta}\big\}},
\]
provided that $C_{\mathsf{b}}\geq1$; and with regards to (ii), we
take advantage of Claim \ref{claim:chernoff-type}. Consequently,
we can combine this with the definition of $\beta^{\mu^{\star},\nu_{0}}$
(cf.~\eqref{eq:defn-beta-mu-nu0}) to arrive at
\begin{align}
\big(d^{\mu^{\star},\nu_{0}}\big)^{\top}\beta^{\mu^{\star},\nu_{0}} & \leq c_{3}\underbrace{\sum_{s\in\mathcal{S}}d^{\mu^{\star},\nu_{0}}\left(s;\rho\right)\mathop{\mathbb{E}}\limits _{a\sim\mu^{\star}(s),b\sim\nu_{0}(s)}\left[\frac{\log\frac{N}{\delta}}{\left(1-\gamma\right)Nd_{\mathsf{b}}\left(s,a,b\right)}\right]}_{\eqqcolon\alpha_{1}}+\frac{4}{N}\nonumber \\
 & \quad+c_{3}\underbrace{\sum_{s\in\mathcal{S}}d^{\mu^{\star},\nu_{0}}\left(s;\rho\right)\mathop{\mathbb{E}}\limits _{a\sim\mu^{\star}(s),b\sim\nu_{0}(s)}\left[\sqrt{\frac{\log\frac{N}{\delta}}{Nd_{\mathsf{b}}\left(s,a,b\right)}\mathsf{Var}_{\widehat{P}_{s,a,b}}\left(V_{\mathsf{pe}}^{-}\right)}\right]}_{\eqqcolon\alpha_{2}},\label{eq:d-b-inner-product-alpha-12}
\end{align}
leaving us with two terms to deal with. 

\paragraph{Bounding the first term $\alpha_{1}$. }Let us begin
with the first term $\alpha_{1}$ in \eqref{eq:d-b-inner-product-alpha-12}.
Recalling that $\nu_{0}:\mathcal{S}\to\mathcal{B}$ is a deterministic
policy (see \eqref{eq:defn-nu0-s}), we can upper bound
\begin{align}
\alpha_{1} & =\sum_{s\in\mathcal{S},a\in\mathcal{A},b\in\mathcal{B}}d^{\mu^{\star},\nu_{0}}\big(s;\rho\big)\mu^{\star}(a\mymid s)\ind\big\{ b=\nu_{0}(s)\big\}\frac{\log\frac{N}{\delta}}{\left(1-\gamma\right)Nd_{\mathsf{b}}(s,a,b)}\nonumber \\
 & \overset{\text{(i)}}{=}\sum_{s\in\mathcal{S},a\in\mathcal{A}}d^{\mu^{\star},\nu_{0}}\big(s,a,\nu_{0}(s);\rho\big)\frac{\log\frac{N}{\delta}}{\left(1-\gamma\right)Nd_{\mathsf{b}}\big(s,a,\nu_{0}(s)\big)}\nonumber \\
 & =\frac{1}{\left(1-\gamma\right)N}\log\frac{N}{\delta}\sum_{s\in\mathcal{S},a\in\mathcal{A}}\frac{d^{\mu^{\star},\nu_{0}}\big(s,a,\nu_{0}(s);\rho\big)}{d_{\mathsf{b}}\big(s,a,\nu_{0}(s)\big)}\nonumber \\
 & \overset{\text{(ii)}}{\leq}\frac{C_{\mathsf{clipped}}^{\star}}{\left(1-\gamma\right)N}\log\frac{N}{\delta}\sum_{s\in\mathcal{S},a\in\mathcal{A}}\frac{d^{\mu^{\star},\nu_{0}}\big(s,a,\nu_{0}(s);\rho\big)}{\min\left\{ d^{\mu^{\star},\nu_{0}}\big(s,a,\nu_{0}(s);\rho\big),\frac{1}{S\left(A+B\right)}\right\} }\nonumber \\
 & \leq\frac{C_{\mathsf{clipped}}^{\star}SA}{\left(1-\gamma\right)N}\log\frac{N}{\delta}+\frac{C_{\mathsf{clipped}}^{\star}S\left(A+B\right)}{\left(1-\gamma\right)N}\log\frac{N}{\delta}\sum_{s\in\mathcal{S},a\in\mathcal{A}}d^{\mu^{\star},\nu_{0}}\big(s,a,\nu_{0}(s);\rho\big)\nonumber \\
 & \leq2\frac{C_{\mathsf{clipped}}^{\star}S\left(A+B\right)}{\left(1-\gamma\right)N}\log\frac{N}{\delta},\label{eq:alpha1-bound-appendix}
\end{align}
 where (i) comes from \eqref{eq:relation-dmunu-ab-s}, and (ii) follows
from Assumption \ref{assumption:uniliteral}. 

\paragraph{Bounding the second term $\alpha_{2}$. } Next, we move
on to bounding the second term $\alpha_{2}$ in \eqref{eq:d-b-inner-product-alpha-12},
which relies on the following result.

\begin{claim}\label{claim:var-perturbation}With probability exceeding
$1-\delta$, there exists some universal constant $c_{4}>0$ such
that
\[
\mathsf{Var}_{\widehat{P}_{s,a,b}}\left(V_{\mathsf{pe}}^{-}\right)\leq2\mathsf{Var}_{P_{s,a,b}}\left(V_{\mathsf{pe}}^{-}\right)+\frac{c_{4}\log\frac{N}{\delta}}{\left(1-\gamma\right)^{2}Nd_{\mathsf{b}}\left(s,a,b\right)}
\]
holds simultaneously for all $(s,a,b)\in\mathcal{S}\times\mathcal{A}\times\mathcal{B}$.
\end{claim}\begin{proof}See Appendix \ref{subsec:proof-claim-var}.
\end{proof}The bound in Claim \ref{claim:var-perturbation} allows
one to deduce that
\begin{align}
\alpha_{2} & =\sum_{s\in\mathcal{S},a\in\mathcal{A},b\in\mathcal{B}}d^{\mu^{\star},\nu_{0}}\big(s;\rho\big)\mu^{\star}(a\mymid s)\ind\big\{ b=\nu_{0}(s)\big\}\sqrt{\frac{\log\frac{N}{\delta}}{Nd_{\mathsf{b}}\left(s,a,b\right)}\mathsf{Var}_{\widehat{P}_{s,a,\nu_{0}(s)}}\left(V_{\mathsf{pe}}^{-}\right)}\nonumber \\
 & =\sum_{s\in\mathcal{S},a\in\mathcal{A}}d^{\mu^{\star},\nu_{0}}\big(s,a,\nu_{0}(s);\rho\big)\sqrt{\frac{\log\frac{N}{\delta}}{Nd_{\mathsf{b}}\big(s,a,\nu_{0}(s)\big)}\mathsf{Var}_{\widehat{P}_{s,a,\nu_{0}(s)}}\left(V_{\mathsf{pe}}^{-}\right)}\nonumber \\
 & \leq c_{5}\sum_{s\in\mathcal{S},a\in\mathcal{A}}d^{\mu^{\star},\nu_{0}}\big(s,a,\nu_{0}(s);\rho\big)\left[\sqrt{\frac{\log\frac{N}{\delta}}{Nd_{\mathsf{b}}\big(s,a,\nu_{0}(s)\big)}\mathsf{Var}_{P_{s,a,\nu_{0}(s)}}\left(V_{\mathsf{pe}}^{-}\right)}+\frac{\log\frac{N}{\delta}}{\left(1-\gamma\right)Nd_{\mathsf{b}}\big(s,a,\nu_{0}(s)\big)}\right]\nonumber \\
 & =c_{5}\underbrace{\sum_{s\in\mathcal{S},a\in\mathcal{A}}d^{\mu^{\star},\nu_{0}}\big(s,a,\nu_{0}(s);\rho\big)\sqrt{\frac{\log\frac{N}{\delta}}{Nd_{\mathsf{b}}\big(s,a,\nu_{0}(s)\big)}\mathsf{Var}_{P_{s,a,\nu_{0}(s)}}\left(V_{\mathsf{pe}}^{-}\right)}}_{\eqqcolon\alpha_{3}}+c_{5}\alpha_{1}\label{eq:alpha2-bound-appendix}
\end{align}
for some universal constant $c_{5}>0$. It thus boils down to bounding
$\alpha_{3}$, which we accomplish next.

\paragraph{Bounding the intermediate term $\alpha_{3}$. } It is
observed that
\begin{align}
\alpha_{3} & \overset{\text{(i)}}{\leq}\sum_{s\in\mathcal{S},a\in\mathcal{A}}d^{\mu^{\star},\nu_{0}}\big(s,a,\nu_{0}(s);\rho\big)\sqrt{\frac{C_{\mathsf{clipped}}^{\star}\log\frac{N}{\delta}}{N\min\left\{ d^{\mu^{\star},\nu_{0}}\big(s,a,\nu_{0}(s);\rho\big),\frac{1}{S\left(A+B\right)}\right\} }\mathsf{Var}_{P_{s,a,\nu_{0}(s)}}\left(V_{\mathsf{pe}}^{-}\right)}\nonumber \\
 & \leq\sqrt{\frac{C_{\mathsf{clipped}}^{\star}\log\frac{N}{\delta}}{N}}\sum_{s\in\mathcal{S},a\in\mathcal{A}}\sqrt{d^{\mu^{\star},\nu_{0}}\big(s,a,\nu_{0}(s);\rho\big)\mathsf{Var}_{P_{s,a,\nu_{0}(s)}}\left(V_{\mathsf{pe}}^{-}\right)}\nonumber \\
 & \quad+\sqrt{\frac{C_{\mathsf{clipped}}^{\star}S\left(A+B\right)\log\frac{N}{\delta}}{N}}\sum_{s\in\mathcal{S},a\in\mathcal{A}}\sqrt{d^{\mu^{\star},\nu_{0}}\big(s,a,\nu_{0}(s);\rho\big)}\cdot\sqrt{d^{\mu^{\star},\nu_{0}}\big(s,a,\nu_{0}(s);\rho\big)\mathsf{Var}_{P_{s,a,\nu_{0}(s)}}\left(V_{\mathsf{pe}}^{-}\right)}\nonumber \\
 & \overset{\text{(ii)}}{\leq}\sqrt{\frac{C_{\mathsf{clipped}}^{\star}}{N}\log\frac{N}{\delta}}\cdot\sqrt{SA}\cdot\sqrt{\sum_{s\in\mathcal{S},a\in\mathcal{A}}d^{\mu^{\star},\nu_{0}}\big(s,a,\nu_{0}(s);\rho\big)\mathsf{Var}_{P_{s,a,\nu_{0}(s)}}\left(V_{\mathsf{pe}}^{-}\right)}\nonumber \\
 & \quad+\sqrt{\frac{C_{\mathsf{clipped}}^{\star}S\left(A+B\right)\log\frac{N}{\delta}}{N}}\left[\sum_{s\in\mathcal{S},a\in\mathcal{A}}d^{\mu^{\star},\nu_{0}}\big(s,a,\nu_{0}(s);\rho\big)\right]\sqrt{\sum_{s\in\mathcal{S},a\in\mathcal{A}}d^{\mu^{\star},\nu_{0}}\big(s,a,\nu_{0}(s);\rho\big)\mathsf{Var}_{P_{s,a,\nu_{0}(s)}}\left(V_{\mathsf{pe}}^{-}\right)}\nonumber \\
 & \leq2\sqrt{\frac{C_{\mathsf{clipped}}^{\star}S\left(A+B\right)\log\frac{N}{\delta}}{N}}\sqrt{\sum_{s\in\mathcal{S},a\in\mathcal{A}}d^{\mu^{\star},\nu_{0}}\big(s,a,\nu_{0}(s);\rho\big)\mathsf{Var}_{P_{s,a,\nu_{0}(s)}}\left(V_{\mathsf{pe}}^{-}\right)}\nonumber \\
 & \overset{\text{(iii)}}{=}2\sqrt{\frac{C_{\mathsf{clipped}}^{\star}S\left(A+B\right)}{N}\log\frac{N}{\delta}}\sqrt{\sum_{s\in\mathcal{S}}d^{\mu^{\star},\nu_{0}}\left(s;\rho\right)\mathop{\mathbb{E}}\limits _{a\sim\mu^{\star}(s),b\sim\nu_{0}(s)}\left[\mathsf{Var}_{P_{s,a,b}}\left(V_{\mathsf{pe}}^{-}\right)\right]}\nonumber \\
 & \overset{\text{(iv)}}{\leq}2\sqrt{\frac{C_{\mathsf{clipped}}^{\star}S\left(A+B\right)}{N}\log\frac{N}{\delta}}\sqrt{\sum_{s\in\mathcal{S}}d^{\mu^{\star},\nu_{0}}\left(s;\rho\right)\mathsf{Var}_{P_{s}^{\mu^{\star},\nu_{0}}}\left(V_{\mathsf{pe}}^{-}\right)}.\label{eq:proof-d-b-1}
\end{align}
Here, (i) follows from Assumption \ref{assumption:uniliteral}; (ii)
holds due to the Cauchy-Schwarz inequality; (iii) is valid since $d^{\mu^{\star},\nu_{0}}\big(s,a,b;\rho\big)=d^{\mu^{\star},\nu_{0}}\big(s;\rho\big)\mu^{\star}(a\mymid s)\ind\{b=\nu_{0}(s)\}$;
and (iv) can be justified as follows
\begin{align*}
 & \mathop{\mathbb{E}}\limits _{a\sim\mu^{\star}(s),b\sim\nu_{0}(s)}\left[\mathsf{Var}_{P_{s,a,b}}\left(V_{\mathsf{pe}}^{-}\right)\right]=\sum_{a\in\mathcal{A}}\sum_{b\in\mathcal{B}}\mu^{\star}\left(a\mymid s\right)\nu_{0}\left(b\mymid s\right)\mathsf{Var}_{P_{s,a,b}}\left(V_{\mathsf{pe}}^{-}\right)\\
 & \qquad\qquad=\sum_{a\in\mathcal{A}}\sum_{b\in\mathcal{B}}\mu^{\star}\left(a\mymid s\right)\nu_{0}\left(b\mymid s\right)\left[P_{s,a,b}\left(V_{\mathsf{pe}}^{-}\circ V_{\mathsf{pe}}^{-}\right)-\left(P_{s,a,b}V_{\mathsf{pe}}^{-}\right)^{2}\right]\\
 & \qquad\qquad\leq\sum_{a\in\mathcal{A}}\sum_{b\in\mathcal{B}}\mu^{\star}\left(a\mymid s\right)\nu_{0}\left(b\mymid s\right)P_{s,a,b}\left(V_{\mathsf{pe}}^{-}\circ V_{\mathsf{pe}}^{-}\right)-\left(\sum_{a\in\mathcal{A}}\sum_{b\in\mathcal{B}}\mu^{\star}\left(a\mymid s\right)\nu_{0}\left(b\mymid s\right)P_{s,a,b}V_{\mathsf{pe}}^{-}\right)^{2}\\
 & \qquad\qquad=P^{\mu^{\star},\nu_{0}}(\cdot\mymid s)\left(V_{\mathsf{pe}}^{-}\circ V_{\mathsf{pe}}^{-}\right)-\big(P^{\mu^{\star},\nu_{0}}(\cdot\mymid s)V_{\mathsf{pe}}^{-}\big)^{2}=\mathsf{Var}_{P^{\mu^{\star},\nu_{0}}(\cdot\mymid s)}\left(V_{\mathsf{pe}}^{-}\right),
\end{align*}
where the penultimate line applies Jensen's inequality, and the last
line relies on the definition of $P^{\mu^{\star},\nu_{0}}$ (see \eqref{eq:defn-P-mu-nu0}).
Defining a vector $v=[v_{s}]_{s\in\mathcal{S}}\in\mathbb{R}^{S}$
such that $v_{s}\coloneqq\mathsf{Var}_{P^{\mu^{\star},\nu_{0}}(\cdot\mymid s)}(V_{\mathsf{pe}}^{-})$,
we obtain
\begin{align*}
v & =P^{\mu^{\star},\nu_{0}}\left(V_{\mathsf{pe}}^{-}\circ V_{\mathsf{pe}}^{-}\right)-\big(P^{\mu^{\star},\nu_{0}}V_{\mathsf{pe}}^{-}\big)\circ\big(P^{\mu^{\star},\nu_{0}}V_{\mathsf{pe}}^{-}\big)\\
 & =P^{\mu^{\star},\nu_{0}}\left(V_{\mathsf{pe}}^{-}\circ V_{\mathsf{pe}}^{-}\right)-\frac{1}{\gamma^{2}}V_{\mathsf{pe}}^{-}\circ V_{\mathsf{pe}}^{-}+\frac{1}{\gamma^{2}}V_{\mathsf{pe}}^{-}\circ V_{\mathsf{pe}}^{-}-\big(P^{\mu^{\star},\nu_{0}}V_{\mathsf{pe}}^{-}\big)\circ\big(P^{\mu^{\star},\nu_{0}}V_{\mathsf{pe}}^{-}\big)\\
 & \overset{\text{(i)}}{\leq}P^{\mu^{\star},\nu_{0}}\left(V_{\mathsf{pe}}^{-}\circ V_{\mathsf{pe}}^{-}\right)-\frac{1}{\gamma^{2}}V_{\mathsf{pe}}^{-}\circ V_{\mathsf{pe}}^{-}+\frac{2}{\gamma^{2}\left(1-\gamma\right)}\big(I-\gamma P^{\mu^{\star},\nu_{0}}\big)V_{\mathsf{pe}}^{-}+\frac{4}{\gamma^{2}\left(1-\gamma\right)}\beta^{\mu^{\star},\nu_{0}}\\
 & \overset{\text{(ii)}}{\leq}P^{\mu^{\star},\nu_{0}}\left(V_{\mathsf{pe}}^{-}\circ V_{\mathsf{pe}}^{-}\right)-\frac{1}{\gamma}V_{\mathsf{pe}}^{-}\circ V_{\mathsf{pe}}^{-}+\frac{2}{\gamma^{2}\left(1-\gamma\right)}\big(I-\gamma P^{\mu^{\star},\nu_{0}}\big)V_{\mathsf{pe}}^{-}+\frac{4}{\gamma^{2}\left(1-\gamma\right)}\beta^{\mu^{\star},\nu_{0}}\\
 & =\big(I-\gamma P^{\mu^{\star},\nu_{0}}\big)\left[\frac{2}{\gamma^{2}\left(1-\gamma\right)}V_{\mathsf{pe}}^{-}-\frac{1}{\gamma}\left(V_{\mathsf{pe}}^{-}\circ V_{\mathsf{pe}}^{-}\right)\right]+\frac{4}{\gamma^{2}\left(1-\gamma\right)}\beta^{\mu^{\star},\nu_{0}},
\end{align*}
where (ii) holds since $\gamma<1$. To see why (i) holds, we make
the following observation:
\begin{align*}
V_{\mathsf{pe}}^{-}\circ V_{\mathsf{pe}}^{-}-\gamma^{2}\big(P^{\mu^{\star},\nu_{0}}V_{\mathsf{pe}}^{-}\big)\circ\big(P^{\mu^{\star},\nu_{0}}V_{\mathsf{pe}}^{-}\big) & =\big(V_{\mathsf{pe}}^{-}-\gamma P^{\mu^{\star},\nu_{0}}V_{\mathsf{pe}}^{-}\big)\circ\big(V_{\mathsf{pe}}^{-}+\gamma P^{\mu^{\star},\nu_{0}}V_{\mathsf{pe}}^{-}\big)\\
 & \leq\big(V_{\mathsf{pe}}^{-}-\gamma P^{\mu^{\star},\nu_{0}}V_{\mathsf{pe}}^{-}+2\beta^{\mu^{\star},\nu_{0}}\big)\circ\big(V_{\mathsf{pe}}^{-}+\gamma P^{\mu^{\star},\nu_{0}}V_{\mathsf{pe}}^{-}\big)\\
 & \leq\frac{2}{1-\gamma}\big(I-\gamma P^{\mu^{\star},\nu_{0}}\big)V_{\mathsf{pe}}^{-}+\frac{4}{1-\gamma}\beta^{\mu^{\star},\nu_{0}},
\end{align*}
where the penultimate line arises from $\beta^{\mu^{\star},\nu_{0}}\geq0$
and $V_{\mathsf{pe}}^{-}+\gamma P^{\mu^{\star},\nu_{0}}V_{\mathsf{pe}}^{-}\geq0$,
and the last line holds since $V_{\mathsf{pe}}^{-}-\gamma P^{\mu^{\star},\nu_{0}}V_{\mathsf{pe}}^{-}+2\beta^{\mu^{\star},\nu_{0}}\geq0$
(due to (\ref{eq:main-proof-2})) and $\widehat{V}_{\mathsf{pe}}+\gamma P^{\mu^{\star},\nu_{0}}\widehat{V}_{\mathsf{pe}}\leq\frac{2}{1-\gamma}1$. 

Making use of the following identity (see \eqref{eq:defn-d-mu-nu-vector})
\[
\big(d^{\mu^{\star},\nu_{0}}\big)^{\top}=\left(1-\gamma\right)\rho^{\top}\big(I-\gamma P^{\mu^{\star},\nu_{0}}\big)^{-1},
\]
we can deduce that
\begin{align}
 & \sum_{s\in\mathcal{S}}d^{\mu^{\star},\nu_{0}}\left(s;\rho\right)\mathsf{Var}_{P^{\mu^{\star},\nu_{0}}(\cdot\mymid s)}\left(V_{\mathsf{pe}}^{-}\right)=\big(d^{\mu^{\star},\nu_{0}}\big)^{\top}v\nonumber \\
 & \quad\leq\left(1-\gamma\right)\rho^{\top}\left[\frac{2}{\gamma^{2}\left(1-\gamma\right)}V_{\mathsf{pe}}^{-}-\frac{1}{\gamma}\left(V_{\mathsf{pe}}^{-}\circ V_{\mathsf{pe}}^{-}\right)\right]+\frac{4}{\gamma^{2}\left(1-\gamma\right)}\big(d^{\mu^{\star},\nu_{0}}\big)^{\top}\beta^{\mu^{\star},\nu_{0}}\nonumber \\
 & \quad\leq\frac{2}{\gamma^{2}}\rho^{\top}V_{\mathsf{pe}}^{-}+\frac{4}{\gamma^{2}\left(1-\gamma\right)}\big(d^{\mu^{\star},\nu_{0}}\big)^{\top}\beta^{\mu^{\star},\nu_{0}}\nonumber \\
 & \quad\leq\frac{2}{\gamma^{2}\left(1-\gamma\right)}+\frac{4}{\gamma^{2}\left(1-\gamma\right)}\big(d^{\mu^{\star},\nu_{0}}\big)^{\top}\beta^{\mu^{\star},\nu_{0}},\label{eq:proof-d-b-2}
\end{align}
where the last line holds since $\rho^{\top}V_{\mathsf{pe}}^{-}\leq\|V_{\mathsf{pe}}^{-}\|_{\infty}\leq\frac{1}{1-\gamma}$.
Taking (\ref{eq:proof-d-b-1}) and (\ref{eq:proof-d-b-2}) collectively
implies that
\begin{align}
\alpha_{3} & \leq2\sqrt{\frac{C_{\mathsf{clipped}}^{\star}S\left(A+B\right)}{N}\log\frac{N}{\delta}}\sqrt{\frac{2}{\gamma^{2}\left(1-\gamma\right)}+\frac{4}{\gamma^{2}\left(1-\gamma\right)}\big(d^{\mu^{\star},\nu_{0}}\big)^{\top}\beta^{\mu^{\star},\nu_{0}}}\nonumber \\
 & \leq8\sqrt{\frac{C_{\mathsf{clipped}}^{\star}S\left(A+B\right)}{N\left(1-\gamma\right)}\log\frac{N}{\delta}}\left(1+\sqrt{\big(d^{\mu^{\star},\nu_{0}}\big)^{\top}\beta^{\mu^{\star},\nu_{0}}}\right),\label{eq:beta-bound-appendix}
\end{align}
where the last relation holds from the assumption $\gamma\geq1/2$
and the elementary inequality $\sqrt{x+y}\leq\sqrt{x}+\sqrt{y}$. 

\paragraph{Putting all this together. }

Taking together the above bounds \eqref{eq:alpha1-bound-appendix},
\eqref{eq:alpha2-bound-appendix} and \eqref{eq:beta-bound-appendix}
(on $\alpha_{1}$, $\alpha_{2}$ and $\alpha_{3}$, respectively)
leads to
\begin{align*}
\big(d^{\mu^{\star},\nu_{0}}\big)^{\top}\beta^{\mu^{\star},\nu_{0}} & \leq c_{3}\alpha_{1}+c_{3}\alpha_{2}+\frac{4}{N}\leq\left(c_{3}+c_{5}\right)\alpha_{1}+c_{5}\alpha_{3}+\frac{4}{N}\\
 & \leq\left(2c_{3}+2c_{5}+4\right)\frac{C_{\mathsf{clipped}}^{\star}S\left(A+B\right)}{\left(1-\gamma\right)N}\log\frac{N}{\delta}\\
 & \quad\quad+8c_{5}\sqrt{\frac{C_{\mathsf{clipped}}^{\star}S\left(A+B\right)}{N\left(1-\gamma\right)}\log\frac{N}{\delta}}\left(1+\sqrt{\big(d^{\mu^{\star},\nu_{0}}\big)^{\top}\beta^{\mu^{\star},\nu_{0}}}\right),
\end{align*}
where we have used the fact that $\frac{1}{N}\leq\frac{C_{\mathsf{clipped}}^{\star}S(A+B)}{\left(1-\gamma\right)N}\log\frac{N}{\delta}$.
In turn, this implies that there exists some sufficiently large constant
$\widetilde{C}>0$ such that
\begin{align*}
\big(d^{\mu^{\star},\nu_{0}}\big)^{\top}\beta^{\mu^{\star},\nu_{0}} & \leq\widetilde{C}\frac{C_{\mathsf{clipped}}^{\star}S\left(A+B\right)}{\left(1-\gamma\right)N}\log\frac{N}{\delta}+\widetilde{C}\sqrt{\frac{C_{\mathsf{clipped}}^{\star}S\left(A+B\right)}{N\left(1-\gamma\right)}\log\frac{N}{\delta}}\left(1+\sqrt{\big(d^{\mu^{\star},\nu_{0}}\big)^{\top}\beta^{\mu^{\star},\nu_{0}}}\right)\\
 & \leq\widetilde{C}\frac{C_{\mathsf{clipped}}^{\star}S\left(A+B\right)}{\left(1-\gamma\right)N}\log\frac{N}{\delta}+\widetilde{C}\sqrt{\frac{C_{\mathsf{clipped}}^{\star}S\left(A+B\right)}{N\left(1-\gamma\right)}\log\frac{N}{\delta}}\\
 & \quad\quad+\frac{\widetilde{C}^{2}}{2}\frac{C_{\mathsf{clipped}}^{\star}S\left(A+B\right)}{N\left(1-\gamma\right)}\log\frac{N}{\delta}+\frac{1}{2}\big(d^{\mu^{\star},\nu_{0}}\big)^{\top}\beta^{\mu^{\star},\nu_{0}},
\end{align*}
where the last relation follows from the AM-GM inequality. Rearranging
terms, we are left with
\[
\big(d^{\mu^{\star},\nu_{0}}\big)^{\top}\beta^{\mu^{\star},\nu_{0}}\leq\big(2\widetilde{C}+\widetilde{C}^{2}\big)\frac{C_{\mathsf{clipped}}^{\star}S\left(A+B\right)}{\left(1-\gamma\right)N}\log\frac{N}{\delta}+2\widetilde{C}\sqrt{\frac{C_{\mathsf{clipped}}^{\star}S\left(A+B\right)}{N\left(1-\gamma\right)}\log\frac{N}{\delta}},
\]
thus concluding the proof.

\subsubsection{Proof of Claim \ref{claim:chernoff-type} \label{subsec:proof-claim-chernoff}}

For any $(s,a,b)\in\mathcal{S}\times\mathcal{A}\times\mathcal{B}$,
if $Nd_{\mathsf{b}}(s,a,b)\leq2\log(N/\delta)$, then the claim (\ref{eq:chernoff-type-bound})
holds trivially. It thus suffices to show that with probability exceeding
$1-\delta$, (\ref{eq:chernoff-type-bound}) holds for all $(s,a,b)$
falling within the following set 
\[
\mathcal{N}_{\mathsf{large}}\coloneqq\left\{ \left(s,a,b\right)\in\mathcal{S}\times\mathcal{A}\times\mathcal{B}:Nd_{\mathsf{b}}\left(s,a,b\right)\geq2\log\frac{N}{\delta}\right\} .
\]
It is straightforward to check that the cardinality of $\mathcal{N}_{\mathsf{large}}$
obeys
\[
\left|\mathcal{N}_{\mathsf{large}}\right|\cdot\frac{2}{N}\log\frac{N}{\delta}=\sum_{\left(s,a,b\right)\in\mathcal{N}_{\mathsf{large}}}\frac{2}{N}\log\frac{N}{\delta}\leq\sum_{\left(s,a,b\right)\in\mathcal{N}_{\mathsf{large}}}d_{\mathsf{b}}\left(s,a,b\right)\leq1,
\]
and as a result,
\[
\left|\mathcal{N}_{\mathsf{large}}\right|\leq\frac{N}{2\log(N/\delta)}\leq\frac{N}{2}.
\]
In view of the Chernoff bound \citep[Exercise 2.3.2]{vershynin2018high},
for any $(s,a,b)\in\mathcal{N}_{\mathsf{large}}$, we have
\begin{align*}
\mathbb{P}\left(N\left(s,a,b\right)\geq\frac{1}{2}Nd_{\mathsf{b}}\left(s,a,b\right)\right) & \leq e^{-Nd_{\mathsf{b}}\left(s,a,b\right)}\left(\frac{e}{2}\right)^{\frac{1}{2}Nd_{\mathsf{b}}\left(s,a,b\right)}=\left(\frac{1}{2e}\right)^{\frac{1}{2}Nd_{\mathsf{b}}\left(s,a,b\right)}\\
 & \leq\left(\frac{1}{2e}\right)^{\log(N/\delta)}\leq\frac{\delta}{N}.
\end{align*}
Invoke the union bound to reach
\[
\mathbb{P}\left(N\left(s,a,b\right)\geq\frac{1}{2}Nd_{\mathsf{b}}\left(s,a,b\right),\forall\left(s,a,b\right)\in\mathcal{N}_{\mathsf{large}}\right)\leq\frac{\delta}{N}\left|\mathcal{N}_{\mathsf{large}}\right|\leq\delta.
\]
Putting the above two cases together completes the proof.

\subsubsection{Proof of Claim \ref{claim:var-perturbation}\label{subsec:proof-claim-var}}

In view of Lemma \ref{lemma:loo}, we know that with probability exceeding
$1-\delta$,
\begin{equation}
\mathsf{Var}_{\widehat{P}_{s,a,b}}\big(\widetilde{V}\big)\leq2\mathsf{Var}_{P_{s,a,b}}\big(\widetilde{V}\big)+O\left(\frac{\log\frac{N}{\delta}}{\left(1-\gamma\right)^{2}N\left(s,a,b\right)}\right)\label{eq:proof-claim-var-1}
\end{equation}
holds simultaneously for all $(s,a,b)\in\mathcal{S}\times\mathcal{A}\times\mathcal{B}$
satisfying $N(s,a,b)\geq1$. As a result, with probability exceeding
$1-\delta$,
\begin{equation}
\mathsf{Var}_{\widehat{P}_{s,a,b}}\big(\widetilde{V}\big)\leq2\mathsf{Var}_{P_{s,a,b}}\big(\widetilde{V}\big)+O\left(\frac{\log\frac{N}{\delta}}{\left(1-\gamma\right)^{2}\max\big\{ N\left(s,a,b\right),\log\frac{N}{\delta}\big\}}\right)\label{eq:proof-claim-var-2}
\end{equation}
holds simultaneously for all $(s,a,b)\in\mathcal{S}\times\mathcal{A}\times\mathcal{B}$.
To see why this is valid, let us divide into two cases: 
\begin{enumerate}
\item When $N(s,a,b)\geq\log\frac{N}{\delta}$, the relation (\ref{eq:proof-claim-var-2})
is clearly guaranteed by (\ref{eq:proof-claim-var-1}); 
\item When $N(s,a,b)\leq\log\frac{N}{\delta}$, then it is seen that
\begin{align*}
\mathsf{Var}_{\widehat{P}_{s,a,b}}\big(\widetilde{V}\big) & \leq\frac{1}{\left(1-\gamma\right)^{2}}=O\left(\frac{\log\frac{N}{\delta}}{\left(1-\gamma\right)^{2}\max\big\{ N\left(s,a,b\right),\log\frac{N}{\delta}\big\}}\right)\\
 & \leq2\mathsf{Var}_{P_{s,a,b}}\big(\widetilde{V}\big)+O\left(\frac{\log\frac{N}{\delta}}{\left(1-\gamma\right)^{2}\max\big\{ N\left(s,a,b\right),\log\frac{N}{\delta}\big\}}\right),
\end{align*}
which again yields (\ref{eq:proof-claim-var-2}). 
\end{enumerate}
Combine (\ref{eq:proof-claim-var-2}) with Claim \ref{claim:chernoff-type}
to establish the desired bound:

\[
\mathsf{Var}_{\widehat{P}_{s,a,b}}\left(V_{\mathsf{pe}}^{-}\right)\leq2\mathsf{Var}_{P_{s,a,b}}\left(V_{\mathsf{pe}}^{-}\right)+O\left(\frac{\log\frac{N}{\delta}}{\left(1-\gamma\right)^{2}Nd_{\mathsf{b}}\left(s,a,b\right)}\right).
\]

%% file: appendix_lower_bound.tex
\section{Proof of Theorem \ref{thm:lower-bound}\label{sec:Proof-of-Theorem-lower-bound}}

This section presents the proof of the minimax lower bound stated
in Theorem~\ref{thm:lower-bound}. Without loss of generality, we
assume throughout the proof that $A\geq B$. In Appendix \ref{subsec:hard-instances}, we will generate a family of hard MG instances indexed by a $A$-dimensional binary parameter $\theta$, which transfers the goal of NE estimation to parameter estimation. Then in Appendix \ref{subsec:sufficient-statistics} and \ref{subsec:minimax-final}, we construct minimax lower bounds by putting a prior distribution over this family of hard MG instances and computing
the posterior probability of failure to differentiate each entry of $\theta$.

\subsection{Constructing a family of hard Markov game instances} \label{subsec:hard-instances}

The first step lies in constructing a family of Markov games $\left\{ \mathcal{MG}_{\theta}\right\} $
each parameterized by a $\theta\in\{p,q\}^{A}$, where
\begin{equation}
p\coloneqq\gamma+14\frac{\left(1-\gamma\right)^{2}\varepsilon}{\gamma}\qquad\text{and}\qquad q\coloneqq\gamma-14\frac{\left(1-\gamma\right)^{2}\varepsilon}{\gamma}.\label{eq:LB-p-q}
\end{equation}
Clearly, under the conditions $\gamma\geq2/3$ and $0<\varepsilon\leq\frac{1}{42(1-\gamma)}$,
it holds that
\begin{equation}
\frac{1}{2}\leq\gamma-\frac{1-\gamma}{2}\leq q<p\leq\gamma+\frac{1-\gamma}{2}\leq1.\label{eq:p-q-lower-bound}
\end{equation}
For each $\theta=\{\theta_{a}\}_{a\in\mathcal{A}}\in\{p,q\}^{A}$,
we define the corresponding Markov game $\mathcal{MG}_{\theta}$,
and the associated initial state distribution $\rho\in\Delta(\mathcal{S})$
and the data distribution $d_{\mathsf{b}}\in\Delta(\mathcal{S}\times\mathcal{A}\times\mathcal{B})$
as follows: 
\begin{itemize}
\item Let the state space be $\mathcal{S}=\left\{ 0,1,2,\ldots,S-1\right\} $,
and the action spaces for the two players be $\mathcal{A}=\{0,1,\ldots,A-1\}$
and $\mathcal{B}=\{0,1,\ldots,B-1\}$, respectively. 
\item Define the transition kernel $P_{\theta}:\mathcal{S}\times\mathcal{A}\times\mathcal{B}\to\Delta(\mathcal{S})$
such that
\begin{align}
P_{\theta}\left(s'\mymid s,a,b\right) & =\begin{cases}
\theta_{a}\ind\left\{ s'=0\right\} +\left(1-\theta_{a}\right)\ind\left\{ s'=1\right\} , & \text{if }s=0\text{ and }b=0\\
\ind\left\{ s'=s\right\} , & \text{if }s\geq1\text{ or }b\geq1
\end{cases}\label{eq:LB-transition}
\end{align}
for any $s,s'\in\mathcal{S}$, $a\in\mathcal{A}$ and $b\in\mathcal{B}$. 
\item Set the reward function to be
\begin{equation}
r\left(s,a,b\right)=\ind\left\{ s=0\right\} \qquad\text{for all }(s,a,b)\in\mathcal{S}\times\mathcal{A}\times\mathcal{B}.\label{eq:defn-reward-function-LB}
\end{equation}
\item The initial state distribution $\rho\in\Delta(\mathcal{S})$ is taken
to be
\begin{equation}
\rho\left(s\right)=\ind\left\{ s=0\right\} ,\qquad\text{for all}\,s\in\mathcal{S}.\label{eq:defn-rho-LB}
\end{equation}
\item The distribution $d_{\mathsf{b}}\in\Delta(\mathcal{S}\times\mathcal{A}\times\mathcal{B})$
generating the batch dataset is defined such that
\begin{equation}
d_{\mathsf{b}}\left(s,a,b\right)=\begin{cases}
\frac{1}{C_{\mathsf{clipped}}^{\star}S\left(A+B\right)}, & \text{if }s=0\\
\frac{1}{AB}\left(1-\frac{AB}{C_{\mathsf{clipped}}^{\star}S\left(A+B\right)}\right), & \text{if }s=1\\
0, & \text{if }s\geq2
\end{cases}\label{eq:LB-db}
\end{equation}
for all $(s,a,b)\in\mathcal{S}\times\mathcal{A}\times\mathcal{B}$,
where $C_{\mathsf{clipped}}^{\star}$ is allowed to be any given quantity
satisfying $C_{\mathsf{clipped}}^{\star}\geq\frac{2AB}{S(A+B)}$.
In other words, for any given $s\in\mathcal{S}$, each action pair
$(a,b)\in\mathcal{A}\times\mathcal{B}$ is sampled with the same rate. 
\end{itemize}
Owing to the simple structure of the above $\mathcal{MG}_{\theta}$,
we are able to compute its value functions and the Nash equilibrium.
Before proceeding, we find it convenient to define two action subsets
\begin{equation}
\mathcal{A}_{p,\theta}\coloneqq\left\{ a\in\mathcal{A}:\theta_{a}=p\right\} \qquad\text{and}\qquad\mathcal{A}_{q,\theta}\coloneqq\left\{ a\in\mathcal{A}:\theta_{a}=q\right\} .\label{eq:defn-action-subsets}
\end{equation}

\begin{lemma}\label{lemma:MG-LB-property}Consider the Markov game
$\mathcal{MG}_{\theta}$. For any policy pair $(\mu,\nu)$, the value
function in $\mathcal{MG}_{\theta}$ obeys
\begin{align*}
V^{\mu,\nu}\left(0\right) & =\frac{1}{1-\gamma+\gamma\mu_{p,\theta}\nu_{0}\left(1-p\right)+\gamma\left(1-\mu_{p,\theta}\right)\nu_{0}\left(1-q\right)},\\
V^{\mu,\nu}\left(s\right) & =0\qquad\qquad\text{for all }s\geq1,
\end{align*}
where
\begin{equation}
\mu_{p,\theta}\coloneqq\sum_{a\in\mathcal{A}_{p,\theta}}\mu(a\mymid0)\qquad\text{and}\qquad\nu_{0}\coloneqq\nu(0\mymid0).\label{eq:defn-mup-nu0}
\end{equation}
In addition, if we define two policies $\mu_{\theta}^{\star}$ and
$\nu_{\theta}^{\star}$ as follows\begin{subequations}\label{eq:LB-NE}
\begin{align}
\mu_{\theta}^{\star}\left(a\mymid s\right) & =\frac{1}{\vert\mathcal{A}_{p,\theta}\vert}\ind\left\{ a\in\mathcal{A}_{p,\theta}\right\} ,\qquad\forall\,s\in\mathcal{S},a\in\mathcal{A},\\
\nu_{\theta}^{\star}\left(b\mymid s\right) & =\ind\left\{ b=0\right\} ,\qquad\qquad\quad\quad\forall\,s\in\mathcal{S},b\in\mathcal{B},
\end{align}
\end{subequations}then the policy pair $(\mu_{\theta}^{\star},\nu_{\theta}^{\star})$
is a Nash equilibrium of $\mathcal{MG}_{\theta}$. Furthermore, we
have 
\[
\big\{\mathcal{MG}_{\theta},\rho,d_{\mathsf{b}}\big\}\in\mathsf{MG}\big(C_{\mathsf{clipped}}^{\star}\big).
\]
\end{lemma}\begin{proof}See Appendix \ref{subsec:proof-lemma-MG-LB-property}.
\end{proof}

We shall also make note of two immediate consequences of Lemma \ref{lemma:MG-LB-property}.
First, Lemma \ref{lemma:MG-LB-property} allows us to compute the
corresponding value function $V^{\star}$ of the unique Nash equilibrium
in $\mathcal{M\mathcal{G}_{\theta}}$ as follows: \begin{subequations}\label{eq:V-theta-star-expression-0}
\begin{equation}
V^{\star}\left(0\right)=\frac{1}{1-\gamma p},\label{eq:V-theta-star-0}
\end{equation}
given that $\mu_{p,\theta}=1$ and $\nu_{0}=1$ hold under the NE
$(\mu_{\theta}^{\star},\nu_{\theta}^{\star})$. In addition, Lemma
\ref{lemma:MG-LB-property} tells us that for any policy $\mu$ of
the max-player, the optimal value function for the min-player obeys
\begin{equation}
V^{\mu,\star}(0)=\frac{1}{1-\gamma p+\gamma\left(1-\mu_{p,\theta}\right)\left(p-q\right)},\label{eq:V-theta-mu-star-0}
\end{equation}
\end{subequations}given that the best response of the min-player
is to set $\nu_{0}=1$. 

\subsection{Identifying sufficient statistics for the parameter $\theta$} \label{subsec:sufficient-statistics}

Denote by $f_{\theta}$ the probability mass function of a single
sample transition $(s_{i},a_{i},b_{i},s_{i}')$ such that
\[
\left(s_{i},a_{i},b_{i}\right)\sim d_{\mathsf{b}},\qquad s_{i}'\sim P_{\theta}\left(\cdot\mymid s_{i},a_{i},b_{i}\right).
\]
In view of the definition of $d_{\mathsf{b}}$ in (\ref{eq:LB-db}),
we can write
\[
f_{\theta}\left(s_{i},a_{i},b_{i},s_{i}'\right)=\begin{cases}
\frac{1}{C_{\mathsf{clipped}}^{\star}S\left(A+B\right)}\theta_{a_{i}} & \text{if }s_{i}=0,b_{i}=0,s_{i}'=0,\\
\frac{1}{C_{\mathsf{clipped}}^{\star}S\left(A+B\right)}\left(1-\theta_{a_{i}}\right) & \text{if }s_{i}=0,b_{i}=0,s_{i}'=1,\\
\frac{1}{C_{\mathsf{clipped}}^{\star}S\left(A+B\right)} & \text{if }s_{i}=0,b_{i}\geq1,s_{i}'=0,\\
\frac{1}{AB}\left(1-\frac{AB}{C_{\mathsf{clipped}}^{\star}S\left(A+B\right)}\right) & \text{if }s_{i}=s_{i}'=1,\\
0 & \text{otherwise}.
\end{cases}
\]
Therefore, the probability mass function of the offline dataset $\{(s_{i},a_{i},b_{i},s_{i}')\}_{i=1}^{N}$
containing $N$ independent data is given by
\begin{align}
\prod_{i=1}^{N}f_{\theta}\left(s_{i},a_{i},b_{i},s_{i}'\right) & =\left[\frac{1}{AB}\left(1-\frac{AB}{C_{\mathsf{clipped}}^{\star}S\left(A+B\right)}\right)\right]^{L}\left[\frac{1}{C_{\mathsf{clipped}}^{\star}S\left(A+B\right)}\right]^{J}\nonumber \\
 & \qquad\cdot\prod_{a\in\mathcal{A}}\left[\frac{1}{C_{\mathsf{clipped}}^{\star}S\left(A+B\right)}\theta_{a}\right]^{X_{a}}\left[\frac{1}{C_{\mathsf{clipped}}^{\star}S\left(A+B\right)}\left(1-\theta_{a}\right)\right]^{M_{a}-X_{a}}\nonumber \\
 & =\left[\frac{1}{AB}\left(1-\frac{AB}{C_{\mathsf{clipped}}^{\star}S\left(A+B\right)}\right)\right]^{L}\left[\frac{1}{C_{\mathsf{clipped}}^{\star}S\left(A+B\right)}\right]^{N+J-L}\prod_{a\in\mathcal{A}}\theta_{a}^{X_{a}}\left(1-\theta_{a}\right)^{M_{a}-X_{a}},\label{eq:proof-lb-1}
\end{align}
where we define
\begin{align*}
J & \coloneqq\sum_{i=1}^{N}\ind\{s_{i}=s_{i}'=0,b_{i}\geq1\},\\
L & \coloneqq\sum_{i=1}^{N}\ind\{s_{i}=s_{i}'=1\},\\
X_{a} & \coloneqq\sum_{i=1}^{N}\ind\left\{ s_{i}=0,a_{i}=a,b_{i}=0,s_{i}'=0\right\} ,\qquad\forall a\in\mathcal{A},\\
M_{a} & \coloneqq\sum_{i=1}^{N}\ind\left\{ s_{i}=0,a_{i}=a,b_{i}=0\right\} ,\qquad\forall a\in\mathcal{A}.
\end{align*}

By the factorization theorem (see, e.g., \citet[Theorem 2.2]{shao2003mathematical}),
the expression (\ref{eq:proof-lb-1}) implies that 
\begin{equation}
(X,M)\qquad\text{with }X=[X_{a}]_{a\in\mathcal{A}}\text{ and }M=[M_{a}]_{a\in\mathcal{A}}\label{eq:sufficient-statistic-X-M}
\end{equation}
forms a sufficient statistic for the underlying parameter $\theta$. 

\subsection{Establishing the minimax lower bound} \label{subsec:minimax-final}

In view of \eqref{eq:V-theta-star-expression-0} and our choice of
$\rho$ (cf.~\eqref{eq:defn-rho-LB}), we can obtain --- for any
policy $\mu$ of the max-player --- that
\begin{align}
V^{\star}\left(\rho\right)-V^{\mu,\star}\left(\rho\right) & =V^{\star}\left(0\right)-V^{\mu,\star}\left(0\right)=\frac{1}{1-\gamma p}-\frac{1}{1-\gamma p+\gamma\left(1-\mu_{p,\theta}\right)\left(p-q\right)}\nonumber \\
 & =\frac{\gamma\left(1-\mu_{p,\theta}\right)\left(p-q\right)}{\left(1-\gamma p\right)\left[1-\gamma p+\gamma\left(1-\mu_{p,\theta}\right)\left(p-q\right)\right]}\nonumber \\
 & =\left(1-\mu_{p,\theta}\right)\frac{28\left(1-\gamma\right)^{2}\varepsilon}{\left[1-\gamma\left(\gamma+14\frac{\left(1-\gamma\right)^{2}\varepsilon}{\gamma}\right)\right]\left[1-\gamma\left(\gamma+14\frac{\left(1-\gamma\right)^{2}\varepsilon}{\gamma}\right)+\left(1-\mu_{p,\theta}\right)28\left(1-\gamma\right)^{2}\varepsilon\right]}\nonumber \\
 & =\left(1-\mu_{p,\theta}\right)\frac{28\varepsilon}{\left[1+\gamma-14\left(1-\gamma\right)\varepsilon\right]\left[1+\gamma+\left(1-2\mu_{p,\theta}\right)14\left(1-\gamma\right)\varepsilon\right]}\nonumber \\
 & \geq\left(1-\mu_{p,\theta}\right)6\varepsilon=6\varepsilon\sum_{a\in\mathcal{A}_{q,\theta}}\mu\left(a\mymid0\right)=6\varepsilon\sum_{a\in\mathcal{A}}\mu\left(a\mymid0\right)\ind\left\{ \theta_{a}=q\right\} ,\label{eq:proof-lb-2}
\end{align}
where the penultimate relation holds provided that $\varepsilon\leq\frac{1}{42(1-\gamma)}$.
As a consequence, we have
\begin{align}
\underset{(\widehat{\mu},\widehat{\nu})}{\inf}\underset{\{\mathcal{MG},\rho,d_{\mathsf{b}}\}\in\mathsf{MG}(C_{\mathsf{clipped}}^{\star})}{\sup}\mathbb{E}\left[V^{\star,\widehat{\nu}}\left(\rho\right)-V^{\widehat{\mu},\star}\left(\rho\right)\right] & \geq\underset{\widehat{\mu}}{\inf}\underset{\{\mathcal{MG},\rho,d_{\mathsf{b}}\}\in\mathsf{MG}(C_{\mathsf{clipped}}^{\star})}{\sup}\mathbb{E}\left[V^{\star}\left(\rho\right)-V^{\widehat{\mu},\star}\left(\rho\right)\right]\nonumber \\
 & \geq\underset{\widehat{\mu}}{\inf}\underset{\mathcal{MG}_{\theta}:\theta\in\left\{ p,q\right\} ^{A}}{\sup}\mathbb{E}\left[V^{\star}\left(\rho\right)-V^{\widehat{\mu},\star}\left(\rho\right)\right]\nonumber \\
 & \geq6\varepsilon\,\underset{\widehat{\mu}}{\inf}\underset{\mathcal{MG}_{\theta}:\theta\in\left\{ p,q\right\} ^{A}}{\sup}\mathbb{E}\left[\sum_{a\in\mathcal{A}}\widehat{\mu}\left(a\mymid0\right)\ind\left\{ \theta_{a}=q\right\} \right],\label{eq:proof-lb-3}
\end{align}
where the first relation holds since $V^{\star,\widehat{\nu}}(s)\geq V^{\mu^{\star},\widehat{\nu}}(s)\geq V^{\mu^{\star},\nu^{\star}}(s)=V^{\star}(s)$,
and the last line follows from (\ref{eq:proof-lb-2}). Here, the expectation
is taken over the randomness of the data distribution induced by $d_{\mathsf{b}}$
and that of the transition kernel of the corresponding Markov game. 

To further lower bound (\ref{eq:proof-lb-3}), we define $\mathcal{U}$
to be the uniform distribution over $\{p,q\}^{A}$ (namely, a random
vector $\theta\sim\mathcal{U}$ obeys $\theta_{a}\overset{\text{i.i.d.}}{=}\begin{cases}
p & \text{w.p. }0.5\\
q & \text{w.p. }0.5
\end{cases}$ for all $a\in\mathcal{A}$). It is readily seen that
\begin{align}
\underset{\widehat{\mu}}{\inf}\underset{\mathcal{MG}_{\theta}:\theta\in\left\{ p,q\right\} ^{A}}{\sup}\mathbb{E}\left[\sum_{a\in\mathcal{A}}\widehat{\mu}\left(a\mymid0\right)\ind\left\{ \theta_{a}=q\right\} \right] & \overset{\text{(i)}}{\geq}\underset{\widehat{\mu}}{\inf}\,\mathbb{E}_{\theta\sim\mathcal{U}}\left[\mathbb{E}_{\mathcal{MG}_{\theta}}\left[\sum_{a\in\mathcal{A}}\widehat{\mu}\left(a\mymid0\right)\ind\left\{ \theta_{a}=q\right\} \right]\right]\nonumber \\
 & =\underset{\widehat{\mu}}{\inf}\,\mathop{\mathbb{E}}\limits _{\mathcal{MG}_{\theta},\theta\sim\mathcal{U}}\left[\mathbb{E}\left[\sum_{a\in\mathcal{A}}\widehat{\mu}\left(a\mymid0\right)\ind\left\{ \theta_{a}=q\right\} \mymid X,M\right]\right]\nonumber \\
 & \overset{\text{(ii)}}{=}\underset{\widehat{\mu}}{\inf}\,\mathop{\mathbb{E}}\limits _{\mathcal{MG}_{\theta},\theta\sim\mathcal{U}}\left[\sum_{a\in\mathcal{A}}\widehat{\mu}\left(a\mymid0\right)\mathbb{E}\Big[\ind\left\{ \theta_{a}=q\right\} \mymid X,M\Big]\right]\nonumber \\
 & \overset{\text{(iii)}}{=}\underset{\widehat{\mu}}{\inf}\,\mathop{\mathbb{E}}\limits _{\mathcal{MG}_{\theta},\theta\sim\mathcal{U}}\left[\sum_{a\in\mathcal{A}}\widehat{\mu}\left(a\mymid0\right)\mathbb{P}\left(\theta_{a}=q\mymid X_{a},M_{a}\right)\right],\label{eq:proof-lb-4}
\end{align}
where the expectations appear after (i) are taken w.r.t.~the randomness
of data distribution induced by $d_{\mathsf{b}}$ and the transition
kernel of $\mathcal{MG}_{\theta}$, with $\theta$ following a prior
distribution $\mathcal{U}$. To see why (ii) holds, note that it suffices
to consider $\widehat{\mu}$ that is a measurable function of the
sufficient statistic $(X,M$) (see \citet[Proposition 3.13]{johnstone2017function});
with regards (iii), we note that it holds since $\{X_{a'},M_{a'}\}_{a'\neq a}$
is independent of $\theta_{a}$ conditional on $(X_{a},M_{a})$. To
further lower bound (\ref{eq:proof-lb-4}), we look at the following
two cases separately. 

\paragraph{Case 1: when the sample size $N$ is not too small.}

Consider the case where 
\begin{equation}
N\geq\widetilde{C}\frac{S\left(A+B\right)C_{\mathsf{clipped}}^{\star}\log\left(A+B\right)}{1-\gamma}=\frac{1}{d_{\mathsf{b}}\left(0,a,0\right)}\cdot\frac{\widetilde{C}\log\left(A+B\right)}{1-\gamma}\label{eq:N-lower-bound-tilde-C}
\end{equation}
holds for some sufficiently large constant $\widetilde{C}>0$. In
order to be compatible with the assumption
\begin{equation}
N<\frac{c_{2}S\left(A+B\right)C_{\mathsf{clipped}}^{\star}}{\left(1-\gamma\right)^{3}\varepsilon^{2}\log\left(A+B\right)}=\frac{1}{d_{\mathsf{b}}\left(0,a,0\right)}\cdot\frac{c_{2}}{\left(1-\gamma\right)^{3}\varepsilon^{2}\log\left(A+B\right)},\label{eq:N-upper-bound-c2}
\end{equation}
it suffices to focus on the regime where
\begin{equation}
\varepsilon\leq\varepsilon_{0}\coloneqq\sqrt{\frac{c_{2}}{2\widetilde{C}}}\frac{1}{\left(1-\gamma\right)\log\left(A+B\right)}.\label{eq:eps-upper-bound-case1}
\end{equation}

In order to understand \eqref{eq:proof-lb-4}, we need to take a look
at $M_{a}$ and $X_{a}$. To begin with, it is straightforward to
check that $M_{a}\sim\mathsf{Binomial}\big(N,d_{\mathsf{b}}(0,a,0)\big)$.
When $N$ is sandwiched between (\ref{eq:N-lower-bound-tilde-C})
and (\ref{eq:N-upper-bound-c2}), we can establish the following high-probability
bound on $M_{a}$. 

\begin{lemma}\label{lemma:Ma-concentration}Suppose that $\widetilde{C}$
is sufficiently large, and $\varepsilon\leq\frac{1}{(1-\gamma)\log A}$.
Then we have
\begin{equation}
\mathbb{P}\left(\frac{c_{3}\log A}{1-\gamma}\leq M_{a}\leq\frac{c_{4}}{\left(1-\gamma\right)^{3}\varepsilon^{2}\log A}\right)\geq1-\frac{2}{A^{4}},\label{eq:proof-lb-7}
\end{equation}
where $c_{3}=\widetilde{C}/2$ and $c_{4}=c_{2}+\sqrt{8c_{2}/c_{\mathsf{ch}}}$,
with $c_{\mathsf{ch}}>0$ some universal constant independent of $c_{2}$.
\end{lemma}\begin{proof}See Appendix \ref{subsec:proof-lemma-Ma-concentration}.\end{proof}

In words, $M_{a}$ is guaranteed to reside within an interval $\big[\frac{c_{3}\log A}{1-\gamma},\frac{c_{4}}{\left(1-\gamma\right)^{3}\varepsilon^{2}\log A}\big]$.
In addition, in order to bound $\mathbb{P}\left(\theta_{a}=q\mymid X_{a},M_{a}\right)$
in \eqref{eq:proof-lb-4}, we seek to first investigate the properties
of $X_{a}\mid\{\theta_{a},M_{a}=M\}$ for some $M\in\big[\frac{c_{3}\log A}{1-\gamma},\frac{c_{4}}{\left(1-\gamma\right)^{3}\varepsilon^{2}\log A}\big]$.
Clearly, conditional on $\theta_{a}$ and $M_{a}=M$, the random variable
$X_{a}$ is distributed as $\mathsf{Binomial}(M,\theta_{a})$. We
make note of the following useful result regarding binomial random
variables.

\begin{lemma}\label{lemma:E-existence}Suppose that $\widetilde{C}$
is sufficiently large and $c_{2}$ is sufficiently small. Consider
any integer $M$ satisfying
\begin{equation}
\frac{c_{3}\log A}{1-\gamma}\leq M\leq\frac{c_{4}}{\left(1-\gamma\right)^{3}\varepsilon^{2}\log A},\label{eq:M-range-lemma}
\end{equation}
and generate $B_{p}\sim\mathsf{Binomial}(M,p)$ and $B_{q}\sim\mathsf{Binomial}(M,q)$.
Then there exists a set $E_{M}\subseteq[M]$ such that
\begin{equation}
\mathbb{P}\left(B_{p}\in E_{M}\right)\geq1-\frac{2}{A^{4}},\qquad\mathbb{P}\left(B_{q}\in E_{M}\right)\geq1-\frac{2}{A^{4}},\label{eq:E-high-prob}
\end{equation}
\begin{equation}
\text{and}\qquad\frac{\mathbb{P}\left(B_{q}=n\right)}{\mathbb{P}\left(B_{p}=n\right)}\geq\frac{1}{2},\qquad\forall n\in E_{M}.\label{eq:E-lower-bound}
\end{equation}
\end{lemma}\begin{proof}See Appendix \ref{subsec:proof-lemma-E-existence}.\end{proof}

In words, when $M$ falls within the range \eqref{eq:M-range-lemma},
there exists a high-probability set $E_{M}$ such that it is not that
easy to differentiate $\mathsf{Binomial}(M,p)$ and $\mathsf{Binomial}(M,q)$.
This result motivates us to pay particular attention the following
set $E\subseteq\mathbb{N}^{2}$:
\[
E\coloneqq\left\{ \left(x,m\right)\in\mathbb{N}^{2}:x\in E_{m},\frac{c_{3}\log A}{1-\gamma}\leq m\leq\frac{c_{4}}{\left(1-\gamma\right)^{3}\varepsilon^{2}\log A}\right\} .
\]
It is seen that 
\begin{align}
\mathbb{P}\Big(\left(X_{a},M_{a}\right)\in E\mymid\theta_{a}=p\Big) & =\mathbb{E}\Big[\mathbb{P}\big(\left(X_{a},M_{a}\right)\in E\mymid M_{a},\theta_{a}=p\big)\mymid\theta_{a}=p\Big]=\mathbb{E}\Big[\mathbb{P}\left(X_{a}\in E_{M_{a}}\mymid M_{a},\theta_{a}=p\right)\mymid\theta_{a}=p\Big]\nonumber \\
 & \geq\mathbb{E}\left[\mathbb{P}\left(X_{a}\in E_{M_{a}}\mymid M_{a},\theta_{a}=p\right)\ind\left\{ \frac{c_{3}\log A}{1-\gamma}\leq M_{a}\leq\frac{c_{4}}{\left(1-\gamma\right)^{3}\varepsilon^{2}\log A}\right\} \mymid\theta_{a}=p\right]\nonumber \\
 & \overset{\text{(i)}}{\geq}\mathbb{E}\left[\left(1-\frac{2}{A^{4}}\right)\ind\left\{ \frac{c_{3}\log A}{1-\gamma}\leq M_{a}\leq\frac{c_{4}}{\left(1-\gamma\right)^{3}\varepsilon^{2}\log A}\right\} \mymid\theta_{a}=p\right]\nonumber \\
 & \overset{\text{(ii)}}{=}\left(1-\frac{2}{A^{4}}\right)\mathbb{P}\left(\frac{c_{3}\log A}{1-\gamma}\leq M_{a}\leq\frac{c_{4}}{\left(1-\gamma\right)^{3}\varepsilon^{2}\log A}\right)\nonumber \\
 & \overset{\text{(iii)}}{\geq}\left(1-\frac{2}{A^{4}}\right)^{2}\geq1-\frac{4}{A^{4}}.\label{eq:proof-lb-8}
\end{align}
Here, (i) utilizes Lemma \ref{lemma:E-existence}, (ii) is valid since
$\theta_{a}$ is independent of $M_{a}$, whereas (iii) follows from
(\ref{eq:proof-lb-5}) and (\ref{eq:proof-lb-6}). Similarly, one
can show that
\begin{equation}
\mathbb{P}\Big(\left(X_{a},M_{a}\right)\in E\mymid\theta_{a}=q\Big)\geq1-\frac{4}{A^{4}}.\label{eq:proof-lb-9}
\end{equation}
Taking (\ref{eq:proof-lb-8}) and (\ref{eq:proof-lb-9}) collectively
yields
\begin{align}
\mathbb{P}\Big(\left(X_{a},M_{a}\right)\in E\Big) & =\mathbb{P}\Big(\left(X_{a},M_{a}\right)\in E\mymid\theta_{a}=p\Big)\mathbb{P}\left(\theta_{a}=p\right)+\mathbb{P}\Big(\left(X_{a},M_{a}\right)\in E\mymid\theta_{a}=q\Big)\mathbb{P}\left(\theta_{a}=q\right)\nonumber \\
 & \geq\frac{1}{2}\left(1-\frac{4}{A^{4}}\right)+\frac{1}{2}\left(1-\frac{4}{A^{4}}\right)=1-\frac{4}{A^{4}}.\label{eq:proof-lb-10}
\end{align}
A little calculation shows that: when $(x_{a},m_{a})\in E$, we have
\begin{align}
 & \mathbb{P}\left(\theta_{a}=q\mymid X_{a}=x_{a},M_{a}=m_{a}\right)=\frac{\mathbb{P}\left(\theta_{a}=q,X_{a}=x_{a}\mymid M_{a}=m_{a}\right)}{\mathbb{P}\left(X_{a}=x_{a}\mymid M_{a}=m_{a}\right)}\nonumber \\
 & =\frac{\mathbb{P}\left(X_{a}=x_{a}\mymid\theta_{a}=q,M_{a}=m_{a}\right)\mathbb{P}\left(\theta_{a}=q\mymid M_{a}=m_{a}\right)}{\mathbb{P}\left(X_{a}=x_{a}\mymid\theta_{a}=p,M_{a}=m_{a}\right)\mathbb{P}\left(\theta_{a}=p\mymid M_{a}=m_{a}\right)+\mathbb{P}\left(X_{a}=x_{a}\mymid\theta_{a}=q,M_{a}=m_{a}\right)\mathbb{P}\left(\theta_{a}=q\mymid M_{a}=m_{a}\right)}\nonumber \\
 & \overset{\text{(i)}}{=}\frac{\mathbb{P}\left(X_{a}=x_{a}\mymid\theta_{a}=q,M_{a}=m_{a}\right)\mathbb{P}\left(\theta_{a}=q\right)}{\mathbb{P}\left(X_{a}=x_{a}\mymid\theta_{a}=p,M_{a}=m_{a}\right)\mathbb{P}\left(\theta_{a}=p\right)+\mathbb{P}\left(X_{a}=x_{a}\mymid\theta_{a}=q,M_{a}=m_{a}\right)\mathbb{P}\left(\theta_{a}=q\right)}\nonumber \\
 & \overset{\text{(ii)}}{=}\frac{\mathbb{P}\left(X_{a}=x_{a}\mymid\theta_{a}=q,M_{a}=m_{a}\right)}{\mathbb{P}\left(X_{a}=x_{a}\mymid\theta_{a}=p,M_{a}=m_{a}\right)+\mathbb{P}\left(X_{a}=x_{a}\mymid\theta_{a}=q,M_{a}=m_{a}\right)}\nonumber \\
 & \overset{\text{(iii)}}{=}\frac{\mathbb{P}\left(B_{q}=x_{a}\right)}{\mathbb{P}\left(B_{p}=x_{a}\right)+\mathbb{P}\left(B_{q}=x_{a}\right)}\overset{\text{(iv)}}{\geq}\frac{1/2}{1+1/2}=\frac{1}{3},\label{eq:proof-lb-11}
\end{align}
where we let $B_{p}$ and $B_{q}$ be two random variables distributed
as $B_{p}\sim\mathsf{Binomial}(m_{a},p)$ and $B_{q}\sim\mathsf{Binomial}(m_{a},q)$.
Here, (i) relies on the fact that $\theta_{a}$ is independent from
$M_{a}$; (ii) comes from the fact that $\mathbb{P}\left(\theta_{a}=p\right)=\mathbb{P}\left(\theta_{a}=q\right)=1/2$;
(iii) follows since conditional on $M_{a}=m_{a}$ and $\theta_{a}=p$
(resp.~$\theta_{a}=q$), the distribution of $X_{a}$ is $\mathsf{Binomial}(m_{a},p)$
(resp.~$\mathsf{Binomial}(m_{a},q)$); and (iv) follows from Lemma
\ref{lemma:E-existence} given that $(x_{a},m_{a})\in E$. We can
thus conclude that 
\begin{align*}
 & \underset{(\widehat{\mu},\widehat{\nu})}{\inf}\underset{\{\mathcal{MG},\rho,d_{\mathsf{b}}\}\in\mathsf{MG}(C^{\star})}{\sup}\mathbb{E}\left[V^{\star,\widehat{\nu}}\left(\rho\right)-V^{\widehat{\mu},\star}\left(\rho\right)\right]\\
 & \quad\overset{\text{(i)}}{\geq}6\varepsilon\,\underset{\widehat{\mu}}{\inf}\,\mathbb{E}\left[\sum_{a\in\mathcal{A}}\widehat{\mu}\left(a\mymid0\right)\mathbb{P}\left(\theta_{a}=q\mymid X_{a},M_{a}\right)\right]\\
 & \quad\geq6\varepsilon\,\underset{\widehat{\mu}}{\inf}\,\mathbb{E}\left[\sum_{a\in\mathcal{A}}\widehat{\mu}\left(a\mymid0\right)\mathbb{P}\left(\theta_{a}=q\mymid X_{a},M_{a}\right)\ind\left\{ \left(X_{a},M_{a}\right)\in E\right\} \right]\\
 & \quad\overset{\text{(ii)}}{\geq}2\varepsilon\,\underset{\widehat{\mu}}{\inf}\,\mathbb{E}\left[\sum_{a\in\mathcal{A}}\widehat{\mu}\left(a\mymid0\right)\ind\left\{ \left(X_{a},M_{a}\right)\in E\right\} \right]\\
 & \quad=2\varepsilon\,\underset{\widehat{\mu}}{\inf}\,\mathbb{E}\left[1-\sum_{a\in\mathcal{A}}\widehat{\mu}\left(a\mymid0\right)\ind\left\{ \left(X_{a},M_{a}\right)\notin E\right\} \right]\\
 & \quad\geq2\varepsilon\,\mathbb{E}\left[1-\sum_{a\in\mathcal{A}}\ind\left\{ \left(X_{a},M_{a}\right)\notin E\right\} \right]=2\varepsilon\left[1-\sum_{a\in\mathcal{A}}\mathbb{P}\Big(\left(X_{a},M_{a}\right)\notin E\Big)\right]\\
 & \quad\overset{\text{(iii)}}{\geq}2\varepsilon\left(1-\frac{4}{A^{3}}\right)\overset{\text{(iv)}}{\geq}\varepsilon.
\end{align*}
Here, (i) follows from (\ref{eq:proof-lb-3}) and (\ref{eq:proof-lb-4});
(ii) makes use of (\ref{eq:proof-lb-11}); (iii) follows from (\ref{eq:proof-lb-10});
and (iv) is valid when $A\geq2$.

\paragraph{Case 2: when the sample size $N$ is small.}

We now turn attention to the complement case where
\[
N<\widetilde{C}\frac{S\left(A+B\right)C_{\mathsf{clipped}}^{\star}\log\left(A+B\right)}{1-\gamma}\leq\frac{c_{2}S\left(A+B\right)C_{\mathsf{clipped}}^{\star}}{\left(1-\gamma\right)^{3}\varepsilon^{2}\log\left(A+B\right)}.
\]
Given that the sample size is smaller than the one in Case 1, it is
trivially seen that the minimax lower bound cannot be better than
the former case, namely, we must have
\[
\underset{(\widehat{\mu},\widehat{\nu})}{\inf}\underset{\{\mathcal{MG},\rho,d_{\mathsf{b}}\}\in\mathsf{MG}(C^{\star})}{\sup}\mathbb{E}\left[V^{\star,\widehat{\nu}}\left(\rho\right)-V^{\widehat{\mu},\star}\left(\rho\right)\right]\geq\varepsilon_{0}=\sqrt{\frac{c_{2}}{2\widetilde{C}}}\frac{1}{\left(1-\gamma\right)\log\left(A+B\right)},
\]
where $\varepsilon_{0}$ is defined in \eqref{eq:eps-upper-bound-case1}. 

\paragraph{Putting Case 1 and Case 2 together. }

Combining the above two cases reveals that: for any $\varepsilon\leq\varepsilon_{0}=\sqrt{\frac{c_{2}}{2\widetilde{C}}}\frac{1}{\left(1-\gamma\right)\log\left(A+B\right)}$
, one necessarily has
\[
\underset{(\widehat{\mu},\widehat{\nu})}{\inf}\underset{\{\mathcal{MG},\rho,d_{\mathsf{b}}\}\in\mathsf{MG}(C^{\star})}{\sup}\mathbb{E}\left[V^{\star,\widehat{\nu}}\left(\rho\right)-V^{\widehat{\mu},\star}\left(\rho\right)\right]\geq\varepsilon
\]
if the sample size $N\leq\frac{c_{2}S\left(A+B\right)C_{\mathsf{clipped}}^{\star}}{\left(1-\gamma\right)^{3}\varepsilon^{2}\log\left(A+B\right)}$.
This concludes the proof of Theorem \ref{thm:lower-bound}. 

\begin{comment}
This is because in this case, the class of estimators can be viewed
as a subset of the class of estimators that only uses the first $N$
samples when the sample size exceeds
\[
N\geq\widetilde{C}\frac{S\left(A+B\right)C_{\mathsf{clipped}}^{\star}\log\left(A+B\right)}{1-\gamma}.
\]
\end{comment}
{} 

\section{Auxiliary lemmas for Theorem \ref{thm:lower-bound}}

\subsection{Proof of Lemma \ref{lemma:MG-LB-property} \label{subsec:proof-lemma-MG-LB-property}}

First of all, it is straightforward to verify that: when initialized
to any state $s\neq0$, the MG will never leave the state $s$ (by
construction of $P_{\theta}$). This combined with the fact that the
rewards are zero whenever $s\neq0$ gives
\begin{equation}
V^{\mu,\nu}\left(s\right)=0\qquad\text{for all}\,s\neq0.\label{eq:proof-MG-LB-property-1}
\end{equation}
As a consequence, it suffices to compute $V_{\theta}^{\mu,\nu}(0)$.
By virtue of the Bellman equation, we obtain
\begin{align*}
V^{\mu,\nu}\left(0\right) & =\underset{a\sim\mu(0),b\sim\nu(0)}{\mathbb{E}}\left[r\left(0,a,b\right)+\gamma\sum_{s'\in\mathcal{S}}P\left(s'\mymid0,a,b\right)V^{\mu,\nu}\left(s'\right)\right]\\
 & \overset{\text{(i)}}{=}1+\gamma\underset{a\sim\mu(0),b\sim\nu(0)}{\mathbb{E}}\big[P_{\theta}\left(0\mymid0,a,b\right)V^{\mu,\nu}\left(0\right)\big]\\
 & =1+\gamma\sum_{a\in\mathcal{A}_{p,\theta}}\mu\left(a\mymid0\right)\nu\left(0\mymid0\right)P_{\theta}\left(0\mymid0,a,0\right)V^{\mu,\nu}\left(0\right)\\
 & \quad+\gamma\sum_{a\in\mathcal{A}_{p,\theta}}\mu\left(a\mymid0\right)\sum_{b\neq0}\nu\left(b\mymid0\right)P_{\theta}\left(0\mymid0,a,b\right)V^{\mu,\nu}\left(0\right)\\
 & \quad+\gamma\sum_{a\in\mathcal{A}_{q,\theta}}\mu\left(a\mymid0\right)\nu\left(0\mymid0\right)P_{\theta}\left(0\mymid0,a,0\right)V^{\mu,\nu}\left(0\right)\\
 & \quad+\gamma\sum_{a\in\mathcal{A}_{q,\theta}}\mu\left(a\mymid0\right)\sum_{b\neq0}\nu\left(b\mymid0\right)P_{\theta}\left(0\mymid0,a,b\right)V^{\mu,\nu}\left(0\right)\\
 & \overset{\text{(ii)}}{=}1+\gamma\mu_{p,\theta}\nu_{0}pV^{\mu,\nu}\left(0\right)+\gamma\mu_{p,\theta}\left(1-\nu_{0}\right)V^{\mu,\nu}\left(0\right)\\
 & \quad+\gamma\left(1-\mu_{p,\theta}\right)\nu_{0}qV^{\mu,\nu}\left(0\right)+\gamma\left(1-\mu_{p,\theta}\right)\left(1-\nu_{0}\right)V^{\mu,\nu}\left(0\right),
\end{align*}
where we remind the readers of the quantities $\mu_{p,\theta}\coloneqq\sum_{a\in\mathcal{A}_{p,\theta}}\mu(a\mymid0)$
and $\nu_{0}\coloneqq\nu(0\mymid0)$. Here, (i) makes use of the fact
that $r(0,a,b)=1$ for all $(a,b)\in\mathcal{A}\times\mathcal{B}$
and (\ref{eq:proof-MG-LB-property-1}), while (ii) follows from the
definition of $P_{\theta}$ in (\ref{eq:LB-transition}). Rearrange
terms to arrive at
\begin{align*}
V^{\mu,\nu}\left(0\right) & =\frac{1}{1-\gamma\mu_{p,\theta}\nu_{0}p-\gamma\mu_{p,\theta}\left(1-\nu_{0}\right)-\gamma\left(1-\mu_{p,\theta}\right)\nu_{0}q-\gamma\left(1-\mu_{p,\theta}\right)\left(1-\nu_{0}\right)}\\
 & =\frac{1}{1-\gamma+\gamma\mu_{p,\theta}\nu_{0}\left(1-p\right)+\gamma\left(1-\mu_{p,\theta}\right)\nu_{0}\left(1-q\right)}.
\end{align*}

Next, let us rewrite $V^{\mu,\nu}(0)$ as follows
\begin{align*}
V^{\mu,\nu}\left(0\right) & =\frac{1}{1-\gamma+\gamma\nu_{0}\left(1-p\right)-\gamma(1-\mu_{p,\theta})\nu_{0}\left(1-p\right)+\gamma\left(1-\mu_{p,\theta}\right)\nu_{0}\left(1-q\right)}\\
 & =\frac{1}{1-\gamma+\gamma\nu_{0}\left(1-p\right)+\gamma\left(1-\mu_{p,\theta}\right)\nu_{0}\left(p-q\right)}.
\end{align*}
Given that $p>q$, we see that: for any policy $\mu$, the best response
of the min-player would be to take $\nu_{0}=1$; and for any policy
$\nu$ with $\nu_{0}\neq0$, the best response of the max-player would
be to set $\mu_{p,\theta}=1$. As a result, it is seen that $(\mu,\nu)$
is an NE if and only if $\mu_{p,\theta}=1$ and $\nu_{0}=1$. This
readily implies that the policy pair $(\mu_{\theta}^{\star},\nu_{\theta}^{\star})$
defined in (\ref{eq:LB-NE}) is an NE of $\mathcal{MG}_{\theta}$. 

We are left with justifying that $\{\mathcal{MG}_{\theta},\rho,d_{\mathsf{b}}\}\in\mathsf{MG}(C_{\mathsf{clipped}}^{\star})$.
Towards this, we first note that: for any $a\in\mathcal{A}$, $b\in\mathcal{B}$,
and any policy pair $(\mu,\nu)$, we can invoke (\ref{eq:LB-db})
to show that
\begin{align}
\frac{\min\left\{ d^{\mu,\nu}\left(0,a,b;\rho\right),\frac{1}{S\left(A+B\right)}\right\} }{d_{\mathsf{b}}\left(0,a,b\right)} & \leq\frac{\frac{1}{S\left(A+B\right)}}{\frac{1}{C_{\mathsf{clipped}}^{\star}S\left(A+B\right)}}=C_{\mathsf{clipped}}^{\star},\label{eq:proof-MG-LB-property-2.1}
\end{align}
and
\begin{equation}
\frac{\min\left\{ d^{\mu,\nu}\left(1,a,b;\rho\right),\frac{1}{S\left(A+B\right)}\right\} }{d_{\mathsf{b}}\left(1,a,b\right)}\leq\frac{\frac{1}{S\left(A+B\right)}}{\frac{1}{AB}\left(1-\frac{AB}{C_{\mathsf{clipped}}^{\star}S\left(A+B\right)}\right)}=\frac{1}{\frac{S\left(A+B\right)}{AB}-\frac{1}{C_{\mathsf{clipped}}^{\star}}}\leq C_{\mathsf{clipped}}^{\star},\label{eq:proof-MG-LB-property-2.2}
\end{equation}
where the last inequality holds as long as $C_{\mathsf{clipped}}^{\star}\geq\frac{2AB}{S(A+B)}$.
In addition, for any $s\geq2$, it is readily seen that $d^{\mu,\nu}(s,a,b;\rho)=0$,
and therefore, 
\begin{equation}
\frac{\min\left\{ d^{\mu,\nu}\left(s,a,b;\rho\right),\frac{1}{S\left(A+B\right)}\right\} }{d_{\mathsf{b}}\left(s,a,b\right)}=0.\label{eq:proof-MG-LB-property-2.3}
\end{equation}
Taking (\ref{eq:proof-MG-LB-property-2.1}), (\ref{eq:proof-MG-LB-property-2.2})
and (\ref{eq:proof-MG-LB-property-2.3}) collectively gives 
\begin{equation}
\max\left\{ \sup_{\mu,s,a,b}\frac{\min\left\{ d^{\mu,\nu_{\theta}^{\star}}\left(s,a,b;\rho\right),\frac{1}{S\left(A+B\right)}\right\} }{d_{\mathsf{b}}\left(s,a,b\right)},\sup_{\nu,s,a,b}\frac{\min\left\{ d^{\mu_{\theta}^{\star},\nu}\left(s,a,b;\rho\right),\frac{1}{S\left(A+B\right)}\right\} }{d_{\mathsf{b}}\left(s,a,b\right)}\right\} \leq C_{\mathsf{clipped}}^{\star}.\label{eq:proof-MG-LB-property-3}
\end{equation}
We still need to justify that the inequality in (\ref{eq:proof-MG-LB-property-3})
is tight. To do so, observe that for any $a\in\mathcal{A}_{p}$, one
has
\begin{align}
d^{\mu_{\theta}^{\star},\nu_{\theta}^{\star}}\left(0,a,0\right) & =\left(1-\gamma\right)\sum_{t=0}^{\infty}\gamma^{t}\mathbb{P}\left(s_{t}=0,a_{t}=a,b_{t}=0\mymid s_{0}\sim\rho;\mu_{\theta}^{\star},\nu_{\theta}^{\star}\right)\nonumber \\
 & \overset{\text{(i)}}{=}\left(1-\gamma\right)\sum_{t=0}^{\infty}\gamma^{t}\mathbb{P}\left(s_{t}=0\mymid s_{0}\sim\rho;\mu_{\theta}^{\star},\nu_{\theta}^{\star}\right)\frac{1}{\vert\mathcal{A}_{p,\theta}\vert}\nonumber \\
 & \overset{\text{(ii)}}{\geq}\frac{1-\gamma}{\big|\mathcal{A}_{p,\theta}\big|}\sum_{t=0}^{\infty}\gamma^{t}p^{t}\overset{\text{(iii)}}{\geq}\frac{1-\gamma}{\big|\mathcal{A}_{p,\theta}\big|}\sum_{t=0}^{\infty}\gamma^{2t}\nonumber \\
 & =\frac{1-\gamma}{\big|\mathcal{A}_{p,\theta}\big|}\frac{1}{1-\gamma^{2}}=\frac{1}{\big|\mathcal{A}_{p,\theta}\big|}\frac{1}{1+\gamma}\geq\frac{1}{2\big|\mathcal{A}_{p,\theta}\big|}.\label{eq:proof-MG-LB-property-4}
\end{align}
Here, (i) is valid since according to (\ref{eq:LB-NE}), $\mathbb{P}(a_{t}=a\mymid s_{t}=0;\mu_{\theta}^{\star})=1/\vert\mathcal{A}_{p,\theta}\vert$
for $a\in\mathcal{A}_{p,\theta}$ and $\mathbb{P}(b_{t}=0\mymid s_{t}=0;\nu_{\theta}^{\star})=1$;
(ii) follows since
\[
\mathbb{P}\left(s_{t}=0\mymid s_{0}\sim\rho;\mu_{\theta}^{\star},\nu_{\theta}^{\star}\right)\geq\mathbb{P}\left(s_{t}=s_{t-1}=\cdots=s_{0}=0\mymid\mu_{\theta}^{\star},\nu_{\theta}^{\star}\right)=p^{t};
\]
and (iii) relies on the fact that $p\geq\gamma$ (cf.~(\ref{eq:LB-p-q})).
Then we can invoke (\ref{eq:proof-MG-LB-property-4}) and (\ref{eq:LB-db})
to demonstrate that
\begin{align}
\frac{\min\big\{ d^{\mu_{\theta}^{\star},\nu_{\theta}^{\star}}\left(0,a,0\right),\frac{1}{S\left(A+B\right)}\big\}}{d_{\mathsf{b}}\left(0,a,0\right)} & =\frac{\frac{1}{S\left(A+B\right)}}{\frac{1}{C_{\mathsf{clipped}}^{\star}S(A+B)}}=C_{\mathsf{clipped}}^{\star}\qquad\text{for any }a\in\mathcal{A}_{p,\theta}.\label{eq:proof-MG-LB-property-5}
\end{align}
Taking collectively (\ref{eq:proof-MG-LB-property-3}) and (\ref{eq:proof-MG-LB-property-5})
gives
\[
\max\left\{ \sup_{\mu,s,a,b}\frac{\min\left\{ d^{\mu,\nu_{\theta}^{\star}}\left(s,a,b;\rho\right),\frac{1}{S\left(A+B\right)}\right\} }{d_{\mathsf{b}}\left(s,a,b\right)},\sup_{\nu,s,a,b}\frac{\min\left\{ d^{\mu_{\theta}^{\star},\nu}\left(s,a,b;\rho\right),\frac{1}{S\left(A+B\right)}\right\} }{d_{\mathsf{b}}\left(s,a,b\right)}\right\} =C_{\mathsf{clipped}}^{\star}.
\]
This allows one to conclude that $\{\mathcal{MG}_{\theta},\rho,d_{\mathsf{b}}\}\in\mathsf{MG}(C_{\mathsf{clipped}}^{\star})$. 

\subsection{Proof of Lemma \ref{lemma:Ma-concentration} \label{subsec:proof-lemma-Ma-concentration}}

Recall that $M_{a}\sim\mathsf{Binomial}(N,d_{\mathsf{b}}(0,a,0))$.
Invoke the Chernoff bound (e.g., \citet[Exercise 2.3.5]{vershynin2018high})
to show the existence of some universal constant $c_{\mathsf{ch}}>0$
such that
\begin{align}
\mathbb{P}\Big(\left|M_{a}-Nd_{\mathsf{b}}\left(0,a,0\right)\right|\geq\delta Nd_{\mathsf{b}}\left(0,a,0\right)\Big) & \leq2\exp\left(-c_{\mathsf{ch}}Nd_{\mathsf{b}}\left(0,a,0\right)\delta^{2}\right)\label{eq:Chernoff-Ma}
\end{align}
holds for any $\delta\in(0,1]$. For $\widetilde{C}$ sufficiently
large, one has (see \eqref{eq:N-lower-bound-tilde-C})
\[
N\geq\frac{4}{c_{\mathsf{ch}}}C_{\mathsf{clipped}}^{\star}S\left(A+B\right)\log A,
\]
or equivalently, 
\[
Nd_{\mathsf{b}}\left(0,a,0\right)\geq\frac{4}{c_{\mathsf{ch}}}\log A.
\]
Take
\[
\delta\coloneqq\sqrt{\frac{4\log A}{c_{\mathsf{ch}}Nd_{\mathsf{b}}\left(0,a,0\right)}}\leq1
\]
in \eqref{eq:Chernoff-Ma} to show that with probability exceeding
$1-2A^{-4}$,
\begin{align}
M_{a} & \leq Nd_{\mathsf{b}}\left(0,a,0\right)+\sqrt{Nd_{\mathsf{b}}\left(0,a,0\right)}\sqrt{\frac{4\log A}{c_{\mathsf{ch}}}}\overset{\mathrm{(i)}}{\leq}\frac{c_{2}}{\left(1-\gamma\right)^{3}\varepsilon^{2}\log A}+\sqrt{\frac{8c_{2}}{c_{\mathsf{ch}}\left(1-\gamma\right)^{3}\varepsilon^{2}}}\nonumber \\
 & \overset{\mathrm{(ii)}}{\leq}\frac{c_{2}+\sqrt{8c_{2}/c_{\mathsf{ch}}}}{\left(1-\gamma\right)^{3}\varepsilon^{2}\log A}=\frac{c_{4}}{\left(1-\gamma\right)^{3}\varepsilon^{2}\log A}\label{eq:proof-lb-5}
\end{align}
holds for some universal constant $c_{4}=c_{2}+\sqrt{8c_{2}/c_{\mathsf{ch}}}$,
where (i) makes use of Condition \eqref{eq:N-upper-bound-c2}, and
(ii) is valid as long as $\varepsilon\leq\frac{1}{(1-\gamma)\log A}$.
Moreover, we can also invoke \eqref{eq:Chernoff-Ma} and \eqref{eq:N-lower-bound-tilde-C}
to show that with probability exceeding $1-2A^{-4}$, 
\begin{align}
M_{a} & \geq Nd_{\mathsf{b}}\left(0,a,0\right)-\sqrt{Nd_{\mathsf{b}}\left(0,a,0\right)}\sqrt{\frac{2\log A}{c_{\mathsf{ch}}}}\overset{\text{(iii)}}{\geq}\widetilde{C}\frac{\log A}{1-\gamma}-\sqrt{\frac{2\widetilde{C}}{c_{\mathsf{ch}}}\frac{\log^{2}A}{1-\gamma}}\overset{\text{(iv)}}{\geq}\frac{c_{3}\log A}{1-\gamma}\label{eq:proof-lb-6}
\end{align}
for some universal constant $c_{3}=\widetilde{C}/2$. Here, (iii)
holds since the function $x^{2}-\sqrt{\frac{2\log A}{c_{\mathsf{ch}}}}x$
is monotonically increasing when $x\geq\sqrt{\frac{\log A}{c_{\mathsf{ch}}}}$,
and 
\[
Nd_{\mathsf{b}}\left(0,a,0\right)\geq\widetilde{C}\frac{\log A}{1-\gamma}\geq\sqrt{\frac{\log A}{c_{\mathsf{ch}}}}
\]
when $\widetilde{C}>0$ is sufficiently large; (iv) holds for $\widetilde{C}>0$
sufficiently large. Taking (\ref{eq:proof-lb-5}) and (\ref{eq:proof-lb-6})
collectively gives
\[
\mathbb{P}\left(\frac{c_{3}\log A}{1-\gamma}\leq M_{a}\leq\frac{c_{4}}{\left(1-\gamma\right)^{3}\varepsilon^{2}\log A}\right)\geq1-\frac{2}{A^{4}}.
\]

\subsection{Proof of Lemma \ref{lemma:E-existence} \label{subsec:proof-lemma-E-existence}}

Given that $B_{p}\sim\mathsf{Binomial}(M,p)$, we have $M-B_{p}\sim\mathsf{Binomial}(M,1-p)$.
In view of the Chernoff bound (cf.~\citet[Exercise 2.3.5]{vershynin2018high}),
we know that for any $\delta\in(0,1]$
\begin{align}
\mathbb{P}\Big(\left|B_{p}-Mp\right|>\delta M\left(1-p\right)\Big) & =\mathbb{P}\Big(\left|M-B_{p}-M\left(1-p\right)\right|>\delta M\left(1-p\right)\Big)\nonumber \\
 & \leq2\exp\left(-c_{\mathsf{ch}}M\left(1-p\right)\delta^{2}\right)\label{eq:proof-E-existence-1}
\end{align}
for some universal constant $c_{\mathsf{ch}}>0$. Recall that $M\geq\frac{c_{3}\log A}{1-\gamma}$.
By taking 
\[
\delta\coloneqq\sqrt{\frac{4\log A}{c_{\mathsf{ch}}M\left(1-p\right)}}
\]
we can guarantee that
\[
\delta\leq\sqrt{\frac{4\left(1-\gamma\right)}{c_{\mathsf{ch}}c_{3}\left(1-p\right)}}=\sqrt{\frac{4\left(1-\gamma\right)}{c_{\mathsf{ch}}c_{3}\left(1-\gamma-14\frac{\left(1-\gamma\right)^{2}\varepsilon}{\gamma}\right)}}\leq\sqrt{\frac{4}{c_{\mathsf{ch}}c_{3}\left(1-14\frac{\left(1-\gamma\right)\varepsilon}{\gamma}\right)}}\overset{\text{(i)}}{\leq}\sqrt{\frac{8}{c_{\mathsf{ch}}c_{3}}}\overset{\text{(ii)}}{\leq}1.
\]
Here, (i) holds when $\varepsilon\leq\frac{1}{42(1-\gamma)}$ and
$\gamma\geq2/3$, while (ii) is valid as long as $c_{3}=\widetilde{C}/2\geq8/c_{\mathsf{ch}}$
(a condition that can be guaranteed as long as $\widetilde{C}$ is
sufficiently large). Then the inequality (\ref{eq:proof-E-existence-1})
tells us that
\begin{align*}
\mathbb{P}\left(\left|B_{p}-Mp\right|>\sqrt{\frac{4\log A}{c_{\mathsf{ch}}}M\left(1-p\right)}\right) & \leq\frac{2}{A^{4}}.
\end{align*}
A similar argument leads to
\[
\mathbb{P}\left(\left|B_{q}-Mq\right|\geq\sqrt{\frac{4\log A}{c_{\mathsf{ch}}}M\left(1-q\right)}\right)\leq\frac{2}{A^{4}}.
\]
In view of these concentration bounds, we can take
\[
E_{M}\coloneqq\Big\{0\leq n\leq M:\left\lfloor \min\left\{ Mp-\Delta_{p},Mq-\Delta_{q}\right\} \right\rfloor \leq n\leq\left\lceil \max\left\{ Mp+\Delta_{p},Mq+\Delta_{q}\right\} \right\rceil \Big\},
\]
where we define, for every $x\in(0,1)$, 
\[
\Delta_{x}\coloneqq\sqrt{\frac{4\log A}{c_{\mathsf{ch}}}M\left(1-x\right)}.
\]
It is clear from the above argument that (\ref{eq:E-high-prob}) is
guaranteed to hold. 

We are left with checking (\ref{eq:E-lower-bound}). Towards this,
we first make the observation that
\begin{align*}
\frac{\mathbb{P}\left(B_{q}=m\right)}{\mathbb{P}\left(B_{p}=m\right)} & =\frac{q^{m}\left(1-q\right)^{M-m}}{p^{m}\left(1-p\right)^{M-m}}=\exp\left(-m\log\frac{p}{q}-\left(M-m\right)\log\frac{1-p}{1-q}\right)\\
 & =\exp\left(-M\left[\mathsf{KL}\left(p\,\Vert\,q\right)+\left(\frac{m}{M}-p\right)\log\frac{p\left(1-q\right)}{q\left(1-p\right)}\right]\right),
\end{align*}
where $\mathsf{KL}(p\,\Vert\,q)$ denotes the Kullback-Leibler (KL)
divergence between $\mathsf{Bernoulli}(p)$ and $\mathsf{Bernoulli}(q)$
\citep{tsybakov2009introduction}: 
\[
\mathsf{KL}\left(p\,\Vert\,q\right)\coloneqq p\log\frac{p}{q}+\left(1-p\right)\log\frac{1-p}{1-q}.
\]
Moreover, it is seen that
\begin{equation}
\mathsf{KL}\left(p\,\Vert\,q\right)\overset{\text{(i)}}{\leq}\frac{\left(p-q\right)^{2}}{p\left(1-p\right)}=\frac{784\left(1-\gamma\right)^{4}\varepsilon^{2}}{p\left(1-\gamma-14\frac{\left(1-\gamma\right)^{2}\varepsilon}{\gamma}\right)\gamma^{2}}\overset{\text{(ii)}}{\leq}\frac{3528\left(1-\gamma\right)^{3}\varepsilon^{2}}{1-14\frac{\left(1-\gamma\right)\varepsilon}{\gamma}}\leq7056\left(1-\gamma\right)^{3}\varepsilon^{2},\label{eq:proof-E-1}
\end{equation}
where (i) follows from \citet[Lemma 10]{li2022settling} and (ii) relies
on (\ref{eq:p-q-lower-bound}). In addition, for any $m\in E_{M}$,
we have $m\leq Mp+\Delta_{q}$, which in turn reveals that
\begin{equation}
\frac{m}{M}-p\leq\frac{\Delta_{q}}{M}\leq\sqrt{\frac{4\log A}{c_{\mathsf{ch}}M}\left(1-\gamma+14\frac{\left(1-\gamma\right)^{2}\varepsilon}{\gamma}\right)}\leq\sqrt{\frac{6\left(1-\gamma\right)\log A}{c_{\mathsf{ch}}M}}\label{eq:proof-E-2}
\end{equation}
provided that $\varepsilon\leq\frac{1}{42(1-\gamma)}$ and $\gamma\leq2/3$.
Recalling that $p>q$, we can also see that
\begin{align}
0 & <\log\frac{p\left(1-q\right)}{q\left(1-p\right)}=\log\left(1+\frac{p-q}{q\left(1-p\right)}\right)\overset{\text{(i)}}{\leq}\frac{p-q}{q\left(1-p\right)}\overset{\text{(ii)}}{\leq}\frac{28\left(1-\gamma\right)^{2}\varepsilon}{\gamma\left(1-\gamma-14\frac{\left(1-\gamma\right)^{2}\varepsilon}{\gamma}\right)}\nonumber \\
 & \leq\frac{42\left(1-\gamma\right)\varepsilon}{1-14\frac{\left(1-\gamma\right)\varepsilon}{\gamma}}\leq\frac{42\left(1-\gamma\right)\varepsilon}{1-\frac{1}{3\gamma}}\leq84\left(1-\gamma\right)\varepsilon,\label{eq:proof-E-3}
\end{align}
where (i) exploits the elementary inequality $\log(1+x)\leq x$ for
all $x\geq0$, and (ii) follows from (\ref{eq:p-q-lower-bound}).
Taking (\ref{eq:proof-E-1}), (\ref{eq:proof-E-2}) and (\ref{eq:proof-E-3})
collectively yields
\begin{align*}
\frac{\mathbb{P}\left(B_{q}=m\right)}{\mathbb{P}\left(B_{p}=m\right)} & \geq\exp\left[-M\left(7056\left(1-\gamma\right)^{3}\varepsilon^{2}+84\left(1-\gamma\right)\varepsilon\sqrt{\frac{6\left(1-\gamma\right)\log A}{c_{\mathsf{ch}}M}}\right)\right]\\
 & \overset{\text{(i)}}{\geq}\exp\left[-\left(\frac{7056c_{4}}{\log A}+84\sqrt{\frac{6c_{4}}{c_{\mathsf{ch}}}}\right)\right]\overset{\text{(ii)}}{\geq}\exp\left[-\left(\frac{7056c_{4}}{\log2}+84\sqrt{\frac{6c_{4}}{c_{\mathsf{ch}}}}\right)\right]\overset{\text{(iii)}}{\geq}\frac{1}{2}.
\end{align*}
Here, (i) follows from the condition that
\[
M\leq\frac{c_{4}}{\left(1-\gamma\right)^{3}\varepsilon^{2}\log A},
\]
(ii) holds when $A\geq2$, and (iii) holds with the proviso that $c_{4}=2c_{2}+\sqrt{8c_{2}/c_{\mathsf{ch}}}$
is sufficiently small, which can happen as long as $c_{2}$ is sufficiently
small.